\documentclass{article}

\PassOptionsToPackage{numbers, compress}{natbib}
\usepackage[main,final]{style_file}

\usepackage[utf8]{inputenc} 
\usepackage[T1]{fontenc}    
\usepackage{amsmath, amsthm, amssymb, amsfonts, bm}   
\usepackage{hyperref}       
\usepackage{url}            
\usepackage{booktabs}       
\usepackage{graphicx}       
\usepackage{caption}        
\usepackage{subcaption}     
\usepackage{nicefrac}       
\usepackage{microtype}      
\usepackage[table]{xcolor}         
\usepackage[capitalize]{cleveref} 
\usepackage{multirow} 
\usepackage{array}    
\usepackage{adjustbox} 
\usepackage{algorithm}
\usepackage{algorithmic}

\ifx\pdfminorversion\undefined\else\pdfminorversion=7\fi 

\newtheorem{theorem}{Theorem}
\newtheorem{corollary}[theorem]{Corollary}
\newtheorem{lemma}[theorem]{Lemma}
\newtheorem{proposition}[theorem]{Proposition}

\newtheorem{remark}{Remark}

\title{OliO: ODE-based Linear Transition Operator for Self-Supervised Time Series Forecasting}

\author{
    Kwangryeol Park \\
    Graduate School of Artificial Intelligence \\
    Ulsan National Institute of Science and Technology \\
    Ulsan, South Korea\\
    \texttt{pkr7098@unist.ac.kr} \\
    \And
    Seulki Lee\thanks{Corresponding Author}\\
    School of Electrical Engineering \\
    Korea Advanced Institute of Science and Technology \\
    Daejeon, South Korea\\
    \texttt{seulki.lee@kaist.ac.kr} \\
}

\begin{document}

\maketitle

\begin{abstract}
    Self-supervised representation learning has  advanced time series forecasting by capturing robust features from unlabeled data. However, existing contrastive and masking-based methods often struggle to explicitly model the underlying temporal flow, either by treating time lag as noise or by disrupting the inherent temporal dynamics essential for forecasting. To address these limitations, we propose OliO (\textbf{O}DE-based \textbf{Li}near Transition \textbf{O}perator), a plug-and-play self-supervised learning method that learns bidirectional temporal flows as a continuous and structurally consistent dynamical system. OliO introduces a transition operator derived from linear Ordinary Differential Equations (ODEs) to align latent representations across arbitrary time shifts. By imposing a strictly upper triangular constraint on the transition matrix, we mathematically ensure numerical stability and enforce a robust polynomial boundary that effectively captures long-term dependencies while preventing the exponential divergence typically found in unconstrained ODE-based models. This structural constraint induces a temporally coherent representation space that preserves the underlying flow essential for accurate forecasting. Extensive experiments across various backbone architectures and benchmarks demonstrate that OliO achieves significant performance gains, providing Mean Squared Error (MSE) reductions of up to 5.75\% over the most competitive grid-searched SOTA baselines. The code is available at \url{https://github.com/eai-lab/OliO}.
\end{abstract}

\section{Introduction}

Time series forecasting plays a crucial role in various real-world applications, such as energy consumption planning, traffic analysis, healthcare monitoring, and financial market prediction~\cite{Autoformer, Informer}. However, building accurate forecasting models remains a significant challenge due to the complex temporal variations and high-dimensional nature of massive time series data. Recently, to address these challenges, self-supervised learning (SSL), which learns meaningful representations from unlabeled data, has been actively adopted in the time series domain. Representative approaches include contrastive learning~\cite{TS2Vec, TS-TCC, CoST}, which learns representations by regulating the similarity of related temporal views, and masked modeling~\cite{Ti-MAE, SimMTM, TimeSiam}, which learns through reconstruction of masked observations.

A central question for forecasting is how the temporal displacement between historical windows should be represented during self-supervised pre-training. Two overlapping windows may share many observations while representing different temporal states. Their temporal displacement specifies both how far apart the windows are and which window precedes the other. Contrastive and masked-modeling objectives provide supervision through representation similarity or observation reconstruction. SoftCLT~\cite{SoftCLT} additionally uses temporal distance to weight similarity, but does not explicitly learn a transition conditioned on the signed lag between the two windows. This motivates an objective that explicitly learns where a historical representation should move under the pair's actual signed temporal displacement.

To address this question, we propose OliO (\textbf{O}DE-based \textbf{Li}near Transition \textbf{O}perator), a plug-and-play self-supervised framework for learning signed, lag-conditioned, bidirectional temporal transitions. Rather than directly aligning representations from different time points, OliO first transitions one representation from its original time point to the other before performing alignment. Specifically, given two representations independently encoded from historical windows separated by a lag $\tau>0$, OliO models the relationship between the two representations using an ODE-derived linear transition operator~\cite{minoza2026spike,somasekharan2026deep,mondal2024efficient}. It then transitions the earlier representation to the later time point using $+\tau$ and aligns it with the later representation. In the reverse direction, OliO similarly transitions the later representation using $-\tau$ and aligns it with the earlier representation.

OliO introduces the projection and transition modules only during pre-training. Both modules are discarded for downstream forecasting, leaving the forecasting architecture of the original backbone unchanged. We evaluate the framework across multiple backbones under a shared-backbone protocol and examine representation quality and transferability through frozen-backbone probing and zero-shot forecasting.

The main contributions of this paper are as follows:
\begin{itemize}
    \item \textbf{Signed temporal-transition pre-training:} We formulate a forecasting-oriented self-supervised objective that learns bidirectional transitions between independently encoded representations of historical windows, conditioned on the magnitude and direction of their actual temporal displacement.
    \item \textbf{Finite-polynomial transition structure:} We parameterize bidirectional transitions between representations of historical windows using a strictly upper-triangular ODE generator. The generator's nilpotency ensures that these transitions are exactly expressible as finite polynomials in temporal lag.
    \item \textbf{Model-agnostic representation learning:} OliO operates on the final-layer representations of forecasting backbones, allowing it to be applied during pre-training without modifying their internal architectures.
\end{itemize}

\section{Related work}
\label{sec:related_work}

\subsection{Self-supervised learning for time series}
Contrastive learning and masked modeling provide effective supervision for learning time-series representations without forecasting labels. Contrastive methods such as TS2Vec~\cite{TS2Vec}, TS-TCC~\cite{TS-TCC}, and CoST~\cite{CoST} regulate similarity between related temporal views while separating negative pairs. Masked-modeling methods such as Ti-MAE~\cite{Ti-MAE} and SimMTM~\cite{SimMTM} learn representations through observation reconstruction. These objectives capture useful temporal structure, but do not explicitly parameterize how one historical representation should transition toward another under their actual signed temporal displacement.

Non-contrastive methods such as SimSiam~\cite{SimSiam} and BYOL~\cite{BYOL} learn by predicting a target representation from another view, with gradients stopped through the target branch. OliO builds on this alignment mechanism. The distinction lies in the input and role of the prediction module: conventional view-alignment heads receive neither the actual temporal separation of two historical windows nor its direction. In contrast, OliO takes the pair-specific signed lag as an input and explicitly incorporates it into the transition module's formulation to transition the source representation before alignment. Thus, its temporal supervision comes from the actual displacement between the windows, in addition to prediction across the two branches.

Time-aware methods incorporate temporal relations more directly. SoftCLT~\cite{SoftCLT} uses temporal distance to weight a contrastive similarity objective. The distance controls the strength of attraction between representations, rather than a transformation toward a lag-conditioned destination; exchanging temporal order leaves the unsigned distance unchanged. TimeSiam~\cite{TimeSiam} incorporates temporally separated subseries and past-to-current reconstruction. It therefore exploits temporal context, while supervising observation reconstruction rather than a pair-specific signed transition between independently encoded historical representations. OliO instead conditions a learnable bidirectional transition on the actual signed lag and aligns each transitioned representation with the independently encoded representation of the other window. This makes temporal displacement determine a representation-level transition, rather than only the strength of similarity or the context for reconstructing observations.

\begin{figure*}[!t]
    \centering
    \setlength{\tabcolsep}{1.0pt} 
    
    \begin{tabular}{cllllc}
        \adjustbox{valign=m, raise=0ex}{\rotatebox{90}{\textbf{UMAP}}} & 
        \begin{subfigure}{0.21\textwidth}
            \centering
            \adjustbox{valign=m}{\includegraphics[width=\textwidth]{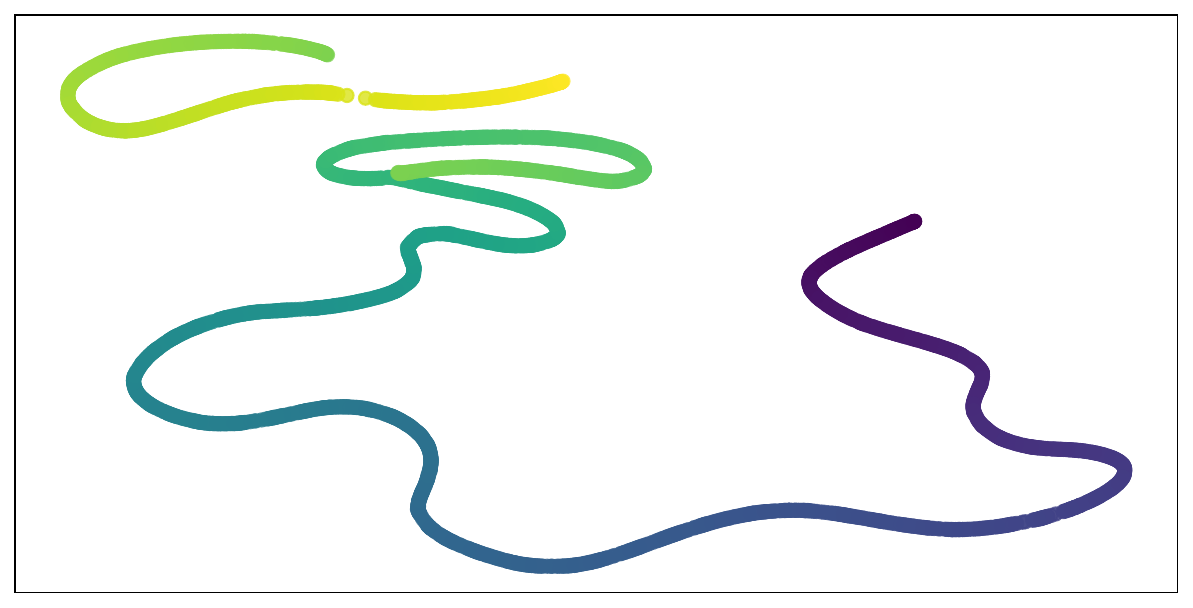}}
        \end{subfigure} &
        \begin{subfigure}{0.21\textwidth}
            \centering
            \adjustbox{valign=m}{\includegraphics[width=\textwidth]{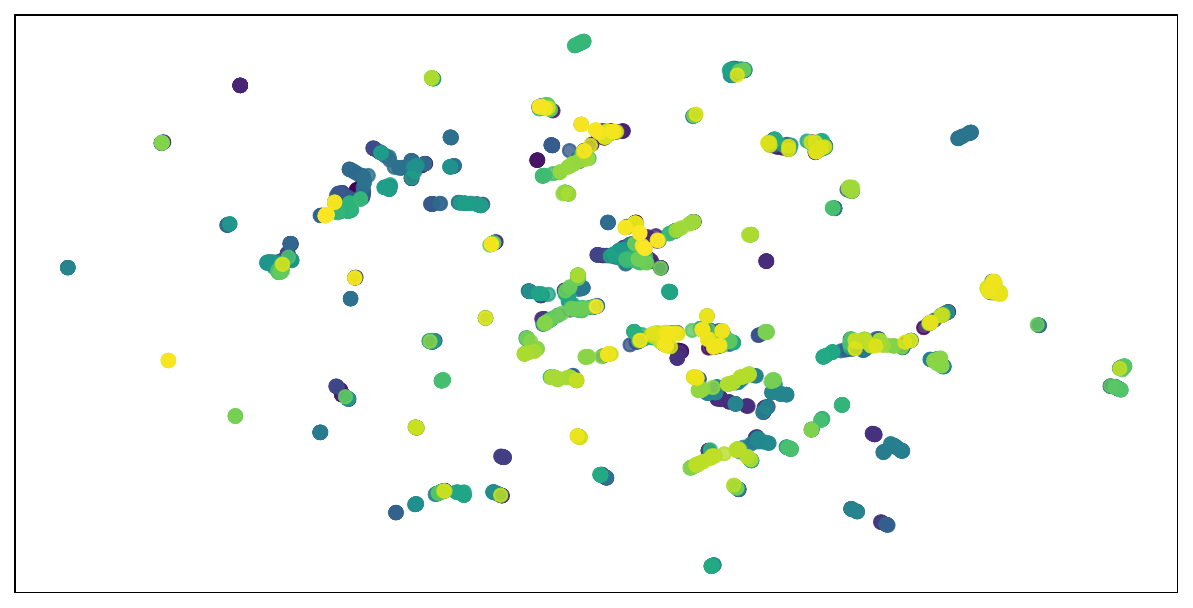}}
        \end{subfigure} &
        \begin{subfigure}{0.21\textwidth}
            \centering
            \adjustbox{valign=m}{\includegraphics[width=\textwidth]{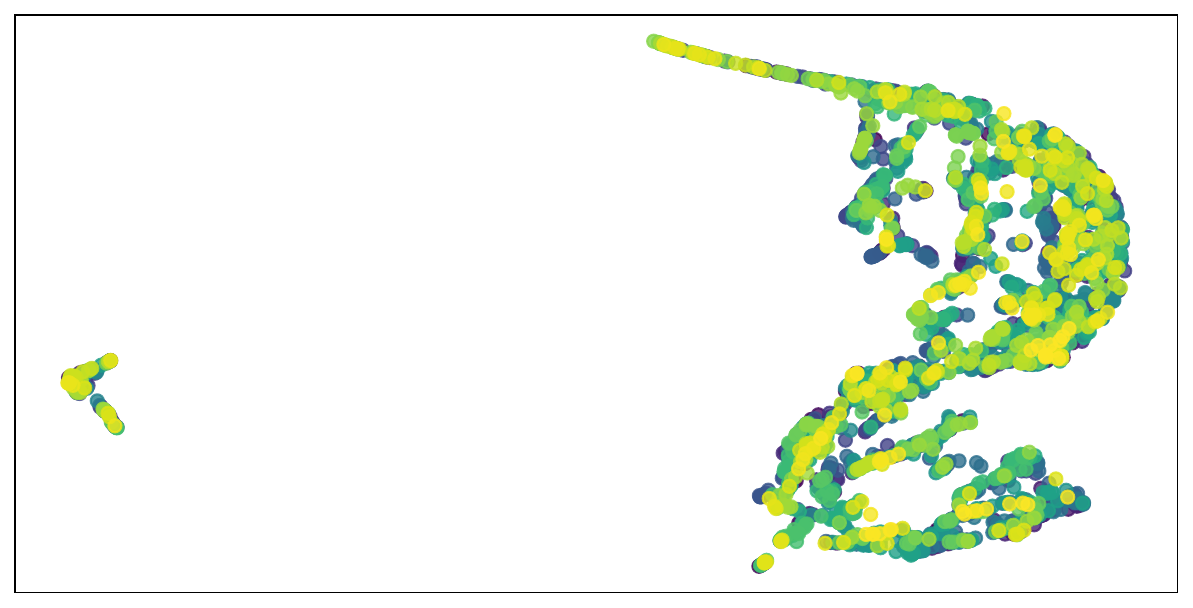}}
        \end{subfigure} &
        \begin{subfigure}{0.21\textwidth}
            \centering
            \adjustbox{valign=m}{\includegraphics[width=\textwidth]{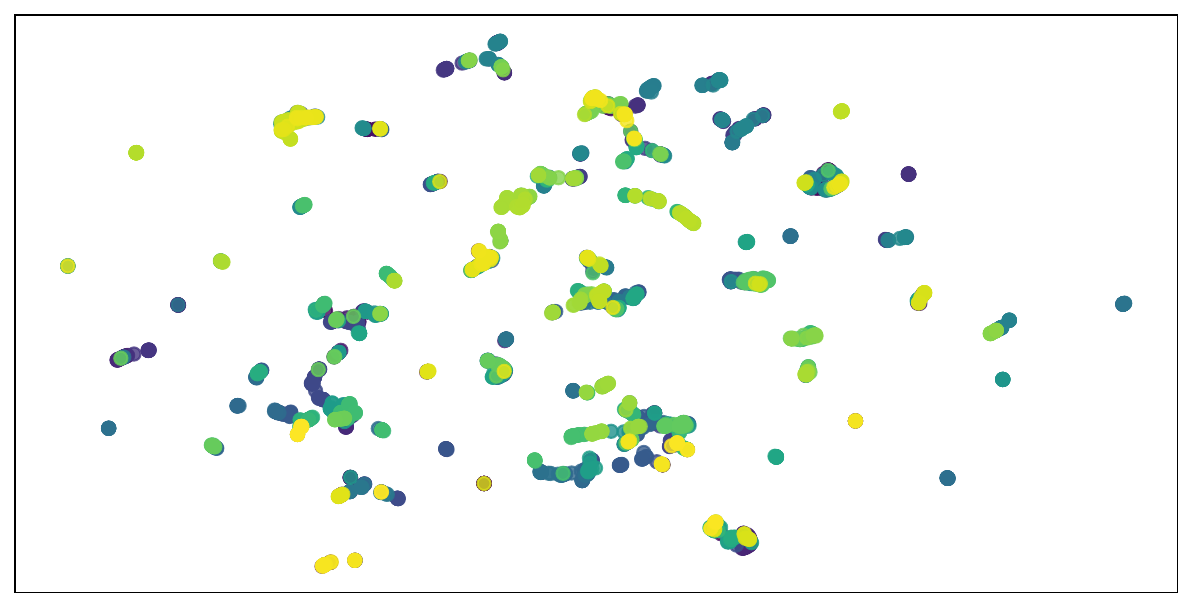}}
        \end{subfigure} &
        \multirow{2}{*}[1.5em]{\adjustbox{valign=m}{ 
            \includegraphics[height=0.20\textwidth]{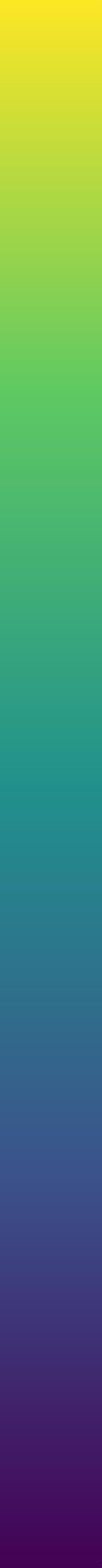}
        }} \\ [1.5em] 
        
        \adjustbox{valign=m, raise=7ex}{\rotatebox{90}{\textbf{PHATE}}} & 
        \begin{subfigure}{0.21\textwidth}
            \centering
            \adjustbox{valign=m}{\includegraphics[width=\textwidth]{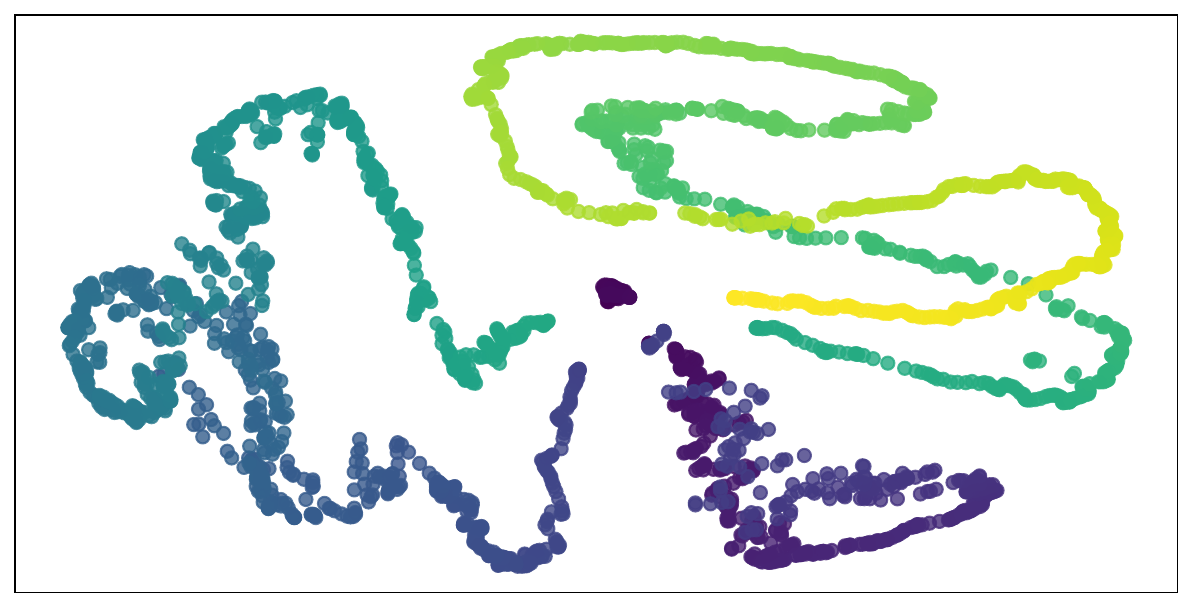}}
            \caption{OliO}
        \end{subfigure} &
        \begin{subfigure}{0.21\textwidth}
            \centering
            \adjustbox{valign=m}{\includegraphics[width=\textwidth]{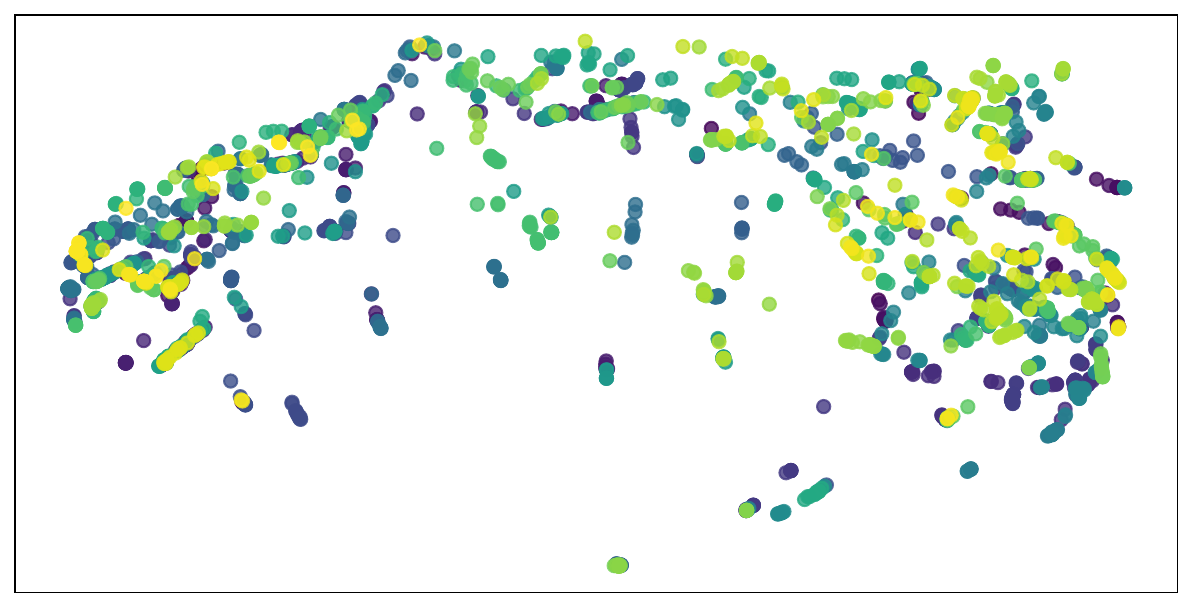}}
            \caption{SimMTM}
        \end{subfigure} &
        \begin{subfigure}{0.21\textwidth}
            \centering
            \adjustbox{valign=m}{\includegraphics[width=\textwidth]{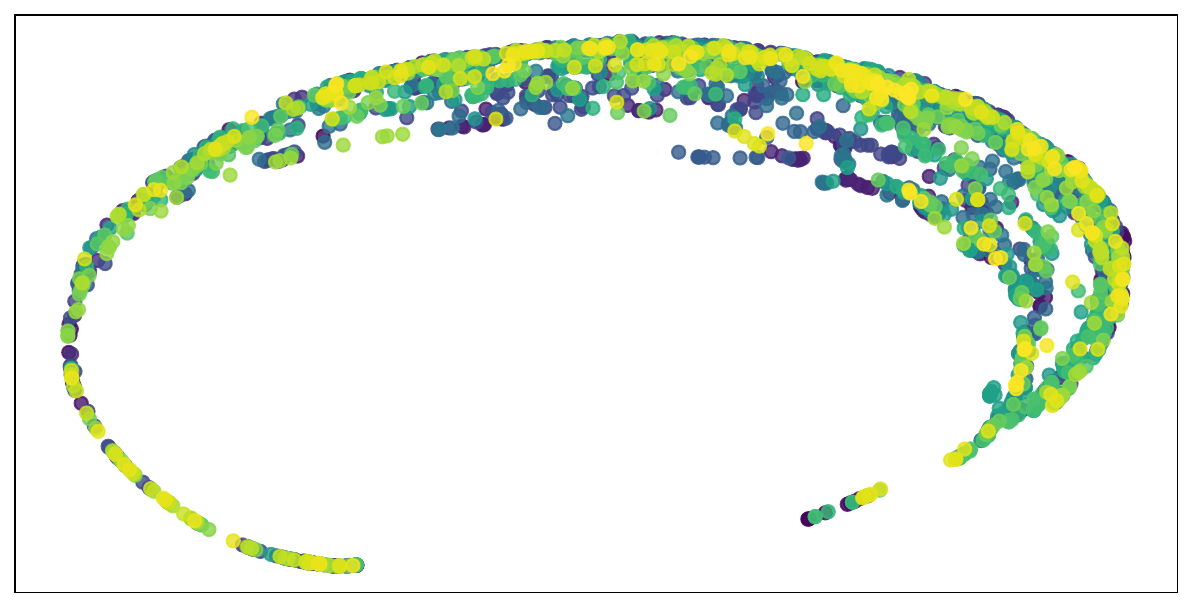}}
            \caption{TS2Vec}
        \end{subfigure} &
        \begin{subfigure}{0.21\textwidth}
            \centering
            \adjustbox{valign=m}{\includegraphics[width=\textwidth]{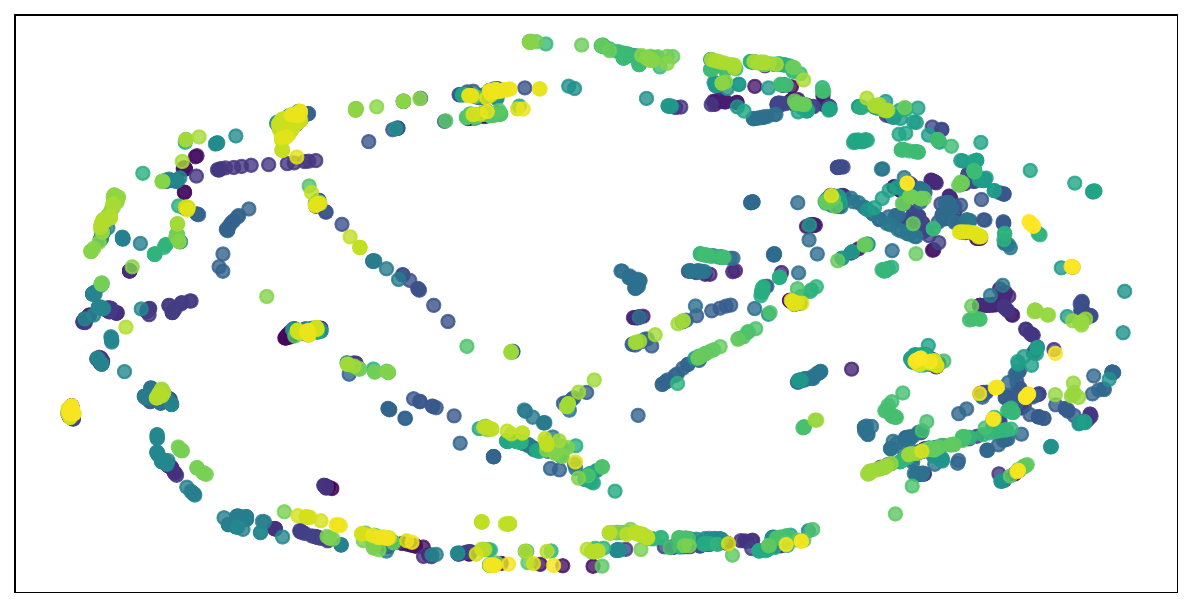}}
            \caption{SoftCLT}
        \end{subfigure} & 
        \\ 
    \end{tabular}
    \caption{Visual comparison of the pseudo output $h$ from the pretrained iTransformer~\cite{iTransformer} backbone using UMAP~\cite{umap} (top row) and PHATE~\cite{phate} (bottom row) on the Weather dataset~\cite{Autoformer}. The color gradient represents the progression of time. While baseline methods (SimMTM~\cite{SimMTM}, TS2Vec~\cite{TS2Vec}, SoftCLT~\cite{SoftCLT}) yield scattered or tightly clustered results that collapse temporal order, OliO maintains a continuous trajectory, effectively capturing the underlying temporal flow.}
    \label{fig:motivation_visual}
\end{figure*}

Fig.~\ref{fig:motivation_visual} provides a qualitative comparison of temporal organization in representations learned on Weather~\cite{Autoformer}. Each point corresponds to a historical window, and its color encodes the window's time index. In SimMTM and SoftCLT, points from different time indices are scattered throughout the embedding, yielding fragmented and intermixed color patterns. TS2Vec forms a denser structure, but colors associated with different time indices remain mixed within it. In contrast, OliO exhibits a smoother progression of colors along its trajectory, suggesting a more temporally ordered organization of the visualized representations. These visualizations complement the forecasting evaluation; additional backbone and dataset visualizations appear in App.~\ref{appendix:latent_visualizations}.

\subsection{Relation to latent dynamics and operator learning}
Latent dynamics and operator-learning methods learn representations for trajectory generation~\cite{LatentODE}, forecasting~\cite{KoopmanNeuralForecaster}, and evolution-operator estimation and analysis~\cite{SelfSupervisedEvolutionOperator}. OliO draws on latent-transition modeling to define a forecasting-oriented pre-training objective: the actual signed lag between independently encoded historical windows supervises bidirectional representation transitions. Its transition module is used to train forecasting representations and is then discarded together with the projection head; downstream prediction is performed by the original backbone. OliO therefore uses established latent-dynamics tools to construct signed temporal supervision for forecasting representations.

\section{Methodology}
\label{sec:methodology}

\subsection{Problem formulation and notation}
Fig.~\ref{fig:olio_architecture} illustrates the OliO framework. We describe the construction using a univariate time series for simplicity; the same formulation applies to multivariate windows. Scalars are denoted by lowercase letters, vectors by bold lowercase letters, and matrices by uppercase letters. Here, $t$ is absolute time and $\tau$ the nonnegative separation between window starts. Subscripts $\theta$ and $\phi$ denote network parameters, and $fw$ and $bw$ denote the forward and backward directions.

\begin{figure}[t!]
    \centering
    \includegraphics[width=\linewidth]{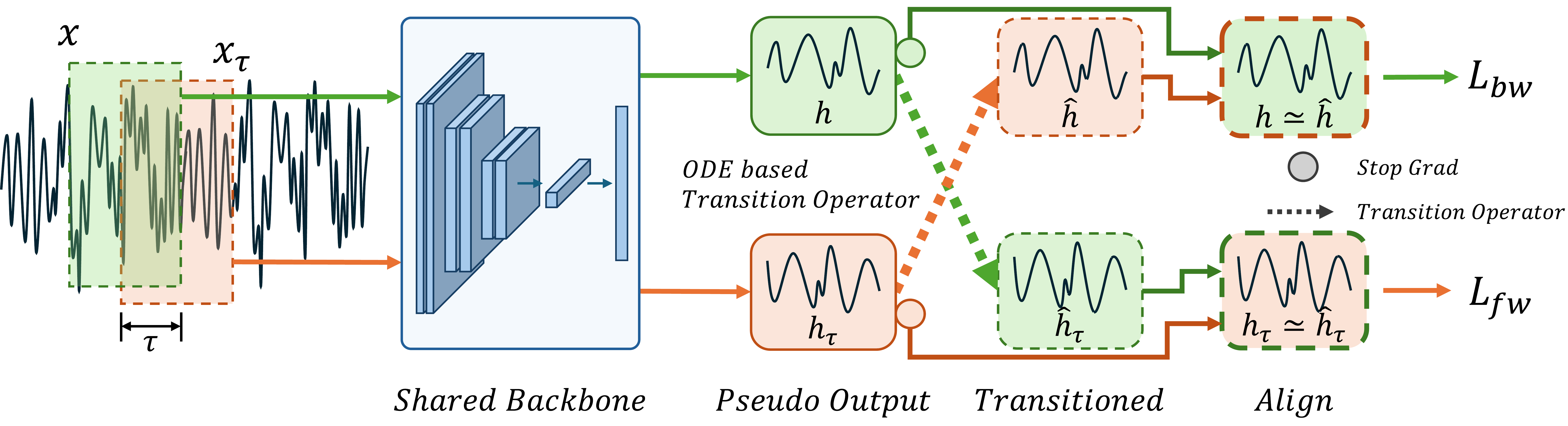}
    \caption{Overall architecture of the OliO framework. (a) Input windows $x$ and $x_\tau$ are encoded by a shared backbone into pseudo-outputs. (b) The forward shift operator aligns $h$ to $h_\tau$ via the matrix exponential $e^{\tau K}$, modeling the transition from past to future. (c) Conversely, the backward shift operator restores the temporal flow from future to past, ensuring bidirectional transition consistency.}
    \label{fig:olio_architecture}
\end{figure}

Consider a discretely sampled time series and a historical window $\bm{x}^{(t)}=[x_t,x_{t+1},\ldots,x_{t+l-1}]^\top\in\mathbb{R}^{l}$ of length $l$. A second window $\bm{x}^{(t+\tau)}$ has the same length and begins $\tau$ time steps later. During standard pre-training, $\tau\in\mathbb{N}$ satisfies $0<\tau<l$, so the two windows overlap. We use $+\tau$ for the transition from the earlier window to the later one and $-\tau$ for the reverse direction.

A shared forecasting backbone $f_\theta$ independently encodes the two windows into final-layer representations $\bm{r}^{(t)},\bm{r}^{(t+\tau)}\in\mathbb{R}^{d}$. A shared projection head $g_\phi$, implemented as a two-layer MLP with ReLU activation, maps them into the pseudo-output space~\cite{pmlr-v119-chen20j}, producing $\bm{h}^{(t)}=g_\phi(\bm{r}^{(t)})$ and $\bm{h}^{(t+\tau)}=g_\phi(\bm{r}^{(t+\tau)})$ in $\mathbb{R}^{d_h}$. OliO transitions each source pseudo-output according to the signed lag before aligning it with the independently encoded pseudo-output of the other window. In each direction, the target remains untransformed.

\subsection{ODE-based transition operator}
We derive the matrix transition from a homogeneous linear ordinary differential equation:
\begin{equation}
    \frac{d\bm{h}}{d\tau}=K\bm{h},
\end{equation}
where $K\in\mathbb{R}^{d_h\times d_h}$ is a learnable generator. The corresponding homogeneous transition of a source pseudo-output $\bm{h}$ over a lag $\tau$ is $e^{\tau K}\bm{h}$, providing a closed-form transition without an iterative nonlinear ODE solver. The ODE describes latent transport without assuming that observations follow a linear or continuous-time process. Training evaluates it at integer lags.

OliO augments this matrix transition with a shared learnable offset $\bm{v}\in\mathbb{R}^{d_h}$, scaled by the signed lag. The forward and backward transitions are
\begin{align}
    \hat{\bm{h}}^{(t+\tau)}
    &=\mathcal{T}_{+\tau}(\bm{h}^{(t)})
    =e^{\tau K}\bm{h}^{(t)}+\tau\bm{v},
    \label{eq:fwd_shift}\\
    \hat{\bm{h}}^{(t)}
    &=\mathcal{T}_{-\tau}(\bm{h}^{(t+\tau)})
    =e^{-\tau K}\bm{h}^{(t+\tau)}-\tau\bm{v}.
    \label{eq:bwd_shift}
\end{align}
Both directions share the same generator $K$ and offset $\bm{v}$; reversing the temporal direction changes the sign of the lag in both terms. The matrix component is ODE-derived, while the shared signed offset supplies additional lag-dependent displacement. The transitioned source $\hat{\bm{h}}^{(t+\tau)}$ is aligned with the later target $\bm{h}^{(t+\tau)}$, and $\hat{\bm{h}}^{(t)}$ is aligned with the earlier target $\bm{h}^{(t)}$.

\subsection{Strictly upper-triangular constraint}
We constrain the generator to be strictly upper triangular, with $k_{ij}=0$ for $i\ge j$. This makes $K$ nilpotent: there exists a smallest positive integer $q\le d_h$ such that $K^q=0$. Nilpotency provides an exact finite-polynomial realization of the matrix transition in both temporal directions.

\begin{theorem}[Finite-polynomial transition structure]
\label{thm:polynomial_distance}
Let $K\in\mathbb{R}^{d_h\times d_h}$ be strictly upper triangular with nilpotency index $q$. For every nonnegative integer lag $\tau\in\mathbb{N}$,
\begin{equation}
    e^{\pm\tau K}=\sum_{j=0}^{q-1}\frac{(\pm\tau)^jK^j}{j!}.
    \label{eq:finite_transition}
\end{equation}
The matrix $e^{\pm\tau K}$ is upper triangular with unit diagonal. For any fixed $\bm{h}$ and $\bm{v}$, the complete transition $\mathcal{T}_{\pm\tau}(\bm{h})=e^{\pm\tau K}\bm{h}\pm\tau\bm{v}$ is a vector-valued polynomial in $\tau$ of degree at most $\max\{q-1,1\}$.
\end{theorem}

The finite expansion follows because every power $K^j$ with $j\ge q$ vanishes; App.~\ref{sec:proofs} provides the derivation. The expansion can contain higher-order terms while excluding exponential dependence on lag. Temporal direction is specified by the sign of $\pm\tau$, whereas triangularity determines the algebraic structure across latent coordinates. Polynomial dependence alone guarantees neither bounded magnitudes nor global distance monotonicity.

\subsection{Computational complexity analysis}
\label{sec:complexity_analysis}
The finite expansion in Eq.~\eqref{eq:finite_transition} can be evaluated directly on a representation vector. Retaining terms through order $p$ requires $p$ matrix--vector products, giving a cost of $O(pd_h^2)$, plus $O(d_h)$ for the additive offset. App.~\ref{sec:complexity_proof} gives the matrix- and vector-wise evaluations using Horner's method.

For exact evaluation, $p=q-1$, so the vector-wise cost is $O((q-1)d_h^2+d_h)$ and remains cubic in the worst case $q=d_h$. Using a fixed lower order $p<q-1$ reduces computation but approximates the complete transition. The main OliO formulation uses the complete transition; lower-order approximations are evaluated separately in Sec.~\ref{sec:ablation_study}. Their accuracy and measured transition-computation speed reflect a trade-off between retained polynomial terms and computational cost. Since the projection and transition modules are used only during pre-training, downstream forecasting retains the inference architecture and cost of the original backbone.

\subsection{Self-supervised objective}
OliO learns the transition relation by aligning each transitioned pseudo-output with its independently encoded target. This objective uses the observed historical windows as supervision without mining negative pairs. Following non-contrastive representation alignment~\cite{SimSiam,BYOL,AliO}, we apply stop-gradient, denoted by $sg(\cdot)$, to the target in each direction. The directional losses are
\begin{align}
    \mathcal{L}_{fw}
    &=1-\operatorname{CosSim}\!\left(
        sg(\bm{h}^{(t+\tau)}),\hat{\bm{h}}^{(t+\tau)}
    \right),\\
    \mathcal{L}_{bw}
    &=1-\operatorname{CosSim}\!\left(
        sg(\bm{h}^{(t)}),\hat{\bm{h}}^{(t)}
    \right).
\end{align}
The total objective averages the two directions:
\begin{equation}
    \mathcal{L}_{\mathrm{OliO}}
    =\frac{1}{2}\left(\mathcal{L}_{fw}+\mathcal{L}_{bw}\right).
    \label{eq:olio_loss}
\end{equation}
Each window serves as a source in one direction and a stopped-gradient target in the other. Stop-gradient breaks gradient symmetry and empirically mitigates representation collapse~\cite{SimSiam,BYOL,AliO}. 

Through this objective, the backbone and projection head learn representations whose transitions are conditioned on the actual temporal displacement between the windows. After pre-training, the projection head and transition module are discarded, and downstream forecasting proceeds through the original backbone and its forecasting head.

\section{Experiments}
\label{sec:experiment}
\begin{table}[!htb]
\centering
\caption{Long-term forecasting results for PatchTST~\cite{PatchTST} and iTransformer~\cite{iTransformer} backbones. Comparison between random initialization (Rand.), OliO (ours), and state-of-the-art SSL methods. TS2Vec~\cite{TS2Vec} is denoted as TS2V, SimMTM~\cite{SimMTM} as MTM, SoftCLT~\cite{SoftCLT} as SCLT, TimeSiam~\cite{TimeSiam} as Siam, and TimesURL~\cite{TimesURL} as URL. Best results are in \textbf{bold} and the second best are \underline{underlined}.}
\label{tab:results_patchtst_itransformer}
\setlength{\tabcolsep}{2pt}
\begin{adjustbox}{max width=\linewidth}
\begin{tabular}{llcccccccccccccc}
\toprule
\multicolumn{2}{c}{\multirow{3}{*}{Dataset}} & \multicolumn{7}{c}{PatchTST} & \multicolumn{7}{c}{iTransformer} \\
\cmidrule(lr){3-9} \cmidrule(lr){10-16}
& & Rand. & \textbf{OliO} & TS2V & MTM & SCLT & Siam & URL & Rand. & \textbf{OliO} & TS2V & MTM & SCLT & Siam & URL \\
\midrule
\multirow{4}{*}{\rotatebox{90}{ETTh1}} & 96  & .3735 & .3716 & \underline{.3708} & .3985 & .3754 & .3871 & \underline{.3708} & \textbf{.3903} & .3969 & .4027 & .3993 & .4040 & \underline{.3914} & .3924 \\
& 192 & .4094 & .4082 & \underline{.4080} & .4299 & .4095 & .4192 & \textbf{.4062} & .4448 & \underline{.4442} & .4507 & .4490 & .4540 & \textbf{.4408} & .4448 \\
& 336 & .4343 & \textbf{.4252} & \underline{.4311} & .4473 & .4357 & .4380 & .4329 & .4884 & \underline{.4865} & .4875 & .4927 & .4932 & \textbf{.4824} & .4868 \\
& 720 & .4582 & \textbf{.4356} & \underline{.4483} & .4560 & .4589 & .4547 & .4520 & .5120 & \underline{.4919} & .4978 & .5043 & .5203 & \textbf{.4812} & .5016 \\ \midrule
\multirow{4}{*}{\rotatebox{90}{ETTh2}} & 96  & .2833 & \textbf{.2769} & .2778 & .2783 & .2802 & .2808 & \underline{.2773} & \textbf{.2988} & \underline{.3002} & .3079 & .3181 & .3016 & .3050 & .3037 \\
& 192 & .3638 & \textbf{.3392} & \underline{.3406} & .3431 & .3441 & .3440 & .3410 & \textbf{.3785} & .3811 & .3923 & .4040 & .3871 & \underline{.3794} & .3891 \\
& 336 & .3918 & \textbf{.3599} & .3606 & .3642 & .3643 & .3676 & \underline{.3601} & .4250 & \underline{.4230} & .4333 & .4496 & .4271 & \textbf{.4180} & .4313 \\
& 720 & .3986 & .3934 & \textbf{.3910} & .3949 & .3960 & \underline{.3925} & .3926 & .4357 & \textbf{.4271} & .4373 & .4386 & .4412 & \underline{.4289} & .4398 \\ \midrule
\multirow{4}{*}{\rotatebox{90}{ETTm1}} & 96  & .2859 & \textbf{.2830} & .2860 & .2839 & .2917 & .2948 & \underline{.2837} & .3456 & \textbf{.3291} & .3367 & .3341 & .3368 & .3435 & \underline{.3315} \\
& 192 & .3266 & \textbf{.3242} & .3279 & .3255 & .3331 & .3363 & \underline{.3252} & .3831 & \textbf{.3711} & .3777 & .3746 & .3764 & .3837 & \underline{.3732} \\
& 336 & .3642 & \underline{.3605} & .3668 & \textbf{.3592} & .3692 & .3672 & .3627 & .4205 & \textbf{.4064} & .4138 & \underline{.4083} & .4097 & .4175 & .4124 \\
& 720 & .4140 & \textbf{.4060} & .4162 & \underline{.4120} & .4154 & .4165 & .4126 & .4896 & \textbf{.4730} & .4844 & \underline{.4784} & .4805 & .4839 & .4812 \\ \midrule
\multirow{4}{*}{\rotatebox{90}{ETTm2}} & 96  & .1660 & \textbf{.1636} & .1669 & .1725 & .1694 & .1692 & \underline{.1658} & .1857 & \textbf{.1812} & .1901 & .1836 & .1875 & \underline{.1826} & .1838 \\
& 192 & .2243 & \textbf{.2189} & .2220 & .2278 & .2249 & .2279 & \underline{.2210} & .2528 & \textbf{.2486} & .2548 & \underline{.2488} & .2539 & .2536 & .2517 \\
& 336 & .2770 & \textbf{.2727} & .2760 & .2835 & .2803 & .2802 & \underline{.2752} & .3146 & \underline{.3105} & .3181 & \textbf{.3095} & .3173 & .3203 & .3134 \\
& 720 & \underline{.3670} & \textbf{.3644} & .3684 & .3706 & .3722 & .3672 & .3701 & .4138 & \textbf{.4088} & .4184 & .4131 & .4151 & .4191 & \underline{.4129} \\ \midrule
\multirow{4}{*}{\rotatebox{90}{Weather}} & 96  & .1529 & \textbf{.1507} & .1615 & .1646 & .1540 & \underline{.1511} & .1528 & \underline{.1776} & \textbf{.1756} & .2139 & .2041 & .2156 & .1778 & .1834 \\
& 192 & .1959 & \textbf{.1948} & .2036 & .2043 & .1975 & \underline{.1949} & .1976 & .2263 & \textbf{.2256} & .2591 & .2499 & .2570 & \underline{.2262} & .2314 \\
& 336 & .2478 & \textbf{.2453} & .2528 & .2549 & \underline{.2466} & .2467 & .2467 & .2829 & \underline{.2825} & .3075 & .3008 & .3085 & \textbf{.2819} & .2865 \\
& 720 & .3221 & \textbf{.3193} & .3242 & .3253 & \underline{.3201} & .3207 & .3216 & .3590 & \underline{.3587} & .3774 & .3706 & .3798 & \textbf{.3574} & .3623 \\ \midrule
\multirow{4}{*}{\rotatebox{90}{ECL}}     & 96  & .1356 & \underline{.1297} & .1314 & .1387 & .1325 & .1300 & \textbf{.1287} & \underline{.1503} & \textbf{.1497} & .2611 & .3656 & .1646 & .2849 & .1508 \\
& 192 & .1515 & \underline{.1476} & .1481 & .1541 & .1493 & .1481 & \textbf{.1471} & \textbf{.1649} & .1671 & .2646 & .3655 & .1789 & .2908 & \underline{.1659} \\
& 336 & .1682 & \textbf{.1633} & .1653 & .1707 & .1656 & .1672 & \underline{.1638} & .1815 & \underline{.1800} & .2880 & .3800 & .1982 & .3050 & \textbf{.1797} \\
& 720 & .2064 & \textbf{.2001} & \underline{.2013} & .2026 & .2017 & .2042 & .2123 & .2131 & \textbf{.2107} & .3216 & .4121 & .2394 & .3138 & \underline{.2114} \\ \midrule
\multicolumn{2}{l}{\textbf{Top1 count}} & 0 & \textbf{18} & 1 & 1 & 0 & 0 & 3 & 4 & \textbf{12} & 0 & 1 & 0 & 6 & 1 \\
\multicolumn{2}{l}{\underline{Top2 count}} & 1 & 3 & 6 & 1 & 2 & 3 & \underline{9} & 2 & \underline{9} & 0 & 3 & 0 & 5 & 5 \\ \bottomrule
\end{tabular}
\end{adjustbox}
\end{table}

We evaluate forecasting performance, representation quality, and transferability, then test the signed-lag supervision and transition structure through controlled ablations.

\subsection{Experimental setup}
\textbf{Datasets and backbones.}
We use ETTh1, ETTh2, ETTm1, ETTm2, Weather, and Electricity (ECL)~\cite{Autoformer}, with prediction horizons $\{96,192,336,720\}$. We evaluate three forecasting backbones: PatchTST~\cite{PatchTST}, which uses temporal patches; iTransformer~\cite{iTransformer}, which attends across variable tokens; and the multimodal TimeVLM~\cite{TimeVLM}. Dataset statistics appear in App.~\ref{sec:datasets}.

\textbf{Baselines and evaluation.}
We compare random initialization with TS2Vec~\cite{TS2Vec}, SimMTM~\cite{SimMTM}, SoftCLT~\cite{SoftCLT}, TimeSiam~\cite{TimeSiam}, TimesURL~\cite{TimesURL}, and OliO using the same forecasting backbones. These comparisons evaluate SSL objectives under shared-backbone adaptation; temporal loss components requiring a timestamp axis are omitted when that axis is unavailable, as documented in App.~\ref{sec:implementation_details}. Following TimeSiam, we fully finetune the pretrained backbone using $L_2$ loss. Input lengths follow the backbone configurations, and standard OliO pre-training samples overlapping window pairs. We retain the baseline-specific hyperparameter searches in App.~\ref{sec:implementation_details}. Configurations are selected by validation MSE, fixed, and evaluated on the test split with seeds 2023, 2024, and 2025. Test performance is excluded from model and hyperparameter selection. Unless stated otherwise, results report three-seed mean MSE; App.~\ref{app:detailed_results} provides per-setting MSE and MAE with sample standard deviations. Few-shot results using 1\%, 5\%, and 10\% of the training data are provided in App.~\ref{app:fewshot_results}.

\subsection{Long-term forecasting}
Tables~\ref{tab:results_patchtst_itransformer} and~\ref{tab:results_timevlm} report forecasting MSE for the three backbones, and Fig.~\ref{fig:performance_improvement} summarizes relative gains across horizons. Against random initialization, OliO reduces MSE by up to 8.1\%, 4.8\%, and 9.2\% on PatchTST, iTransformer, and TimeVLM, respectively. On PatchTST, it ranks first in 18 of 24 settings. The improvement varies by dataset and backbone; some SSL methods also underperform random initialization, showing that pre-training benefits depend on the downstream setting.

\begin{table}[!tb]
\begin{minipage}[t]{0.57\linewidth}
\vspace{0pt}
\centering
\captionof{table}{Long-term forecasting results for TimeVLM~\cite{TimeVLM}. TS2Vec is denoted as TS2V, SimMTM as MTM, SoftCLT as SCLT, TimeSiam as Siam, and TimesURL as URL. Best results are in \textbf{bold} and the second best are \underline{underlined}.}
\label{tab:results_timevlm}
\setlength{\tabcolsep}{2pt}
\begin{adjustbox}{max width=\linewidth}
\begin{tabular}{llccccccc}
\toprule
\multicolumn{2}{c}{\multirow{2}{*}{Dataset}} & \multicolumn{7}{c}{TimeVLM} \\
\cmidrule(lr){3-9}
& & Rand. & \textbf{OliO} & TS2V & MTM & SCLT & Siam & URL \\ \midrule
\multirow{4}{*}{\rotatebox{90}{ETTh1}} & 96  & .3704 & .3669 & .3666 & .3666 & \textbf{.3660} & .3779 & .3672 \\
& 192 & .4074 & .4029 & .4040 & .4038 & \textbf{.4017} & .4074 & \underline{.4027} \\
& 336 & .4659 & \underline{.4232} & .4287 & .4341 & \underline{.4232} & .4285 & .4250 \\
& 720 & .4981 & .4587 & .4634 & .4930 & \textbf{.4556} & \underline{.4577} & .4596 \\ \midrule
\multirow{4}{*}{\rotatebox{90}{ETTh2}} & 96  & .2818 & \underline{.2755} & .2766 & .2862 & .2762 & .2777 & \textbf{.2754} \\
& 192 & .3497 & \textbf{.3359} & .3369 & .3409 & .3377 & .3409 & \underline{.3368} \\
& 336 & .3767 & \textbf{.3613} & \underline{.3621} & .3664 & .3624 & .3713 & .3625 \\
& 720 & .4117 & \underline{.4009} & .4069 & .4073 & .4119 & .4218 & \textbf{.3977} \\ \midrule
\multirow{4}{*}{\rotatebox{90}{ETTm1}} & 96  & .2908 & \underline{.2895} & .2896 & .2899 & .2903 & .2940 & \underline{.2895} \\
& 192 & .3354 & \textbf{.3271} & .3310 & .3348 & \underline{.3308} & .3331 & .3309 \\
& 336 & .3676 & \textbf{.3593} & .3617 & .3663 & .3615 & .3619 & \underline{.3605} \\
& 720 & .4209 & \underline{.4096} & .4102 & .4104 & .4102 & .4165 & \textbf{.4081} \\ \midrule
\multirow{4}{*}{\rotatebox{90}{ETTm2}} & 96  & .1690 & .1674 & \underline{.1673} & .1691 & .1688 & .1678 & \underline{.1673} \\
& 192 & .2251 & \underline{.2216} & .2233 & .2238 & \textbf{.2211} & .2227 & .2229 \\
& 336 & .2745 & .2733 & \underline{.2702} & .2728 & \textbf{.2699} & .2716 & .2712 \\
& 720 & .3676 & \textbf{.3521} & .3540 & .3581 & .3544 & .3590 & .3540 \\ \midrule
\multirow{4}{*}{\rotatebox{90}{Weather}} & 96  & .1495 & .1490 & .1490 & .1492 & \textbf{.1488} & .1520 & .1491 \\
& 192 & \underline{.1937} & \underline{.1937} & .1938 & .1946 & .1942 & .1956 & .1938 \\
& 336 & .2445 & \textbf{.2433} & \underline{.2444} & .2451 & .2453 & .2453 & .2445 \\
& 720 & .3180 & \underline{.3148} & \textbf{.3147} & .3163 & .3170 & .3151 & .3149 \\ \midrule
\multirow{4}{*}{\rotatebox{90}{ECL}}     & 96  & .1353 & .1319 & \textbf{.1315} & \underline{.1318} & .1325 & .1346 & .1363 \\
& 192 & .1513 & .1483 & \textbf{.1478} & .1487 & \underline{.1481} & .1533 & .1534 \\
& 336 & .1673 & \textbf{.1643} & .1670 & .1665 & \underline{.1651} & .1692 & .1727 \\
& 720 & .2080 & \textbf{.2044} & .2091 & .2080 & .2077 & .2080 & \underline{.2061} \\ \midrule
\multicolumn{2}{l}{\textbf{Top1 count}} & 0 & \textbf{8} & 3 & 0 & 6 & 0 & 3 \\
\multicolumn{2}{l}{\underline{Top2 count}} & 1 & \underline{8} & 4 & 1 & 4 & 1 & 6 \\ \bottomrule
\end{tabular}
\end{adjustbox}

\end{minipage}\hfill
\begin{minipage}[t]{0.40\linewidth}
\vspace{0pt}
\centering
\includegraphics[width=\linewidth]{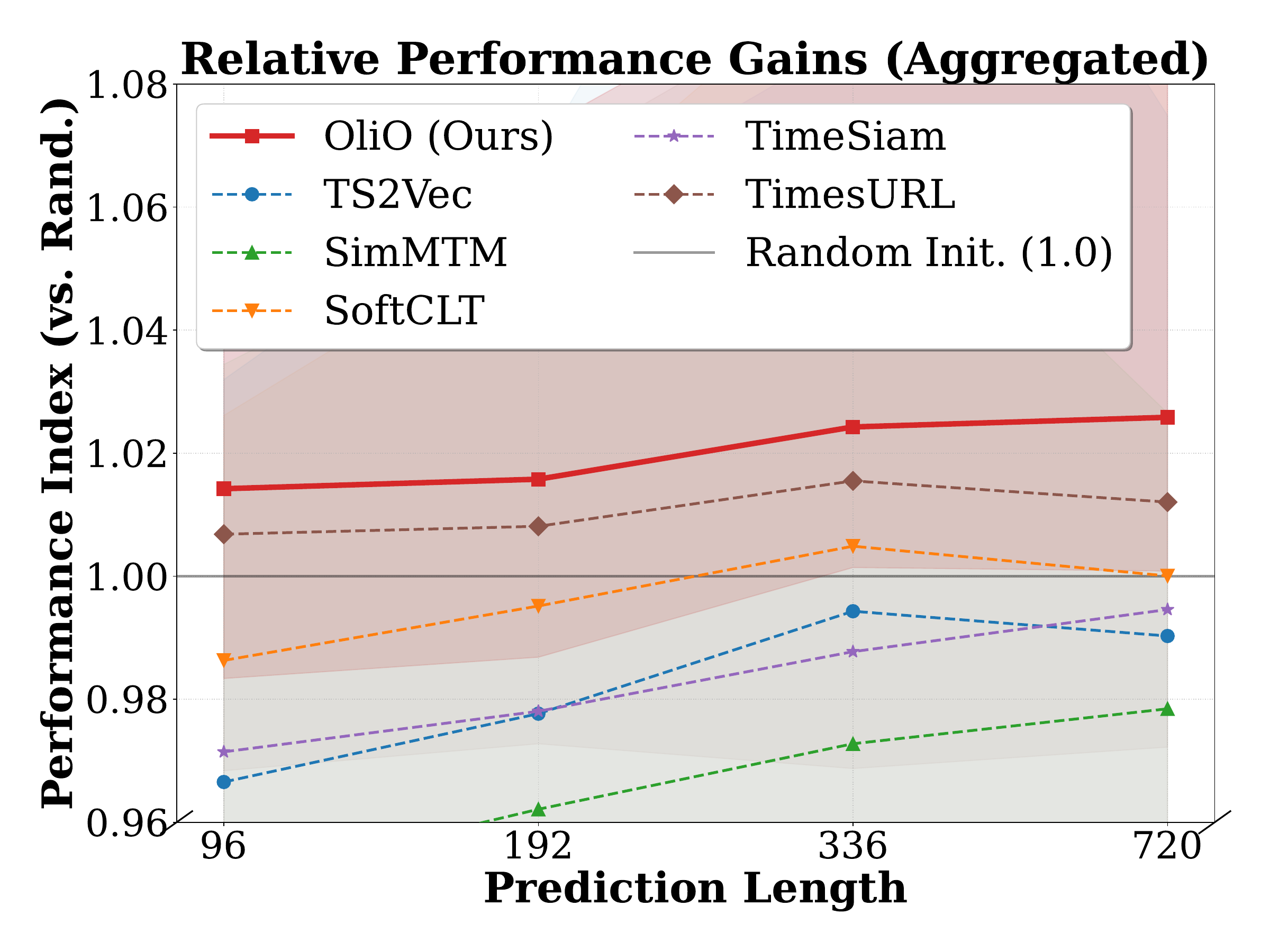}
\captionof{figure}{Relative performance improvement of our proposed method, OliO, compared to state-of-the-art SSL baselines. The performance index is defined as $MSE_{\text{Rand}} / MSE_{\text{Meth.}}$ across varying prediction lengths, where a value of 1.0 represents the random initialization baseline, and \textbf{higher values indicate better performance}. Solid lines indicate the mean index aggregated across three backbone architectures (PatchTST, iTransformer, TimeVLM) and all datasets, with shaded regions representing the min-max performance envelope. Note the truncated Y-axis $[0.96, 1.08]$ for focused visualization of performance gains.}
\label{fig:performance_improvement}

\end{minipage}
\end{table}

\textbf{Effect size and seed variation.}
TimesURL is the strongest baseline by aggregate MSE. Using unrounded three-seed means, OliO outperforms it in 58/72 backbone--dataset--horizon settings, with mean, median, and maximum relative MSE improvements of 0.88\%, 0.55\%, and 5.75\%. Averaging the four horizons within each dataset--backbone group, OliO wins 17/18 groups; a two-sided exact Wilcoxon signed-rank test gives $W=3$, $p=3.81\times10^{-5}$. It also wins all six dataset-level averages. These grouped comparisons reduce dependence among horizon-level results. Table~\ref{tab:representation_summary} summarizes full-finetuning performance and seed variation alongside frozen-backbone results. Each entry $\bar m\pm\bar s$ averages the per-setting three-seed means and sample standard deviations, respectively; $\bar s$ describes typical seed variation, rather than the standard deviation of an aggregated score. Further statistical comparisons are in App.~\ref{app:fullshot_statistics}.

\begin{table}[!htb]
\centering
\caption{OliO--TimesURL comparison under full finetuning and frozen-backbone linear probing.}
\label{tab:representation_summary}
\setlength{\tabcolsep}{5pt}
\begin{adjustbox}{max width=\linewidth}
\begin{tabular}{lcccc}
\toprule
& \multicolumn{2}{c}{Full finetuning} & \multicolumn{2}{c}{Frozen backbone} \\
\cmidrule(lr){2-3}\cmidrule(lr){4-5}
Backbone & TimesURL & OliO & TimesURL & OliO \\
\midrule
PatchTST & 0.2925 $\pm$ 0.0007 & \textbf{0.2898 $\pm$ 0.0007} & \textbf{0.2994 $\pm$ 0.0011} & 0.2997 $\pm$ 0.0008 \\
iTransformer & 0.3300 $\pm$ 0.0017 & \textbf{0.3262 $\pm$ 0.0022} & 0.3412 $\pm$ 0.0007 & \textbf{0.3369 $\pm$ 0.0007} \\
TimeVLM & 0.2918 $\pm$ 0.0018 & \textbf{0.2906 $\pm$ 0.0014} & -- & -- \\
\bottomrule
\end{tabular}
\end{adjustbox}
\end{table}

\subsection{Representation quality and transferability}
\label{sec:representation_quality}
\textbf{Frozen-backbone linear probing.}
To isolate representation quality from full finetuning, we freeze the pretrained backbone and train an identical linear forecasting head. Across PatchTST and iTransformer, six datasets, and four horizons (48 settings), OliO outperforms TimesURL in 35/48 settings and 8/12 dataset--backbone averages. Overall MSE decreases from 0.3203 to 0.3183. As Table~\ref{tab:representation_summary} shows, iTransformer improves from 0.3412 to 0.3369, while PatchTST is comparable (0.2994 versus 0.2997). Thus, the aggregate representation-quality gain does not imply an improvement for every backbone. Dataset-level means and seed variation appear in App.~\ref{app:representation_quality}.

\textbf{Zero-shot forecasting.}
We additionally use LightGTS-4M~\cite{wang2025lightgts} with its original architecture, forecasting settings, and pre-training protocol. A single pretrained checkpoint is evaluated directly on eight downstream datasets disjoint from pre-training, without downstream parameter updates or checkpoint selection. Table~\ref{tab:zeroshot_main} compares OliO, TimesURL, and the default mask-reconstruction objective over 32 dataset--horizon settings. OliO achieves average MSE 0.3034, compared with 0.3114 for mask reconstruction and 0.3219 for TimesURL. Relative to mask reconstruction, average MAE decreases from 0.3469 to 0.3391. The MSE gains include Electricity, Solar, and Traffic, while Weather favors mask reconstruction. These results extend the evaluation beyond full finetuning and complement the frozen-backbone comparison. Complete horizon-level results are in App.~\ref{app:zero_shot}.

\begin{table}[!htb]
\centering
\caption{Zero-shot forecasting results with the LightGTS-4M time-series foundation model~\cite{wang2025lightgts}. Mask Recon. denotes its default mask-reconstruction pre-training. Values average four forecasting horizons. Best results are in \textbf{bold}, and tied best values are \underline{underlined}. Lower MSE is better.}
\label{tab:zeroshot_main}
\setlength{\tabcolsep}{8pt}
\begin{tabular}{lccc}
\toprule
Dataset & TimesURL & Mask Recon. & OliO \\
\midrule
ETTh1 & 0.4096 & 0.3958 & \textbf{0.3937} \\
ETTh2 & 0.3501 & 0.3474 & \textbf{0.3420} \\
ETTm1 & 0.3461 & 0.3335 & \textbf{0.3329} \\
ETTm2 & 0.2566 & \underline{0.2461} & \underline{0.2461} \\
Weather & 0.2357 & \textbf{0.2134} & 0.2182 \\
ECL & 0.2556 & 0.2476 & \textbf{0.2325} \\
Solar & 0.1679 & 0.1712 & \textbf{0.1612} \\
Traffic & 0.5540 & 0.5364 & \textbf{0.5010} \\
\midrule
Average & 0.3219 & 0.3114 & \textbf{0.3034} \\
\bottomrule
\end{tabular}
\end{table}

\subsection{Ablation studies}
\label{sec:ablation_study}
\textbf{Actual signed temporal displacement.}
Table~\ref{tab:signed_lag} tests which aspects of the signed lag provide useful supervision on PatchTST and iTransformer across 48 settings. Each control changes only the supplied lag relative to the original formulation.
\emph{Fixed} retains the correct direction but replaces the pair-specific magnitude with a constant. Its 1.10\% higher MSE indicates that the actual separation carries useful information beyond temporal order.
\emph{Random Magnitude} retains the correct sign but supplies a mismatched magnitude. The 2.57\% increase shows that varying the conditioning value alone does not recover the benefit of matching it to the actual window separation.
\emph{Random Direction} preserves the actual magnitude and opposite signs between branches, but randomizes their correspondence to earlier--later order. Its 1.82\% increase isolates the value of assigning signs according to the true temporal direction, beyond simply using opposite signs in the two branches.
\emph{Unsigned} supplies the correct magnitude with a positive sign in both branches. Its 4.86\% increase, the largest among these controls, supports distinguishing forward from backward transport. Together, the controls support conditioning the transition on both the actual separation and its temporal direction.
With the transition module and alignment objective held fixed, these results indicate that the actual temporal relation between windows provides useful supervision for forecasting.

\textbf{Exponential transport and additive bias.}
Table~\ref{tab:transition_parameterization} examines the contributions of matrix transport and the additive offset over the same 48 settings.
\emph{Bias-only} removes matrix transport, applying the same displacement to every representation at a given signed lag. Its 3.83\% higher MSE supports allowing the displacement to depend on the source representation.
\emph{K-only} retains exponential transport but removes the offset. The 1.40\% increase indicates that a learned translation complements the multiplicative transformation of the source state; neither component alone matches the complete transition.
\emph{MLP} tests whether a generic nonlinear mapping can replace the explicit transition structure. Its 3.13\% increase supports the proposed structured parameterization over this nonlinear alternative.
\emph{Spec-K} replaces the strictly upper-triangular generator with a full generator using spectral normalization. It remains numerically stable but has 1.18\% higher MSE. Thus, the performance gap persists when both variants train without numerical divergence, supporting a contribution of the generator structure beyond avoiding numerical failure. Dataset-level results are in App.~\ref{app:transition_parameterization}.

\begin{table}[!htb]
\centering

\begin{minipage}[t]{0.52\linewidth}
\centering
\vspace{0pt}
\captionof{table}{
Effect of signed temporal displacement over 48 settings.
}

\label{tab:signed_lag}
\setlength{\tabcolsep}{2pt}
\begin{adjustbox}{max width=\linewidth}
\begin{tabular}{lccc}
\toprule
Variant & Avg. MSE & $\Delta$ MSE & Wins \\
\midrule
Original          & \textbf{0.3080} & --       & -- \\
Fixed             & 0.3114          & +1.10\%  & 40/48 \\
Random Direction  & 0.3136          & +1.82\%  & 40/48 \\
Random Magnitude  & 0.3159          & +2.57\%  & 41/48 \\
Unsigned          & 0.3230          & +4.86\%  & 45/48 \\
\bottomrule
\end{tabular}
\end{adjustbox}
\end{minipage}
\hfill
\begin{minipage}[t]{0.45\linewidth}
\centering
\vspace{0pt}
\captionof{table}{
Ablation of alternative transition parameterizations over 48 settings.
}

\label{tab:transition_parameterization}
\setlength{\tabcolsep}{2pt}
\begin{adjustbox}{max width=\linewidth}
\begin{tabular}{lccc}
\toprule
Variant & Avg. MSE & $\Delta$ MSE & Wins \\
\midrule
\textbf{OliO} & \textbf{0.3080} & -- & -- \\
Spec-K        & 0.3116 & +1.18\% & 39/48 \\
K-only        & 0.3123 & +1.40\% & 42/48 \\
MLP           & 0.3176 & +3.13\% & 45/48 \\
Bias-only     & 0.3198 & +3.83\% & 47/48 \\
\bottomrule
\end{tabular}
\end{adjustbox}
\end{minipage}

\end{table}

\textbf{Bidirectionality, approximation order, and efficiency.}
Table~\ref{tab:ablation_mean} evaluates the original three-backbone, six-dataset setting (72 horizon-level configurations). Bidirectional learning achieves average MSE 0.3022, versus 0.3240 forward-only and 0.3245 backward-only, and improves all six dataset averages. This supports supervising both directions without requiring exact inverse composition. Table~\ref{tab:ablation_order} compares the generator constraint and retained order $p$. The unconstrained full generator diverges in these experiments, whereas low-order polynomial approximations remain finite but can lose accuracy, especially on ECL. Strict upper triangularity guarantees a finite expansion, rather than bounded magnitudes for arbitrary lags. Fig.~\ref{fig:speed_up} shows up to $3.8\times$ faster transition computation with lower orders during pre-training; downstream inference uses the unchanged forecasting backbone. App.~\ref{app:ablation_study} provides backbone-specific results and orders $p=1,\ldots,5$.

\begin{table}[!htb]
\centering
\begin{minipage}[t]{0.30\linewidth}
\centering
\vspace{0pt}
\captionof{table}{\textbf{Effect of bidirectional modeling.} Average MSE across all backbones.}
\label{tab:ablation_mean}
\vspace{3pt}
\setlength{\tabcolsep}{3pt}
\begin{adjustbox}{max width=\linewidth}
\begin{tabular}{lccc}
\toprule
Dataset & \textbf{Bi} & Fwd & Bwd \\
\midrule
ETTh1   & \textbf{.4260} & .4485 & .4465 \\
ETTh2   & \textbf{.3562} & .3625 & .3618 \\
ETTm1   & \textbf{.3616} & .3938 & .3930 \\
ETTm2   & \textbf{.2653} & .2714 & .2717 \\
Weather & \textbf{.2378} & .2495 & .2452 \\
ECL     & \textbf{.1664} & .2185 & .2285 \\
\bottomrule
\end{tabular}
\end{adjustbox}
\end{minipage}
\hfill
\begin{minipage}[t]{0.40\linewidth}
\centering
\vspace{0pt}
\captionof{table}{\textbf{Structural constraints and order $p$.} Evaluation (MSE) of transition matrix constraints and Taylor orders.}
\label{tab:ablation_order}
\vspace{3pt}
\setlength{\tabcolsep}{2pt}
\begin{adjustbox}{max width=\linewidth}
\begin{tabular}{lccccc}
\toprule
Dataset & Full & $p=1$ & $p=2$ & $p=3$ & \textbf{Exact} \\
\midrule
ETTh1   & NaN & .4434 & .4315 & .4315 & \textbf{.4260} \\
ETTh2   & NaN & .3661 & .3624 & .3606 & \textbf{.3562} \\
ETTm1   & NaN & .3752 & .3695 & .3686 & \textbf{.3616} \\
ETTm2   & NaN & .2674 & .2661 & .2678 & \textbf{.2653} \\
Weather & NaN & .2470 & .2480 & .2382 & \textbf{.2378} \\
ECL     & NaN & .2137 & .2063 & .2029 & \textbf{.1664} \\
\bottomrule
\end{tabular}
\end{adjustbox}
\end{minipage}
\hfill
\begin{minipage}[t]{0.26\linewidth}
\vspace{0pt}
\centering
\includegraphics[width=\linewidth]{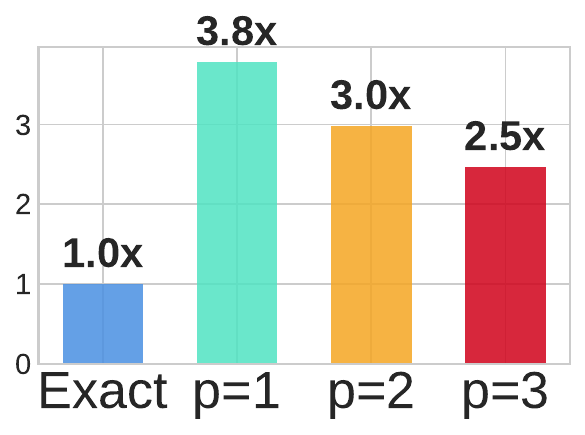}
\captionof{figure}{\textbf{Approximation speedup.} Latency improvement relative to full matrix exponential.}
\label{fig:speed_up}
\end{minipage}
\end{table}

\subsection{Overlap controls and finite-lag behavior}
\textbf{Shared observations.}
Varying overlap produces no monotonic relationship with MSE. Comparing 100\% and 0\% overlap across six representative dataset--backbone settings yields differences within 0.7\% in five cases; 0\% overlap improves over the model without OliO in five cases. These endpoints change both overlap and lag. To isolate shared observations, Table~\ref{tab:matched_masking_main} instead holds the window pair, signed lag, 50\% mask budget, backbone, training schedule, and hyperparameters fixed. We zero either the shared half of the second window or an equal-sized unique half. The aggregate MSE gap is only 0.22\%, with 23 versus 25 lower-MSE settings. Thus, shared observations alone do not explain the gains under this controlled intervention. Full sweeps and endpoint results are in App.~\ref{app:overlap_analysis}; masking details are in App.~\ref{app:matched_masking}.

\textbf{Higher-order terms and distance ordering.}
Table~\ref{tab:higher_order_transition_main} compares $e^{\tau K}$ with $I+\tau K$ using relative Frobenius error. On learned ETTh1 and ETTh2 operators, the discrepancy increases with lag and reaches 1.004 and 0.430 at $\tau=64$, respectively. Hence, the learned finite-lag transport does not generally reduce to its first-order term. Separately, checkpoint measurements over adjacent lags $k=0,\ldots,7$ show complete distance ordering for 88\%--100\% of evaluated representations, with positive mean and median squared-displacement increments at every adjacent lag. This is empirical evidence over the evaluated range, rather than a global monotonicity guarantee. App.~\ref{app:higher_order_transition} gives the approximation definition and additional lags; App.~\ref{app:distance_ordering} details the checkpoint diagnostic.

\begin{table}[!tb]
\begin{minipage}[t]{0.59\linewidth}
\vspace{0pt}
\centering
\captionof{table}{Matched masking at 50\% overlap, averaged over six datasets and four horizons. ``Wins'': Overlap-Mask / Unique-Mask. Lower MSE is better.}
\label{tab:matched_masking_main}
\setlength{\tabcolsep}{3pt}
\begin{adjustbox}{max width=\linewidth}
\begin{tabular}{lcccc}
\toprule
Backbone & Overlap mask & Unique mask & Gap & Wins \\
\midrule
PatchTST & 0.29818 & \textbf{0.29767} & +0.17\% & 9 / 15 \\
iTransformer & 0.35022 & \textbf{0.34932} & +0.26\% & 14 / 10 \\
\midrule
Overall & 0.32420 & \textbf{0.32349} & +0.22\% & 23 / 25 \\
\bottomrule
\end{tabular}
\end{adjustbox}
\end{minipage}\hfill
\begin{minipage}[t]{0.38\linewidth}
\vspace{0pt}
\centering
\captionof{table}{Relative approximation error between the complete transition and its first-order approximation.}
\label{tab:higher_order_transition_main}
\setlength{\tabcolsep}{4pt}
\begin{tabular}{ccc}
\toprule
Lag $\tau$ & ETTh1 & ETTh2 \\
\midrule
1 & 0.018 & 0.000 \\
20 & 0.857 & 0.104 \\
40 & 0.965 & 0.258 \\
64 & 1.004 & 0.430 \\
\bottomrule
\end{tabular}
\end{minipage}
\end{table}

\section{Future work and discussion}
\label{sec:discussion}

The signed-lag controls (Table~\ref{tab:signed_lag}) favor conditioning transitions on the actual temporal separation and direction. The overlap sweeps and matched masking control (Table~\ref{tab:matched_masking_main}) suggest that shared observations alone do not explain the forecasting gains. Together, these findings support learning temporal change by aligning window representations through transitions matched to their observed temporal relation.

The ODE determines OliO's homogeneous matrix transition, while the shared signed offset is parameterized separately. The complete affine map need not be the flow of an autonomous affine ODE. Each lag-specific affine map is invertible, but the opposite-lag transition need not be its inverse. Bidirectional alignment therefore supervises both temporal directions without imposing exact inverse composition. The projection and transition modules provide pre-training supervision and are discarded for downstream forecasting.

Koopman-based methods~\cite{Koopman1931,Lusch2018} motivate representing nonlinear observations in coordinates suited to linear transport. OliO uses a nonlinear projection head for this purpose, but a finite-dimensional learned projection does not guarantee exact linearization of arbitrary dynamics. Its nilpotent generator models finite-lag polynomial transients rather than sustained periodic orbits. A single state-independent generator can also be restrictive when transitions vary across states or regimes; state-dependent transport is a possible extension.

Beyond forecasting, signed-lag supervision may be useful for other time-series tasks. SoftCLT~\cite{SoftCLT} has demonstrated the value of distance-aware representations for classification, motivating further evaluation of OliO for structural-break detection and activity classification on benchmarks such as UCR~\cite{UCR} and UEA~\cite{UEA}. Such transfer remains an empirical question rather than a consequence of the finite-polynomial transition theorem.

\section{Conclusion}
We introduced OliO, a self-supervised pre-training framework that learns bidirectional transitions between representations of historical windows using their actual signed temporal displacement. Its transition combines an ODE-derived matrix exponential with a shared offset scaled by the signed lag. A strictly upper-triangular generator makes the transition exactly expressible as a finite polynomial in lag in both temporal directions. Across three forecasting backbones and six datasets, OliO improves on TimesURL, the strongest SSL baseline by aggregate MSE, in 58 of 72 settings, with up to 5.75\% lower MSE. Frozen-backbone probing and zero-shot evaluation provide additional evidence of representation quality and transferability. Controlled ablations support using the actual lag magnitude and direction, combining exponential transport with an additive offset, and supervising both temporal directions. The projection and transition modules are discarded after pre-training, so downstream forecasting retains the original backbone's architecture and inference cost. These findings support using observed temporal displacement to guide representation learning for forecasting.

\begin{ack}
This work was supported by the Technology Innovation Program (Project No. RS-2026-25618509 and RS-2026-25620133) funded by the Ministry of Trade, Industry and Energy (MOTIE, Republic of Korea) and the Institute of Information \& Communications Technology Planning \& Evaluation (IITP) grant funded by the Korea government (MSIT) (No. RS-2026-25617720, AI Star Fellowship Support Program).
\end{ack}

\bibliographystyle{plainnat}
\bibliography{references}

\clearpage
\newpage
\appendix
\section{Further latent space visualizations}
\label{appendix:latent_visualizations}

We visualize representations learned with OliO across three forecasting backbones using UMAP and PHATE. Colors indicate the time indices of historical windows, allowing temporal organization to be inspected within each two-dimensional projection. These qualitative views accompany the forecasting and representation-quality evaluations in Secs.~\ref{sec:experiment} and~\ref{sec:representation_quality}.

\subsection{Visualizations for the PatchTST backbone}
\label{appendix:latent_patchtst}
Figs.~\ref{fig:latent_patchtst_etth} and~\ref{fig:latent_patchtst_ettm1} show PatchTST representations on ETT datasets. The PHATE projections display connected trajectories with gradual changes in time-index color, while UMAP provides a complementary view of the representation geometry.
\begin{figure*}[!h]
    \centering
    \setlength{\tabcolsep}{1.0pt} 
    \begin{tabular}{ll}
        \begin{subfigure}{0.48\textwidth}
            \centering
            \includegraphics[width=\textwidth]{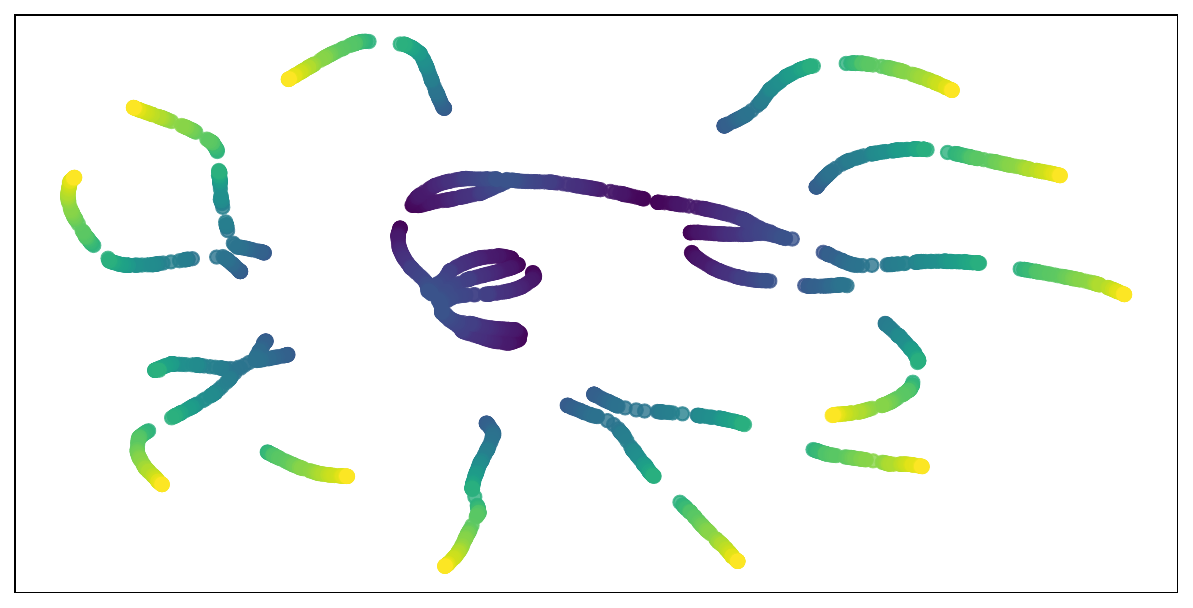}
            \caption{ETTh1 (UMAP)}
        \end{subfigure} &
        \begin{subfigure}{0.48\textwidth}
            \centering
            \includegraphics[width=\textwidth]{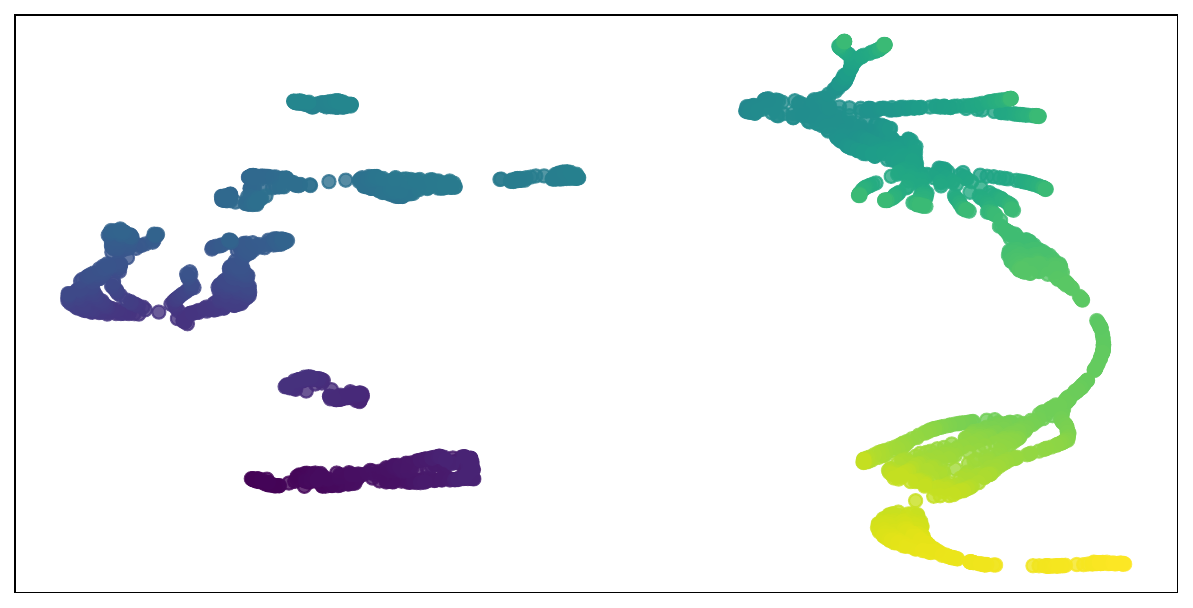}
            \caption{ETTh2 (UMAP)}
        \end{subfigure} \\ [1.5em]
        \begin{subfigure}{0.48\textwidth}
            \centering
            \includegraphics[width=\textwidth]{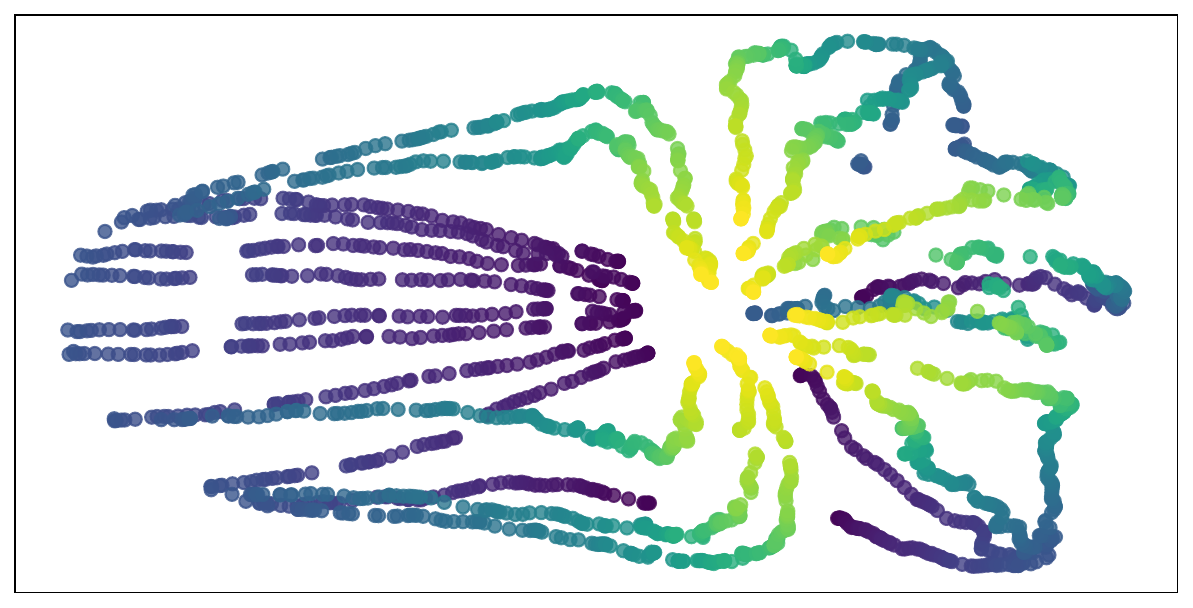}
            \caption{ETTh1 (PHATE)}
        \end{subfigure} &
        \begin{subfigure}{0.48\textwidth}
            \centering
            \includegraphics[width=\textwidth]{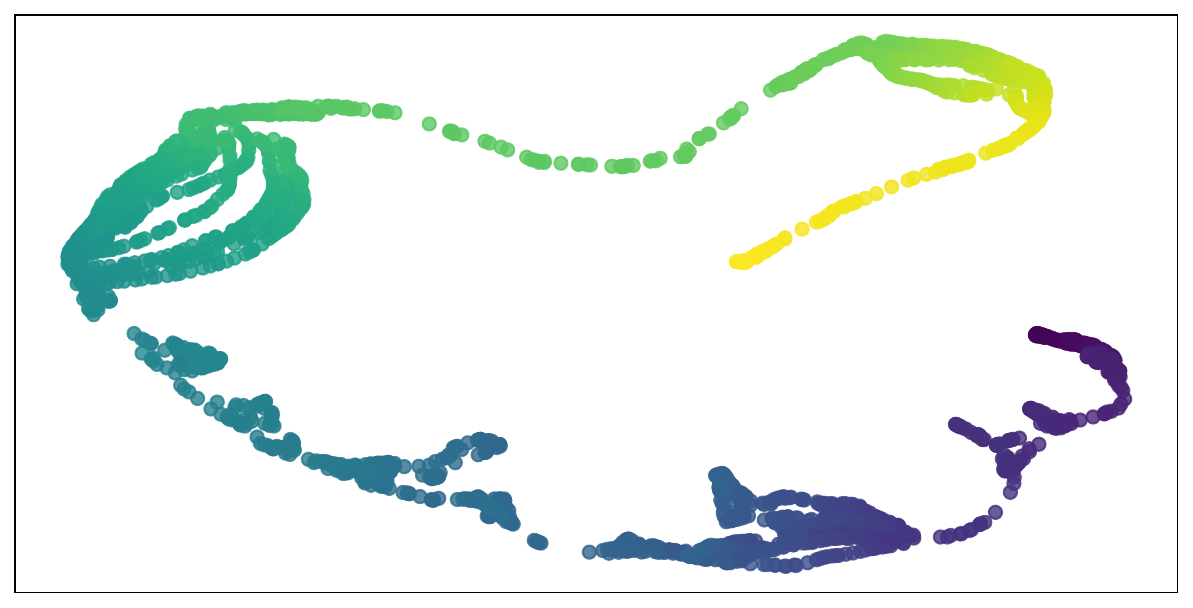}
            \caption{ETTh2 (PHATE)}
        \end{subfigure}
    \end{tabular}
    \caption{Latent space visualizations of the PatchTST backbone on ETTh1 and ETTh2 datasets using UMAP (top) and PHATE (bottom). For ETTh2, we utilize the causal pretraining configuration. The continuous trajectories in the PHATE plots demonstrate that OliO effectively preserves the temporal flow.}
    \label{fig:latent_patchtst_etth}
\end{figure*}

\begin{figure*}[!t]
    \centering
    \setlength{\tabcolsep}{1.0pt} 
    \begin{tabular}{ll}
        \begin{subfigure}{0.48\textwidth}
            \centering
            \includegraphics[width=\textwidth]{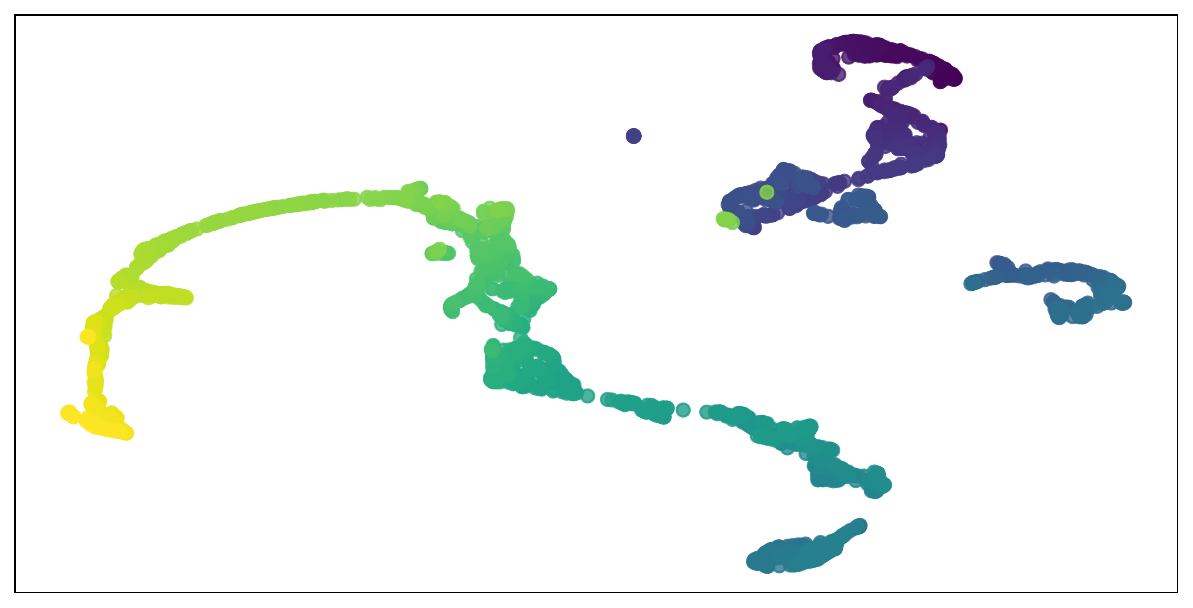}
            \caption{ETTm1 (UMAP)}
        \end{subfigure} &
        \begin{subfigure}{0.48\textwidth}
            \centering
            \includegraphics[width=\textwidth]{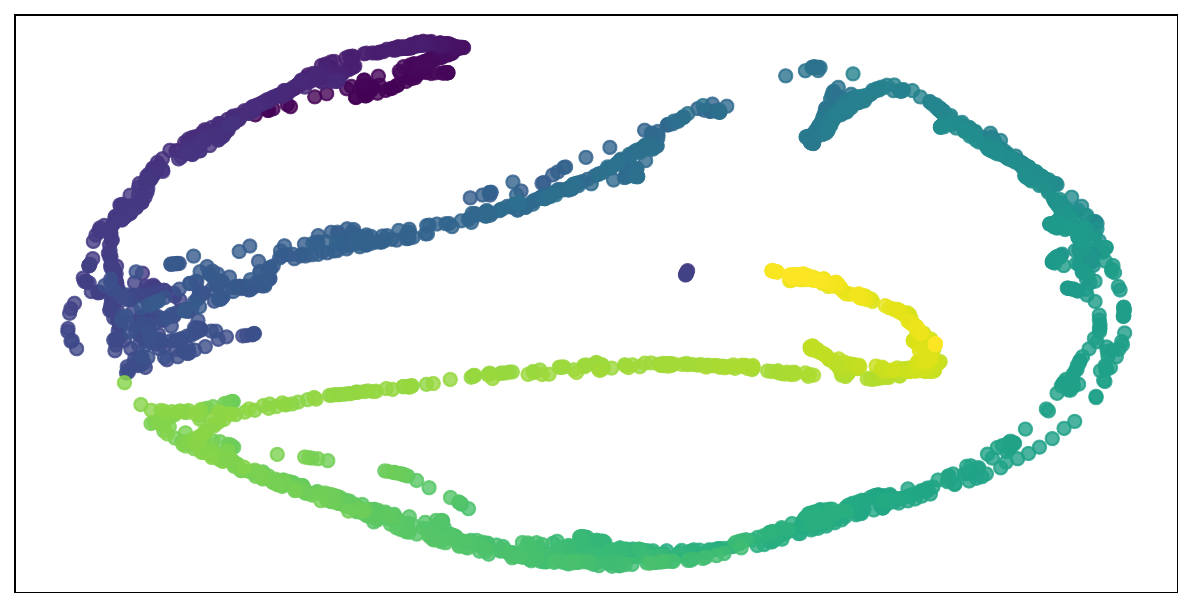}
            \caption{ETTm1 (PHATE)}
        \end{subfigure}
    \end{tabular}
    \caption{Latent space visualization of PatchTST on the ETTm1 dataset, comparing UMAP and PHATE embeddings.}
    \label{fig:latent_patchtst_ettm1}
\end{figure*}

\clearpage

\subsection{Visualizations for the iTransformer backbone}
\label{appendix:latent_itransformer}
Fig.~\ref{fig:latent_itransformer} shows iTransformer representations on ETTm1 and ETTm2. Both projections exhibit trajectories with temporal color progression, with differences in shape and overlap between the datasets and projection methods.
\begin{figure*}[!h]
    \centering
    \setlength{\tabcolsep}{1.0pt} 
    \begin{tabular}{ll}
        \begin{subfigure}{0.48\textwidth}
            \centering
            \includegraphics[width=\textwidth]{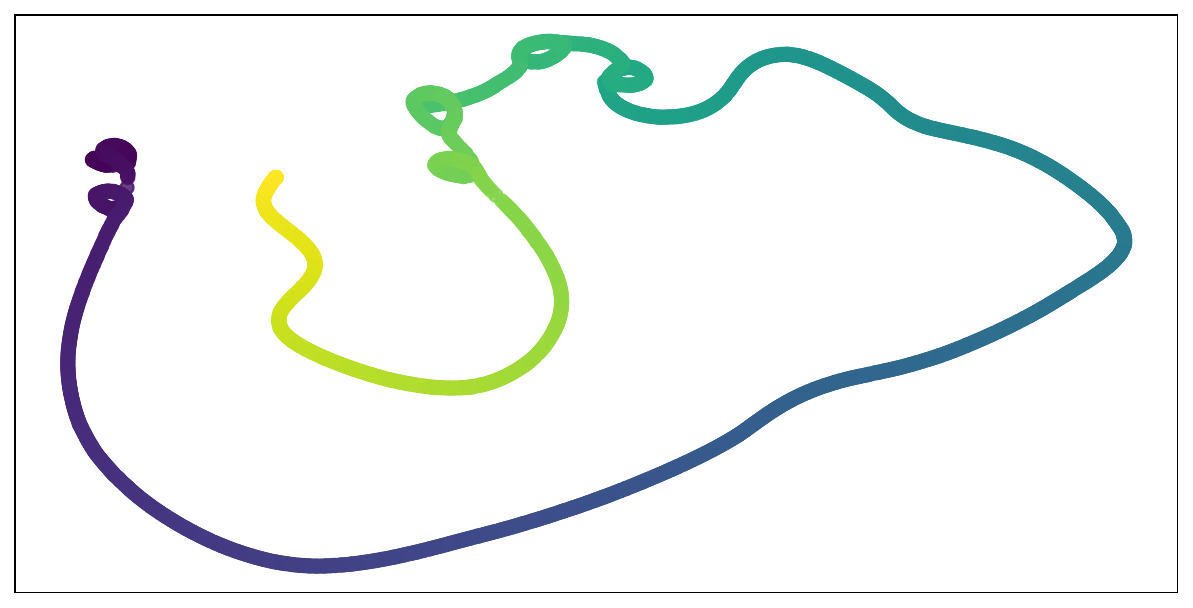}
            \caption{ETTm1 (UMAP)}
        \end{subfigure} &
        \begin{subfigure}{0.48\textwidth}
            \centering
            \includegraphics[width=\textwidth]{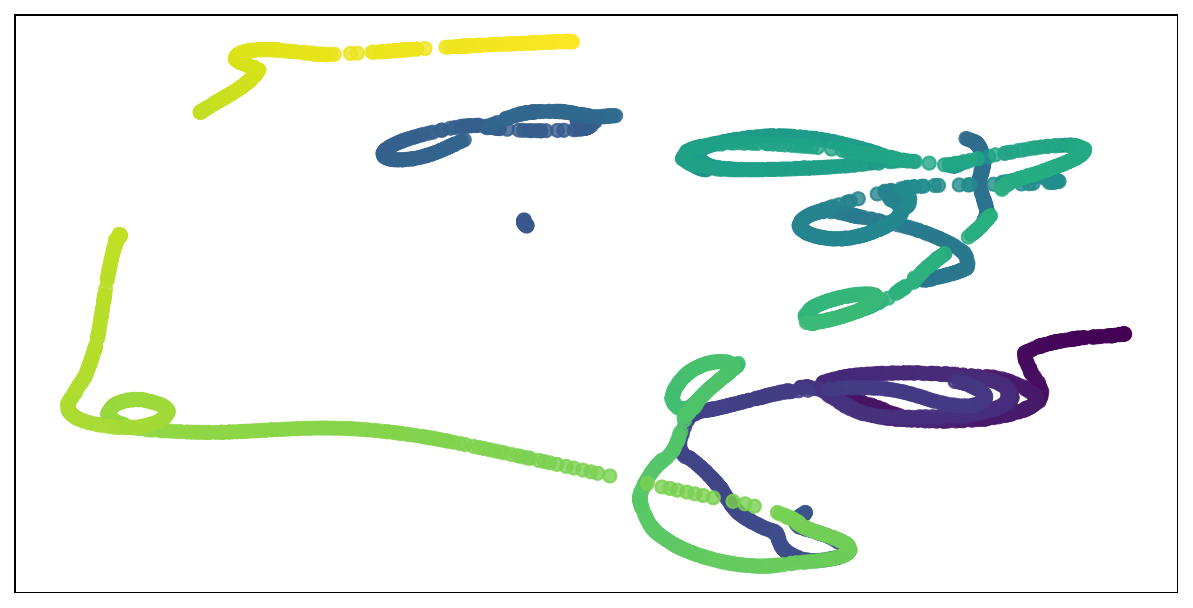}
            \caption{ETTm2 (UMAP)}
        \end{subfigure} \\ [1.5em]
        \begin{subfigure}{0.48\textwidth}
            \centering
            \includegraphics[width=\textwidth]{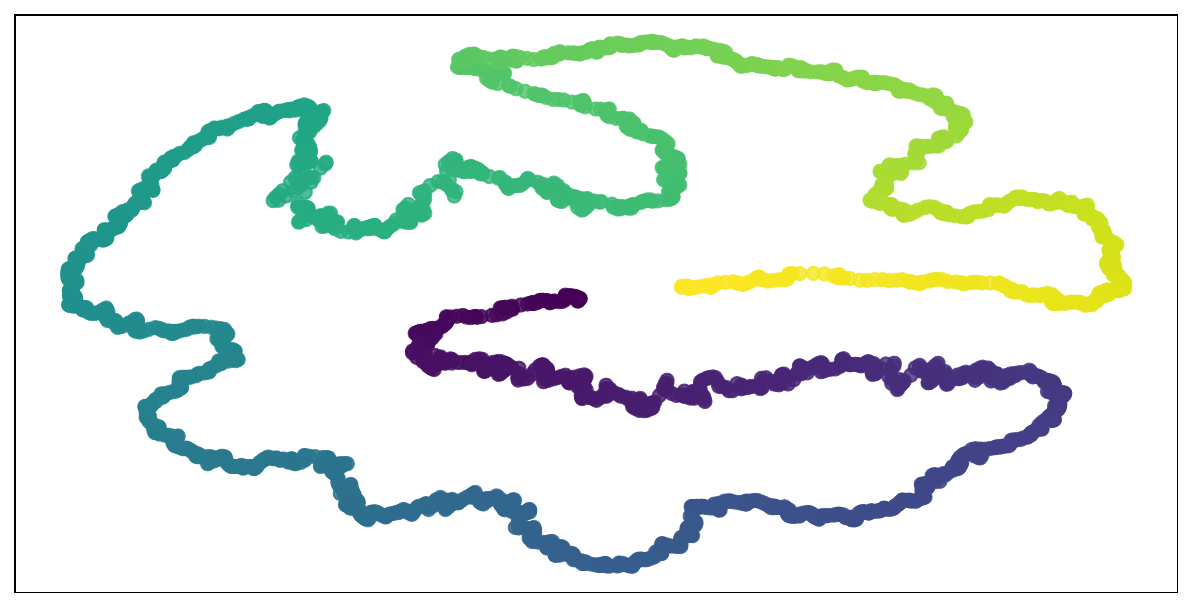}
            \caption{ETTm1 (PHATE)}
        \end{subfigure} &
        \begin{subfigure}{0.48\textwidth}
            \centering
            \includegraphics[width=\textwidth]{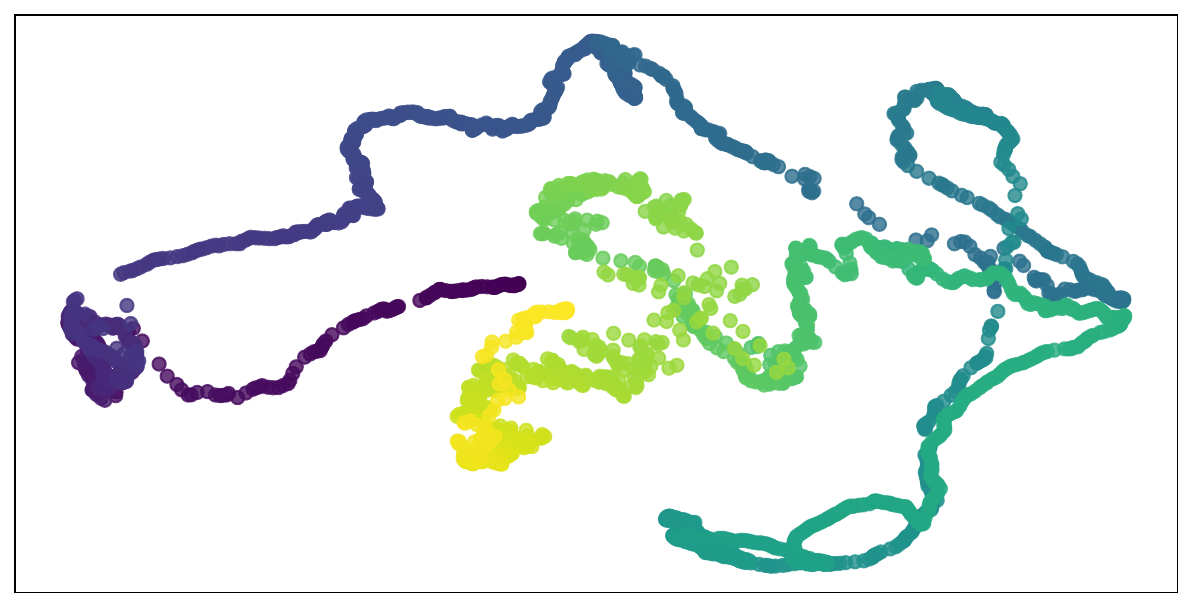}
            \caption{ETTm2 (PHATE)}
        \end{subfigure}
    \end{tabular}
    \caption{Latent space visualizations of iTransformer on ETTm1 and ETTm2 datasets.}
    \label{fig:latent_itransformer}
\end{figure*}

\clearpage

\subsection{Visualizations for the TimeVLM backbone}
\label{appendix:latent_timevlm}
Figs.~\ref{fig:latent_timevlm_etth}--\ref{fig:latent_timevlm_weather} show TimeVLM representations on ETT and Weather. Several projections contain overlapping regions and less distinct temporal color progression than the PatchTST and iTransformer examples. These panels document the geometry of the learned multimodal representations; the forecasting results provide the corresponding quantitative evaluation.
\begin{figure*}[!h]
    \centering
    \setlength{\tabcolsep}{1.0pt} 
    \begin{tabular}{ll}
        \begin{subfigure}{0.48\textwidth}
            \centering
            \includegraphics[width=\textwidth]{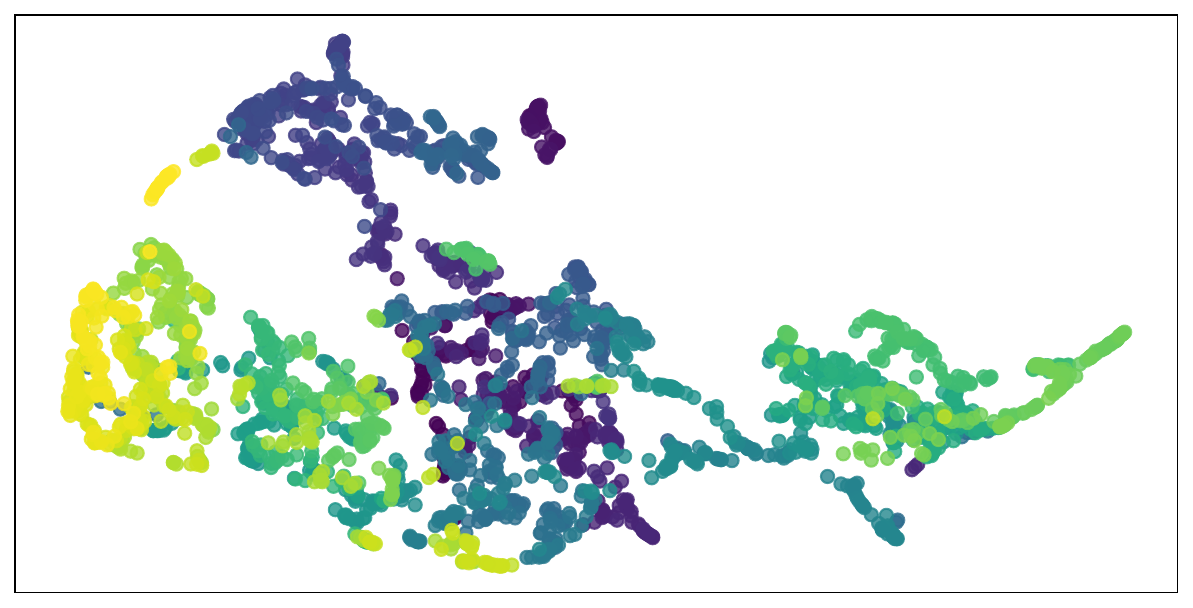}
            \caption{ETTh1 (UMAP)}
        \end{subfigure} &
        \begin{subfigure}{0.48\textwidth}
            \centering
            \includegraphics[width=\textwidth]{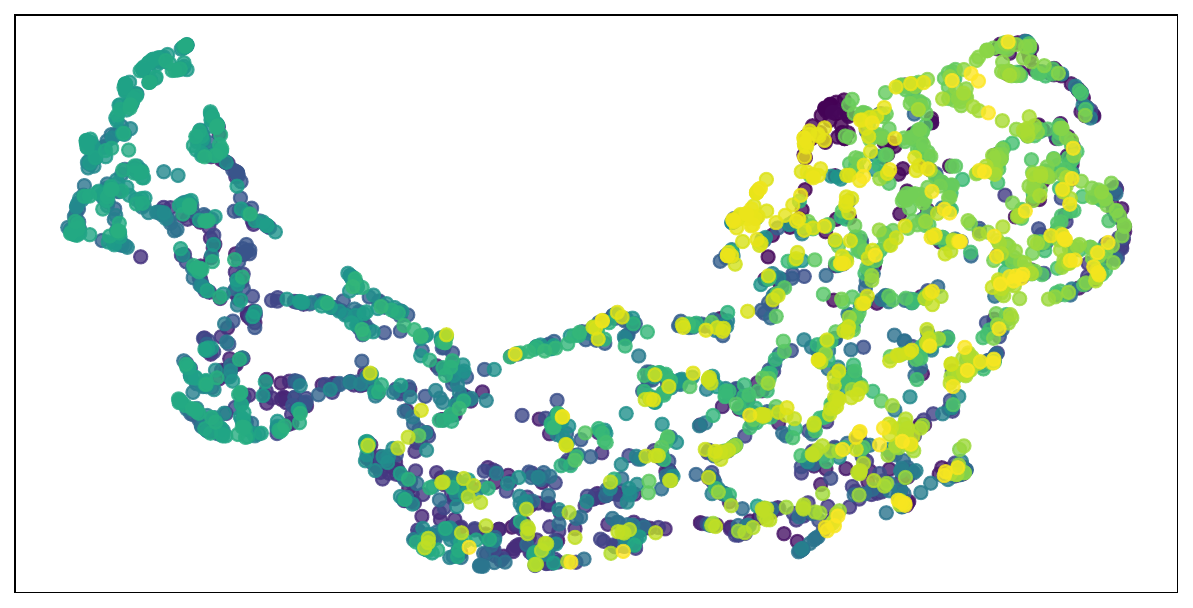}
            \caption{ETTh2 (UMAP)}
        \end{subfigure} \\ [1.5em]
        \begin{subfigure}{0.48\textwidth}
            \centering
            \includegraphics[width=\textwidth]{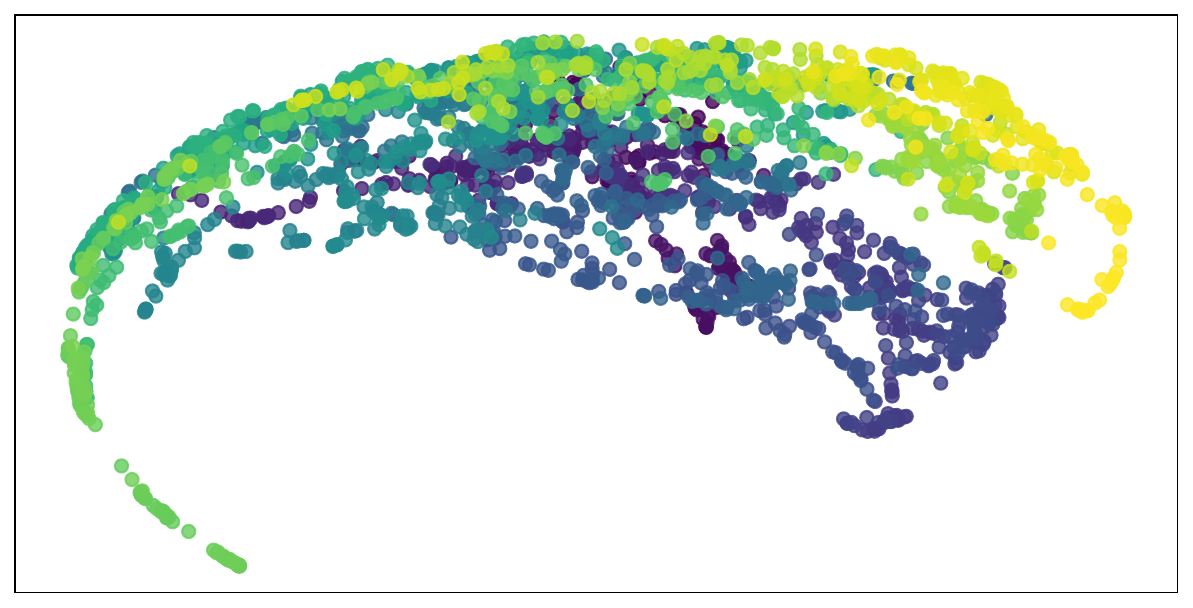}
            \caption{ETTh1 (PHATE)}
        \end{subfigure} &
        \begin{subfigure}{0.48\textwidth}
            \centering
            \includegraphics[width=\textwidth]{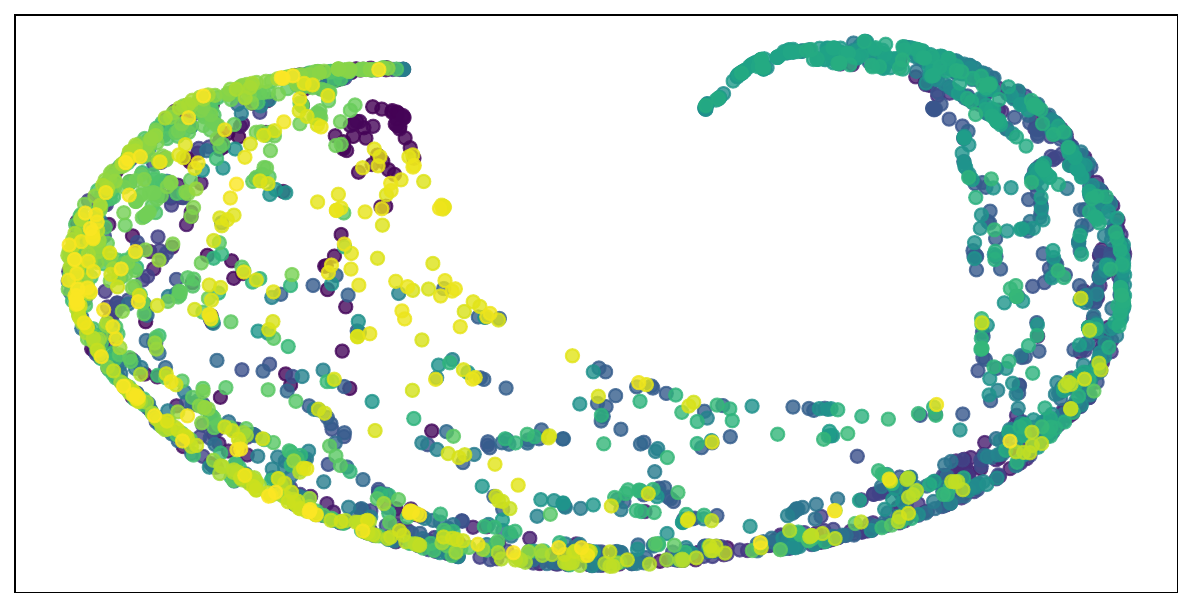}
            \caption{ETTh2 (PHATE)}
        \end{subfigure}
    \end{tabular}
    \caption{Latent space visualizations of TimeVLM on ETTh1 and ETTh2 datasets.}
    \label{fig:latent_timevlm_etth}
\end{figure*}

\begin{figure*}[!t]
    \centering
    \setlength{\tabcolsep}{1.0pt} 
    \begin{tabular}{ll}
        \begin{subfigure}{0.48\textwidth}
            \centering
            \includegraphics[width=\textwidth]{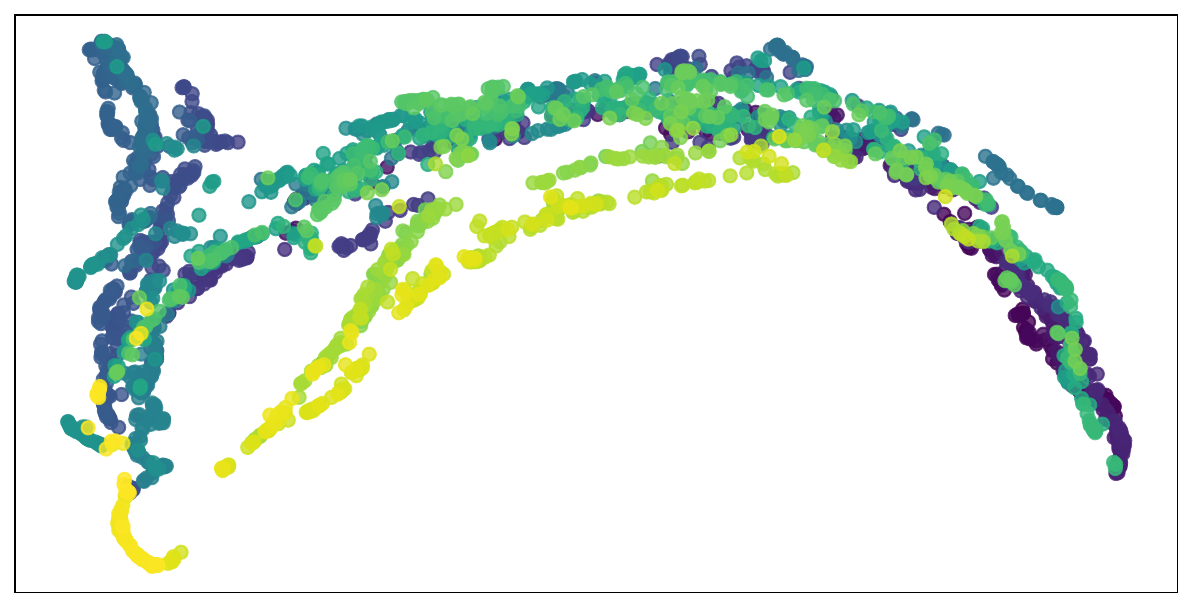}
            \caption{ETTm1 (UMAP)}
        \end{subfigure} &
        \begin{subfigure}{0.48\textwidth}
            \centering
            \includegraphics[width=\textwidth]{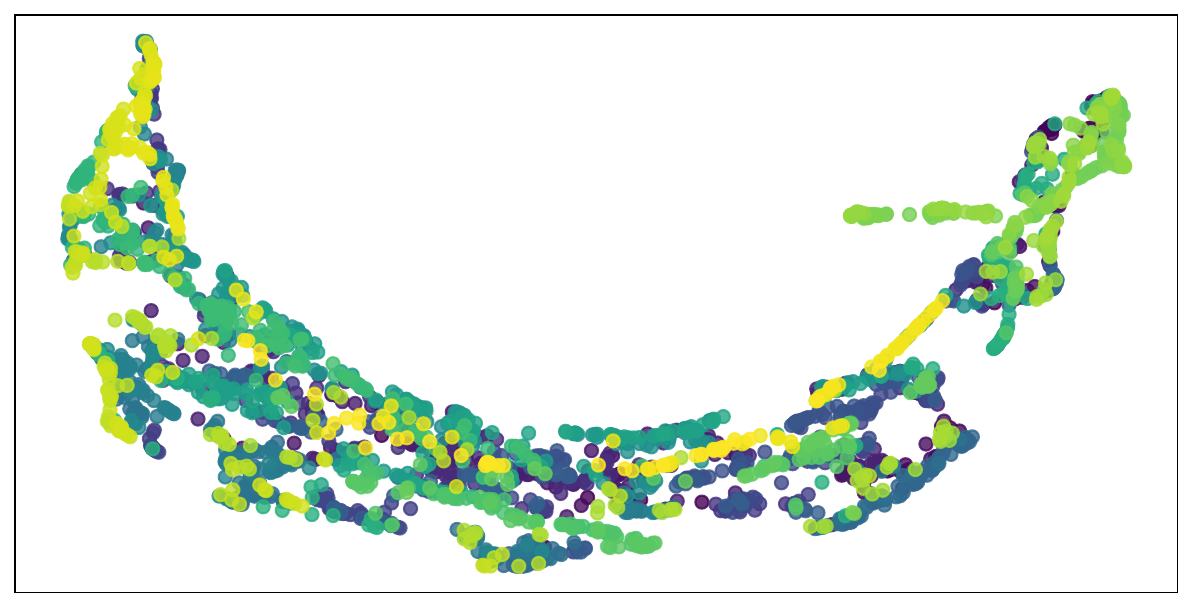}
            \caption{ETTm2 (UMAP)}
        \end{subfigure} \\ [1.5em]
        \begin{subfigure}{0.48\textwidth}
            \centering
            \includegraphics[width=\textwidth]{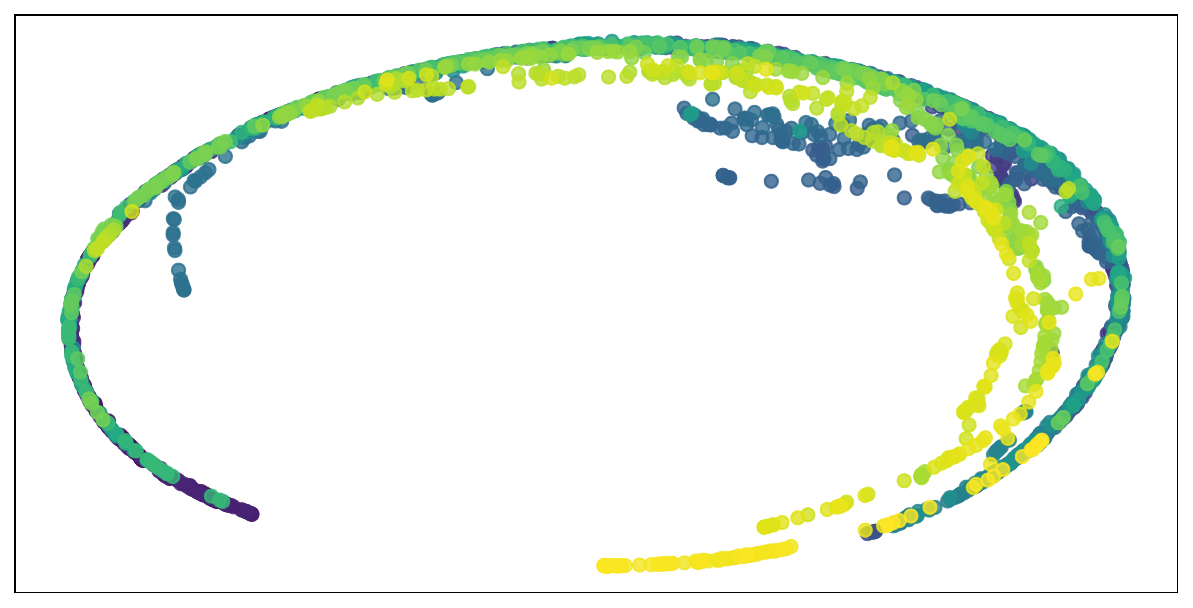}
            \caption{ETTm1 (PHATE)}
        \end{subfigure} &
        \begin{subfigure}{0.48\textwidth}
            \centering
            \includegraphics[width=\textwidth]{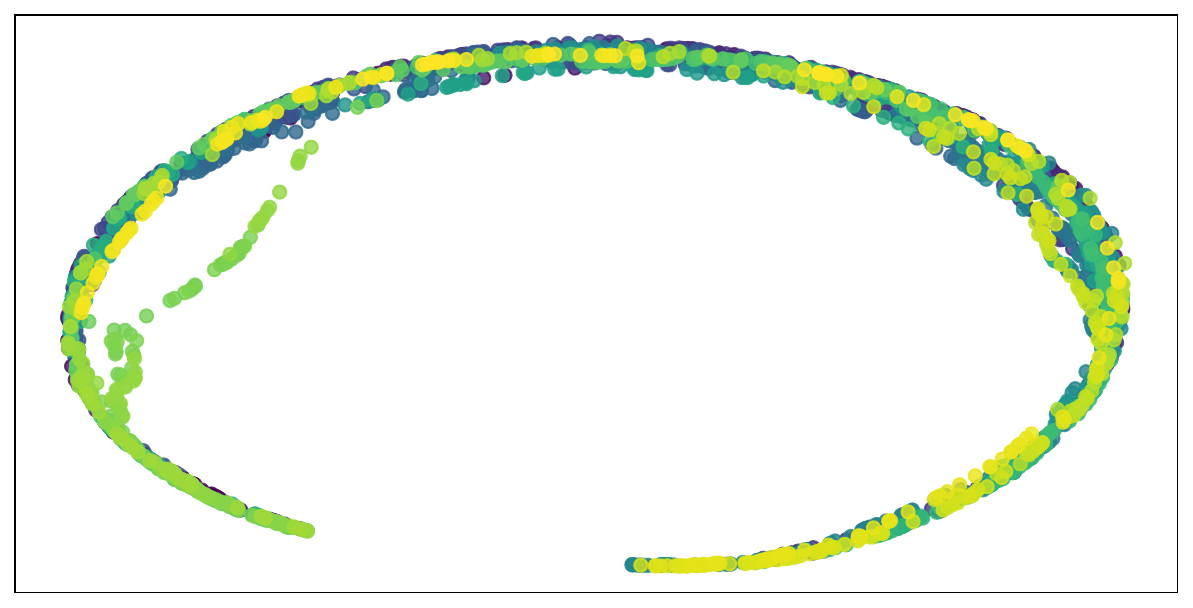}
            \caption{ETTm2 (PHATE)}
        \end{subfigure}
    \end{tabular}
    \caption{Latent space visualizations of TimeVLM on ETTm1 and ETTm2 datasets.}
    \label{fig:latent_timevlm_ettm}
\end{figure*}

\begin{figure*}[!t]
    \centering
    \setlength{\tabcolsep}{1.0pt} 
    \begin{tabular}{ll}
        \begin{subfigure}{0.48\textwidth}
            \centering
            \includegraphics[width=\textwidth]{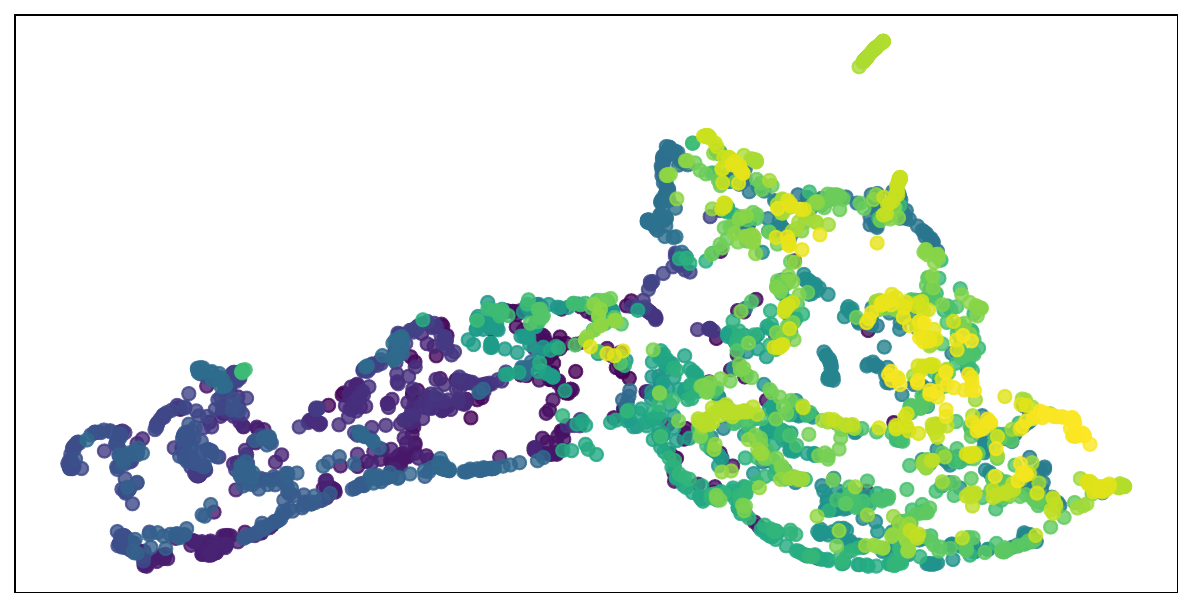}
            \caption{Weather (UMAP)}
        \end{subfigure} &
        \begin{subfigure}{0.48\textwidth}
            \centering
            \includegraphics[width=\textwidth]{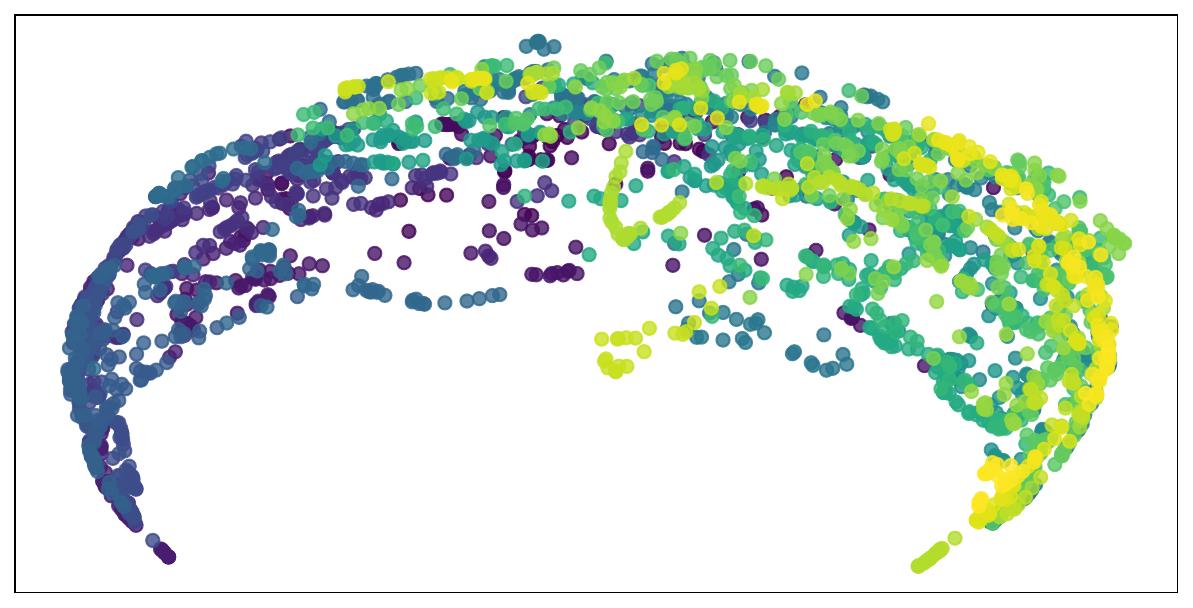}
            \caption{Weather (PHATE)}
        \end{subfigure}
    \end{tabular}
    \caption{Latent space visualization of TimeVLM on the Weather dataset, comparing UMAP and PHATE embeddings.}
    \label{fig:latent_timevlm_weather}
\end{figure*}

\clearpage

\section{Mathematical proofs and transition analysis}
\label{sec:proofs}

We provide proofs of the finite-polynomial transition structure and its displacement properties, followed by a comparison with the spectral behavior of a general linear generator.

\subsection{Proof of Theorem 1 and polynomial structure}
For Theorem~\ref{thm:polynomial_distance}, let $q\le d_h$ denote the nilpotency index of $K$.

\begin{proof}
According to the Taylor series expansion of the matrix exponential, we have:
\begin{equation}
    e^{\tau K} = I + \tau K + \frac{\tau^2 K^2}{2!} + \dots = \sum_{n=0}^{\infty} \frac{(\tau K)^n}{n!}
\end{equation}
A strictly upper-triangular matrix satisfies $K^{d_h}=0$. By the definition of $q$, all powers $K^n$ with $n\ge q$ vanish. Consequently, the infinite series truncates to a finite sum:
\begin{equation}
    e^{\tau K} = \sum_{n=0}^{q-1} \frac{\tau^n}{n!} K^n
\end{equation}
Since $K^0=I$ and every positive power of $K$ is strictly upper triangular, $e^{\tau K}$ is upper triangular with unit diagonal. Every entry is a polynomial of degree at most $q-1$ in $\tau$. Adding $\tau\bm{v}$ makes the complete transition a vector-valued polynomial of degree at most $\max\{q-1,1\}$ for fixed $\bm{h}$ and $\bm{v}$. The backward case follows by replacing $\tau$ with $-\tau$, including the offset $-\tau\bm{v}$; the unit-diagonal property and degree bounds are unchanged. This proves both directions of Theorem~\ref{thm:polynomial_distance}.
\end{proof}

\begin{remark}[Implemented signed offset]
In Eqs.~\eqref{eq:fwd_shift} and \eqref{eq:bwd_shift}, OliO directly parameterizes the offset as $\pm\tau\bm{v}$ using a single shared learnable vector $\bm{v}$. The offset is linear in lag, so the complete transition retains the finite-polynomial dependence established by Theorem~\ref{thm:polynomial_distance}.
\end{remark}

\begin{lemma}[Squared displacement of a fixed source]
\label{lemma:l2_distance}
For fixed $\bm{h}$ and $\bm{v}$, the squared displacement $\|\mathcal{T}_{\pm\tau}(\bm{h})-\bm{h}\|_2^2$ is a polynomial in $\tau$ of degree at most $2\max\{q-1,1\}$ for either choice of sign.
\end{lemma}

\begin{proof}
Let $P(\tau)=e^{\tau K}-I$. By Theorem~\ref{thm:polynomial_distance}, $P(\tau)$ is a polynomial matrix. Including the implemented offset gives
\begin{equation}
    \|\mathcal{T}_{\tau}(\bm{h})-\bm{h}\|_2^2
    =\|P(\tau)\bm{h}+\tau\bm{v}\|_2^2.
\end{equation}
The displacement is a vector-valued polynomial of degree at most $\max\{q-1,1\}$. Squaring and summing its entries gives the stated degree bound. Replacing $\tau$ with $-\tau$ gives the backward case with the same degree bound.
\end{proof}

These identities characterize displacement of a fixed source. The training objective aligns independently encoded endpoint representations through cosine similarity, leaving their Euclidean norms unconstrained.

\subsection{Local separation and asymptotic growth}
\label{app:displacement_growth}
For analysis, extend the transition to a real lag $s\in\mathbb R$ by $\mathcal T_s(\bm h)=e^{sK}\bm h+s\bm v$. This extension agrees with the integer-lag formulation used in pre-training. For fixed $\bm h$ and $\bm v$, define
\begin{equation}
    \bm d_h(s)=(e^{sK}-I)\bm h+s\bm v,
    \qquad D_h(s)=\|\bm d_h(s)\|_2^2.
    \label{eq:fixed_source_displacement}
\end{equation}

\begin{proposition}[Local displacement separation]
\label{prop:local_displacement}
If $K\bm h+\bm v\ne0$, there exists $\epsilon>0$ such that $D_h(s)$ is strictly increasing for $0<s<\epsilon$.
\end{proposition}
\begin{proof}
Expansion at zero gives $\bm d_h(s)=s(K\bm h+\bm v)+O(s^2)$, and hence
\begin{equation}
    D_h(s)=s^2\|K\bm h+\bm v\|_2^2+O(s^3).
\end{equation}
Since $D_h$ is polynomial, its derivative is $2s\|K\bm h+\bm v\|_2^2+O(s^2)$, which is positive for sufficiently small positive $s$.
\end{proof}
The condition holds almost surely for any state distribution assigning zero probability to the affine stationary set $\{\bm h:K\bm h+\bm v=0\}$.

\begin{proposition}[Asymptotic displacement growth]
\label{prop:asymptotic_displacement}
Suppose $\bm d_h(s)$ has positive degree $r$ and leading vector coefficient $\bm c_r\ne0$. Then
\begin{equation}
    \|\bm d_h(s)\|_2=\|\bm c_r\|_2s^r(1+o(1))
    \quad\text{as }s\to+\infty,
\end{equation}
and its norm is eventually strictly increasing.
\end{proposition}
\begin{proof}
Factoring out $s^r$ yields $s^{-r}\bm d_h(s)\to\bm c_r$, giving the norm asymptotic. The squared norm has leading term $\|\bm c_r\|_2^2s^{2r}$, whose derivative is eventually positive. The norm is then positive and strictly increasing as well.
\end{proof}
These results describe local separation and eventual growth. Behavior at intermediate lags depends on the learned coefficients. App.~\ref{app:distance_ordering} reports finite-range measurements.

\subsection{Algebraic structure and spectral behavior}

\begin{corollary}[Spectral comparison]
\label{cor:exponential_instability}
For a general linear generator, each Jordan block contributes an exponential factor $e^{\tau\lambda_i}$ multiplied by a finite polynomial. A nonzero real part of $\lambda_i$ introduces a mode growing exponentially in one of the two signed temporal directions. For purely imaginary eigenvalues, oscillatory factors may be accompanied by polynomial terms from nontrivial Jordan blocks.
\end{corollary}

\begin{proof}
Using Jordan normal form over $\mathbb{C}$, write $K=PJP^{-1}$ with blocks $J_i=\lambda_i I+N_i$ of size $n_i$, where $N_i$ is nilpotent. The block exponential is
\begin{equation}
    e^{\tau J_i} = e^{\tau (\lambda_i I + N_i)} = e^{\tau \lambda_i} e^{\tau N_i} = e^{\tau \lambda_i} \left( \sum_{k=0}^{n_i-1} \frac{\tau^k}{k!} N_i^k \right)
\end{equation}

\textbf{Case 1: Positive real part.}
If $\operatorname{Re}(\lambda_i)>0$, the factor $|e^{\tau\lambda_i}|=e^{\tau\operatorname{Re}(\lambda_i)}$ grows exponentially for positive lag.

\textbf{Case 2: Negative real part.}
For $\operatorname{Re}(\lambda_i)<0$, the block decays in the forward direction because exponential decay dominates its polynomial factor. Writing $\alpha=-\operatorname{Re}(\lambda_i)>0$ gives
\begin{equation}
    \lim_{\tau \to \infty} \frac{\tau^k}{e^{\alpha \tau}} = \lim_{\tau \to \infty} \frac{k\tau^{k-1}}{\alpha e^{\alpha \tau}} = \dots = \lim_{\tau \to \infty} \frac{k!}{\alpha^k e^{\alpha \tau}} = 0
\end{equation}
The same mode grows in the backward direction because reversing the lag reverses the sign of the real part in the exponential factor.

\textbf{Case 3: Zero real part.}
For $\lambda_i=i\omega$, the exponential factor has unit magnitude and can be oscillatory. A nontrivial Jordan block adds polynomial factors, so zero real parts alone do not guarantee bounded transport. Nonzero imaginary eigenvalues do not yield the finite-polynomial matrix dependence induced by nilpotency.

Strict upper triangularity supplies nilpotency and exact finite-polynomial transport in both signed directions. The magnitude of the resulting polynomial transients depends on the learned coefficients and the evaluated lag.
\end{proof}

\subsection{Computational complexity and polynomial truncation}
\label{sec:complexity_proof}
We analyze the cost of evaluating the transition and the approximation introduced by truncating its polynomial expansion.

\textbf{Taylor expansion and nilpotency.} The continuous transition operator $e^{\tau K}$ is defined by the infinite Taylor series:
\begin{equation}
    e^{\tau K} = I + \tau K + \frac{\tau^2 K^2}{2!} + \dots = \sum_{n=0}^{\infty} \frac{(\tau K)^n}{n!}
\end{equation}
By Theorem~\ref{thm:polynomial_distance}, the nilpotency index satisfies $q\le d_h$, so the exact series terminates at degree at most $q-1$:
\begin{equation}
    e^{\tau K} = \sum_{n=0}^{q-1} \frac{\tau^n}{n!} K^n
\end{equation}

\textbf{Polynomial truncation.} For $p<q-1$, an approximation retains terms through order $p$:
\begin{equation}
    \hat{S}_p(\tau) = \sum_{n=0}^{p} \frac{(\tau K)^n}{n!}
\end{equation}
To evaluate this polynomial efficiently, we apply \textbf{Horner's method}~\cite{Horner1819}, which reformulates the sum into a nested form. For the matrix exponential, the $p$-th order approximation can be computed as:
\begin{equation}
    \text{Horner}(M, p) = I + M \left( I + \frac{M}{2} \left( I + \frac{M}{3} \dots \left( I + \frac{M}{p} \right) \dots \right) \right)
\end{equation}
where $M=\tau K$ for the forward direction. This nested evaluation uses $p$ matrix--matrix multiplications instead of explicitly forming each power separately~\cite{Higham2008Functions,Moler2003Nineteen}. For $p=0$, the polynomial is simply $I$. Alg.~\ref{alg:horner} summarizes the matrix-level computation, with cost $O(pd_h^3)$.

Applying the polynomial directly to $\bm{h}$ requires only matrix--vector products, giving cost $O(pd_h^2)$ plus $O(d_h)$ for the signed offset, as detailed in Alg.~\ref{alg:horner_vector}. Exact evaluation uses $p=q-1$ and remains cubic in the worst case $q=d_h$; using a fixed lower order reduces cost while approximating the complete transition.

\begin{algorithm}[h]
    \caption{Matrix exponential approximation via Horner's method}
    \label{alg:horner}
    \begin{algorithmic}[1]
        \STATE \textbf{Input:} Matrix $M \in \mathbb{R}^{d_h \times d_h}$, truncation order $p$
        \STATE \textbf{Output:} Approximated transition operator $\hat{S}_p \in \mathbb{R}^{d_h \times d_h}$
        \STATE $\hat{S}_p \leftarrow I_{d_h \times d_h}$ \COMMENT{Initialize with identity matrix}
        \FOR{$i = p$ \textbf{down to} $1$}
            \STATE $\hat{S}_p \leftarrow I_{d_h \times d_h} + \frac{1}{i} M \cdot \hat{S}_p$ \COMMENT{$\mathcal{O}(d_h^3)$ per iteration}
        \ENDFOR
        \RETURN $\hat{S}_p$
    \end{algorithmic}
\end{algorithm}

\begin{algorithm}[h]
    \caption{Efficient representation transition via vector-wise Horner's method}
    \label{alg:horner_vector}
    \begin{algorithmic}[1]
        \STATE \textbf{Input:} Representation $\bm{h}$, matrix $M=\tau K$, order $p$, lag $\tau\in\mathbb{N}$, shared offset $\bm{v}$
        \STATE \textbf{Output:} Transitioned representation $\hat{\bm{h}}_\tau \in \mathbb{R}^{d_h}$
        \STATE $\bm{y} \leftarrow \bm{h}$ \COMMENT{Initialize accumulation vector}
        \FOR{$i = p$ \textbf{down to} $1$}
            \STATE $\bm{y} \leftarrow \bm{h} + \frac{1}{i} M \bm{y}$ \COMMENT{$\mathcal{O}(d_h^2)$ per iteration}
        \ENDFOR
        \STATE $\hat{\bm{h}}_\tau \leftarrow \bm{y} + \tau\bm{v}$ \COMMENT{Add shared signed offset}
        \RETURN $\hat{\bm{h}}_\tau$
    \end{algorithmic}
\end{algorithm}

Alg.~\ref{alg:horner_vector} shows the forward evaluation. The backward evaluation uses $M=-\tau K$ and adds $-\tau\bm{v}$ at the final step; the operation count is unchanged.

\textbf{Effect of truncation.} Truncation omits $E_p(\tau)=\sum_{n=p+1}^{q-1}\tau^nK^n/n!$. The generator can be trained within the truncated parameterization, but optimization does not guarantee recovery of the omitted terms or preservation of forecasting accuracy. The order ablation in Sec.~\ref{sec:ablation_study} evaluates this accuracy--cost trade-off separately from the complete transition.

\clearpage
\begin{figure*}[t]
    \centering
    \includegraphics[width=0.8\textwidth]{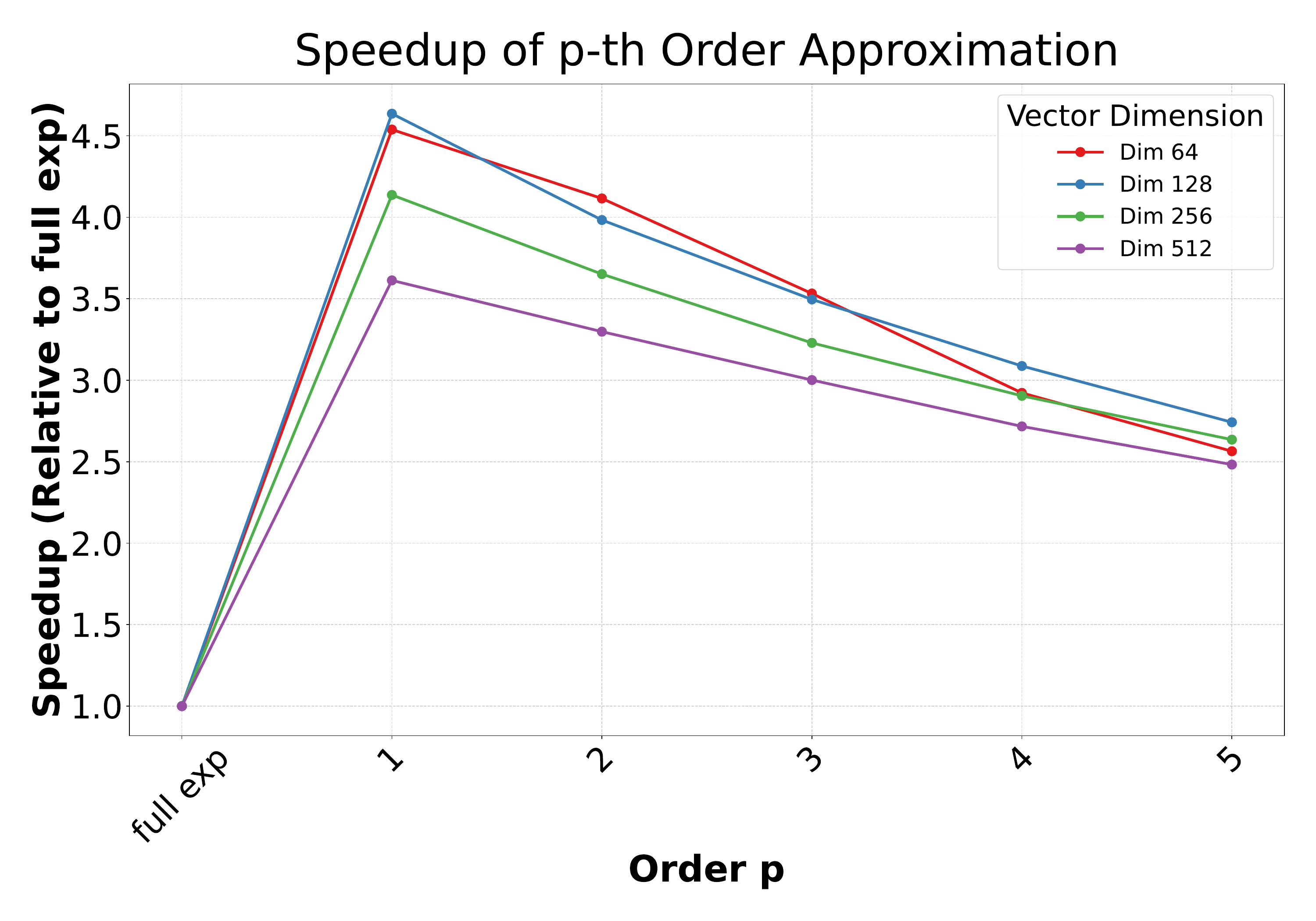}
    \caption{Speedup of the OliO module's $p$-th order approximation relative to the full matrix exponential (computed via \texttt{torch.matrix\_exp}). The benchmark evaluates the computational efficiency of the \texttt{TriangularCausalShift} module using 4,096 samples per operation, with values averaged over 5 independent trials on an NVIDIA RTX 3090 GPU. The first-order approximation ($p=1$) provides an acceleration of approximately 3.5$\times$ to 4.5$\times$ compared to the exact computation, highlighting the high-throughput capability of OliO in processing high-dimensional temporal representations.}
    \label{fig:speed_comparison}
\end{figure*}

\textbf{Empirical transition-computation efficiency.}
Fig.~\ref{fig:speed_comparison} reports a dimension-wise benchmark of the core \texttt{OliO} transition module (\texttt{TriangularCausalShift}) relative to \texttt{torch.matrix\_exp}. The benchmark uses 4,096 samples per operation, latent dimensions $\{64,128,256,512\}$, and five independent trials on an NVIDIA RTX 3090 GPU. Each speedup is relative to the full matrix exponential at the same dimension. Speedup varies with representation size and decreases as more polynomial terms are retained. Fig.~\ref{fig:speed_up} gives the order-wise summary in the main text. These measurements characterize transition computation during pre-training. Downstream forecasting uses the original backbone after the projection and transition modules have been discarded.

\clearpage

\section{Datasets}
\label{sec:datasets}

We describe the six benchmark datasets used for full-shot and few-shot forecasting. Table~\ref{tab:dataset_statistics} summarizes their characteristics. The eight downstream datasets used in the LightGTS zero-shot evaluation are reported in App.~\ref{app:zero_shot}.

\begin{table}[h]
\centering
\caption{Statistics of the benchmark datasets used in our forecasting experiments.}
\label{tab:dataset_statistics}
\begin{adjustbox}{width=\linewidth}
\begin{tabular}{lccccr}
\toprule
Dataset & Channels & Sampling Rate & Time Steps & Train/Val/Test Split & Features \\ \midrule
ETTh1, h2 & 7 & 1 hour & 14,307 & 8,545 / 2,881 / 2,881 & Transformer oil temperature \\
ETTm1, m2 & 7 & 15 min & 57,507 & 34,465 / 11,521 / 11,521 & Transformer oil temperature \\
Weather & 21 & 10 min & 52,603 & 36,792 / 5,271 / 10,540 & Meteorological indicators \\
ECL & 321 & 1 hour & 26,211 & 18,317 / 2,633 / 5,261 & Power consumption \\ \bottomrule
\end{tabular}
\end{adjustbox}
\end{table}


\subsection{Textual prompts for TimeVLM}
For all methods using the TimeVLM backbone, we use the textual prompts supplied by its official implementation without modification. The quotations below document these inherited text inputs, while Table~\ref{tab:dataset_statistics} describes the benchmark data used for forecasting.

\paragraph{ETT (ETTh1, ETTh2, ETTm1, ETTm2)}
\begin{quote}
\textit{``The Electricity Transformer Temperature (ETT) is a crucial indicator in the electric power long-term deployment. This dataset consists of 2 years data from two separated counties in China. To explore the granularity on the Long sequence time-series forecasting (LSTF) problem, different subsets are created, \{ETTh1, ETTh2\} for 1-hour-level and ETTm1 for 15-minutes-level. Each data point consists of the target value "oil temperature" and 6 power load features. The train/val/test is 12/4/4 months.''}
\end{quote}

\paragraph{Weather}
\begin{quote}
\textit{``Weather is recorded every 10 minutes for the 2020 whole year, which contains 21 meteorological indicators, such as air temperature, humidity, etc.''}
\end{quote}

\paragraph{Electricity (ECL)}
\begin{quote}
\textit{``Measurements of electric power consumption in one household with a one-minute sampling rate over a period of almost 4 years. Different electrical quantities and some sub-metering values are available. This archive contains 2075259 measurements gathered in a house located in Sceaux (7km of Paris, France) between December 2006 and November 2010 (47 months).''}
\end{quote}
The ECL experiments use the 321-channel hourly benchmark listed in Table~\ref{tab:dataset_statistics}. The Electricity quotation above is retained as part of the inherited TimeVLM prompt configuration.

\section{Implementation details}
\label{sec:implementation_details}

In this section, we provide the detailed experimental setup, including the hardware environment, software framework, loss functions, and hyperparameters used for our experiments.

\subsection{Experimental setup}
All experiments were conducted on a single NVIDIA GeForce RTX 3090 GPU with 24GB of memory. We implemented our framework using the PyTorch library. For all downstream forecasting tasks, we employed the $L_2$ loss for training. The performance was evaluated using two standard metrics: Mean Squared Error (MSE) and Mean Absolute Error (MAE), defined as follows:
\begin{equation}
    \text{MSE} = \frac{1}{n} \sum_{i=1}^{n} (y_i - \hat{y}_i)^2
\end{equation}
\begin{equation}
    \text{MAE} = \frac{1}{n} \sum_{i=1}^{n} |y_i - \hat{y}_i|
\end{equation}
where $y_i$ and $\hat{y}_i$ denote the ground truth and the predicted values, respectively, and $n$ is the total number of data points.

Training, validation, and test periods are chronologically separated. Pre-training pairs are drawn from the training period and exclude held-out forecasting targets. We use seeds 2023, 2024, and 2025 for three independent runs; tables report means and sample standard deviations.

\subsection{Self-supervised methods}
We compared OliO with several state-of-the-art self-supervised learning (SSL) methods, which can be broadly categorized into two groups: Masked and Reconstruction-based methods, and Contrastive Learning methods. Table~\ref{tab:ssl_methods} summarizes the compared methods.

\begin{table}[h]
\centering
\caption{Summary of self-supervised learning methods compared in our experiments.}
\label{tab:ssl_methods}
\begin{adjustbox}{width=\linewidth}
\begin{tabular}{lll}
\toprule
Category & Method & Key Feature \\ \midrule
\multirow{2}{*}{Masked \& Reconstruction} & TimeSiam \cite{TimeSiam} & Siamese modeling with temporal masking \\
 & SimMTM \cite{SimMTM} & Masked time-series modeling with manifold learning \\ \midrule
\multirow{3}{*}{Contrastive Learning} & TimesURL \cite{TimesURL} & Universal relational learning for time series \\
 & SoftCLT \cite{SoftCLT} & Soft contrastive learning for time series \\
 & TS2Vec \cite{TS2Vec} & Hierarchical contrastive learning on timestamp level \\ \bottomrule
\end{tabular}
\end{adjustbox}
\end{table}

Each SSL objective uses the representation interface available in the selected backbone. For iTransformer~\cite{iTransformer} and TimeVLM~\cite{TimeVLM}, temporal loss components requiring a timestamp axis are omitted when the latent representation has no such axis.

PatchTST retains a temporal latent axis compatible with timestamp-indexed objectives. The other backbones use their available latent interfaces under the same plug-and-play protocol. The evaluated TimeSiam implementation retains temporally separated subseries, Siamese encoding, masking, lineage embeddings, its cross-attention decoder, and past-to-current reconstruction. These configurations define the shared-backbone comparisons reported in Sec.~\ref{sec:experiment}.

\subsection{Finetuning and hyperparameter search}
After pre-training, we fully finetune each backbone for forecasting. For each method--backbone--dataset--horizon setting, we select one configuration by validation MSE from the predefined search space. The selected configuration is fixed before three-seed test evaluation; test performance is excluded from hyperparameter and checkpoint selection. Baseline-specific search spaces follow the protocols documented in the original papers. Specifically, for TimeSiam, we utilized the settings reported in \citet[App. A.2, Table 10]{TimeSiam}, which specify a learning rate of $1 \times 10^{-4}$ and a dropout search space of $\{0, 0.1, 0.2, 0.3\}$. Similarly, for SimMTM, we adopted the configuration from \citet[App. A.4, Table 10]{SimMTM} with a learning rate of $1 \times 10^{-4}$ and a dropout rate search space of $\{0, 0.1, 0.2, 0.3\}$. For SoftCLT, we used a learning rate of $1 \times 10^{-3}$ and searched for the optimal dropout rate within $\{0, 0.1, 0.2, 0.3\}$, as detailed in \citet[App. C, Table C.1]{SoftCLT}. For other SSL methods and backbone models, validation selection uses the following search space:
\begin{itemize}
    \item Learning rate: $\{1 \times 10^{-3}, 5 \times 10^{-4}, 1 \times 10^{-4}, 1 \times 10^{-5}\}$
    \item Dropout rate: $\{0, 0.1, 0.2, 0.3\}$
\end{itemize}

In general, we observed that employing a sequence length for pre-training that is twice that of fine-tuning was more effective for capturing diverse temporal dynamics to model the long-term dependencies. For the fine-tuning stage, a learning rate of $1 \times 10^{-4}$ and a dropout rate of $0.0$ usually yielded the better performance across most benchmarks. Regarding the projection dimension ($d_h$), as no clear performance trend was observed during our preliminary grid search, we followed common practices in contrastive learning and utilized dimensions within the range of $128 \sim 256$.

For OliO pre-training with the iTransformer backbone, we also performed a grid search to determine the optimal hyperparameters for each dataset. The search space included \texttt{projection\_dim} $\in \{128, 256\}$, \texttt{learning\_rate} $\in \{1 \times 10^{-3}, 7 \times 10^{-4}\}$, \texttt{train\_epochs} $\in \{10, 20\}$, and \texttt{max\_tau} $\in \{8, 32, 64\}$. The dataset-specific settings were as follows:
\begin{itemize}
    \item Electricity: \texttt{seq\_len} $\in \{96, 192\}$, \texttt{batch\_size} $= 16$.
    \item ETTh1: \texttt{seq\_len} $\in \{96, 192\}$, \texttt{batch\_size} $= 32$.
    \item ETTh2, ETTm1, ETTm2: \texttt{seq\_len} $\in \{336, 672\}$, \texttt{batch\_size} $= 32$.
    \item Weather: \texttt{seq\_len} $\in \{336, 672\}$, \texttt{batch\_size} $= 16$.
\end{itemize}

For OliO pre-training with the TimeVLM backbone, the search space for \texttt{projection\_dim}, \texttt{learning\_rate}, \texttt{train\_epochs}, and \texttt{max\_tau} was kept identical to the iTransformer experiments. The dataset-specific configurations were as follows:
\begin{itemize}
    \item Electricity: \texttt{seq\_len} $\in \{512, 1024\}$, \texttt{batch\_size} $= 16$.
    \item ETTh1, ETTh2, ETTm1, ETTm2, and Weather: \texttt{seq\_len} $\in \{512, 1024\}$, \texttt{batch\_size} $= 32$.
\end{itemize}

For OliO pre-training with the PatchTST backbone, the search space for the common hyperparameters was also maintained as described above. The dataset-specific settings were as follows:
\begin{itemize}
    \item Electricity: \texttt{seq\_len} $\in \{336, 672\}$, \texttt{batch\_size} $= 8$.
    \item ETTh1, ETTh2, ETTm1, ETTm2, and Weather: \texttt{seq\_len} $\in \{336, 672\}$, \texttt{batch\_size} $= 32$.
\end{itemize}

\subsection{Computational cost}
Each individual experimental session (including pre-training and forecasting fine-tuning for a single dataset) typically takes approximately one hour on a single NVIDIA RTX 3090 GPU.

\subsection{Visualization settings}
Dimensionality-reduction tools such as t-SNE~\cite{tsne}, UMAP~\cite{umap}, and PHATE~\cite{phate} provide low-dimensional views of learned representations. We use UMAP and PHATE to obtain complementary two-dimensional projections, with points colored by the time index of each historical window. App.~\ref{appendix:latent_visualizations} presents the projections alongside the forecasting and representation-quality evaluations in Secs.~\ref{sec:experiment} and~\ref{sec:representation_quality}.

\subsection{Baseline model architectures}
\label{sec:baseline_models}

In our experiments, we evaluated OliO across three representative backbone architectures, each highlighting different structural paradigms in time-series modeling. For all models, we followed the official hyperparameter settings (e.g., embedding dimension, number of layers, and feed-forward network dimension) as specified in their respective repositories.

\begin{itemize}
    \item \textbf{PatchTST} \cite{PatchTST}: A representative patch-based, channel-independent model. By segmenting time series into sub-series patches, it captures local semantic information while maintaining computational efficiency. Following the official implementation, the patch size was set to 12.
    \begin{itemize}
        \item Repository: \url{https://github.com/yuqinie98/patchtst}
        \item Settings: $\text{seq\_len}=336$, $\text{pred\_len} \in \{96, 192, 336, 720\}$.
    \end{itemize}

    \item \textbf{iTransformer} \cite{iTransformer}: A channel-dependent architecture that treats each channel as a single token. This approach leverages the multi-head attention mechanism to explicitly capture inter-channel dependencies, making it effective for multivariate forecasting.
    \begin{itemize}
        \item Repository: \url{https://github.com/thuml/iTransformer}
        \item Settings: $\text{seq\_len}=96$, $\text{pred\_len} \in \{96, 192, 336, 720\}$.
    \end{itemize}

    \item \textbf{TimeVLM} \cite{TimeVLM}: A state-of-the-art multimodal model that integrates vision, text, and time-series data. It aligns temporal features with natural language descriptions to enhance predictive performance.
    \begin{itemize}
        \item Repository: \url{https://github.com/CityMind-Lab/ICML25-TimeVLM}
        \item Settings: $\text{seq\_len}=512$, $\text{pred\_len} \in \{96, 192, 336, 720\}$.
    \end{itemize}
\end{itemize}

All backbones use the same prediction horizons. Each backbone retains its standard preprocessing and normalization across the compared SSL methods.

\section{Additional representation quality and transferability}
\label{app:representation_quality}

\subsection{Frozen-backbone linear probing}

We provide the detailed frozen-backbone linear probing results corresponding to
Sec.~\ref{sec:representation_quality}. After self-supervised pre-training, all backbone parameters are frozen and only an identical linear forecasting head is trained. We compare OliO with TimesURL using PatchTST and iTransformer across six datasets and four prediction horizons. Each reported dataset value averages the four forecasting horizons. The deviation shown after $\pm$ is the average of their three-seed sample standard deviations, rather than the standard deviation of the horizon-averaged score.

\begin{table}[t]
\centering
\caption{Frozen-backbone linear probing results. Each dataset value averages four forecasting horizons. Lower MSE is better.}
\label{tab:frozen_probe_full}
\begin{tabular}{llcc}
\toprule
Backbone & Dataset & TimesURL & OliO \\
\midrule

\multirow{7}{*}{PatchTST}
& Electricity & 0.1846 $\pm$ 0.0012 & \textbf{0.1833 $\pm$ 0.0004} \\
& ETTh1       & 0.4147 $\pm$ 0.0021 & \textbf{0.4107 $\pm$ 0.0009} \\
& ETTh2       & \textbf{0.3409 $\pm$ 0.0006} & 0.3461 $\pm$ 0.0011 \\
& ETTm1       & \textbf{0.3546 $\pm$ 0.0009} & 0.3616 $\pm$ 0.0012 \\
& ETTm2       & 0.2550 $\pm$ 0.0006 & \textbf{0.2528 $\pm$ 0.0005} \\
& Weather     & 0.2469 $\pm$ 0.0015 & \textbf{0.2439 $\pm$ 0.0005} \\
& \textbf{Average}
              & \textbf{0.2994 $\pm$ 0.0011} & 0.2997 $\pm$ 0.0008 \\
\midrule

\multirow{7}{*}{iTransformer}
& Electricity & 0.2266 $\pm$ 0.0004 & \textbf{0.2161 $\pm$ 0.0003} \\
& ETTh1       & \textbf{0.4535 $\pm$ 0.0011} & 0.4536 $\pm$ 0.0011 \\
& ETTh2       & \textbf{0.3796 $\pm$ 0.0013} & 0.3811 $\pm$ 0.0012 \\
& ETTm1       & 0.4176 $\pm$ 0.0002 & \textbf{0.4108 $\pm$ 0.0006} \\
& ETTm2       & 0.2896 $\pm$ 0.0005 & \textbf{0.2862 $\pm$ 0.0007} \\
& Weather     & 0.2805 $\pm$ 0.0005 & \textbf{0.2738 $\pm$ 0.0005} \\
& \textbf{Average}
              & 0.3412 $\pm$ 0.0007 & \textbf{0.3369 $\pm$ 0.0007} \\
\midrule
\multicolumn{2}{l}{\textbf{Overall average}}
& 0.3203 & \textbf{0.3183} \\
\bottomrule
\end{tabular}
\end{table}

Across the 48 backbone--dataset--horizon configurations, OliO achieves lower MSE than TimesURL in 35 settings. When the four horizons are aggregated at the dataset level, OliO performs better in 8 of 12 backbone--dataset combinations. PatchTST shows comparable overall performance between the two methods, whereas OliO provides a clearer improvement on iTransformer, reducing the average MSE from 0.3412 to 0.3369. Overall, the aggregate MSE decreases from 0.3203 to 0.3183.

\subsection{Zero-shot performance with LightGTS-4M}
\label{app:zero_shot}

We follow the LightGTS-4M~\cite{wang2025lightgts} zero-shot protocol, preserving its architecture, forecasting hyperparameters, and pre-training data. A single pretrained checkpoint is evaluated on eight downstream datasets disjoint from pre-training, with four forecasting horizons per dataset. Downstream evaluation uses the pretrained parameters directly, without parameter updates or checkpoint selection. We compare TimesURL, the default LightGTS mask-reconstruction objective, and OliO across the resulting 32 settings.

Table~\ref{tab:zeroshot_full} reports horizon-level MSE. The eight-dataset average is 0.3034 for OliO, 0.3114 for mask reconstruction, and 0.3219 for TimesURL. The main text summarizes these same datasets after averaging over the four horizons.

\paragraph{Additional error metric.}
Across the eight datasets, average MAE is 0.3469 for mask reconstruction and 0.3391 for OliO. Dataset-average MAE also decreases from 0.348 to 0.335 on Electricity, from 0.254 to 0.232 on Solar, and from 0.420 to 0.399 on Traffic. These dataset averages cover the four forecasting horizons and are rounded to three decimal places.

\begin{table}[t]
\centering
\caption{
Full zero-shot forecasting results with the LightGTS-4M time-series foundation model~\cite{wang2025lightgts}.
Mask Recon. denotes its default mask-reconstruction pre-training.
Best results are in \textbf{bold} and the second best are \underline{underlined}. Ranks are determined from the reported values rounded to four decimal places. If multiple methods share the best reported value, no method is counted as Top-1 and all tied methods are treated as Top-2.
Lower MSE is better.
}
\label{tab:zeroshot_full}

\begin{adjustbox}{max width=\linewidth}
\begin{tabular}{llccc}
\toprule
\multicolumn{2}{c}{\multirow{2}{*}{Dataset}}
& \multicolumn{3}{c}{LightGTS-4M} \\
\cmidrule(lr){3-5}
& & TimesURL & Mask Recon. & \textbf{OliO} \\
\midrule

\multirow{4}{*}{\rotatebox{90}{ETTh1}}
& 96  & .3578 & \underline{.3490} & \textbf{.3481} \\
& 192 & .3967 & \underline{.3902} & \textbf{.3859} \\
& 336 & .4213 & \underline{.4115} & \textbf{.4061} \\
& 720 & .4626 & \textbf{.4324} & \underline{.4345} \\
\midrule

\multirow{4}{*}{\rotatebox{90}{ETTh2}}
& 96  & .2814 & \underline{.2802} & \textbf{.2758} \\
& 192 & \underline{.3322} & .3353 & \textbf{.3309} \\
& 336 & \underline{.3592} & .3632 & \underline{.3592} \\
& 720 & .4274 & \underline{.4110} & \textbf{.4022} \\
\midrule

\multirow{4}{*}{\rotatebox{90}{ETTm1}}
& 96  & .3083 & \textbf{.2884} & \underline{.2905} \\
& 192 & .3287 & \underline{.3166} & \textbf{.3164} \\
& 336 & .3526 & \underline{.3423} & \textbf{.3417} \\
& 720 & .3947 & \underline{.3866} & \textbf{.3828} \\
\midrule

\multirow{4}{*}{\rotatebox{90}{ETTm2}}
& 96  & .1863 & \textbf{.1650} & \underline{.1704} \\
& 192 & .2345 & \textbf{.2205} & \underline{.2206} \\
& 336 & .2747 & \underline{.2677} & \textbf{.2655} \\
& 720 & \underline{.3309} & .3311 & \textbf{.3277} \\
\midrule

\multirow{4}{*}{\rotatebox{90}{Weather}}
& 96  & .1587 & \textbf{.1420} & \underline{.1450} \\
& 192 & .2063 & \textbf{.1818} & \underline{.1867} \\
& 336 & .2587 & \textbf{.2305} & \underline{.2381} \\
& 720 & .3190 & \textbf{.2994} & \underline{.3028} \\
\midrule

\multirow{4}{*}{\rotatebox{90}{ECL}}
& 96  & .1859 & \underline{.1684} & \textbf{.1650} \\
& 192 & .2264 & \underline{.2172} & \textbf{.1994} \\
& 336 & \underline{.2681} & .2698 & \textbf{.2423} \\
& 720 & .3421 & \underline{.3349} & \textbf{.3232} \\
\midrule

\multirow{4}{*}{\rotatebox{90}{Solar}}
& 96  & .1560 & \underline{.1553} & \textbf{.1453} \\
& 192 & \underline{.1668} & .1684 & \textbf{.1539} \\
& 336 & \underline{.1741} & .1777 & \textbf{.1656} \\
& 720 & \textbf{.1746} & .1835 & \underline{.1798} \\
\midrule

\multirow{4}{*}{\rotatebox{90}{Traffic}}
& 96  & .4561 & \underline{.4069} & \textbf{.4037} \\
& 192 & .5166 & \underline{.4897} & \textbf{.4549} \\
& 336 & \underline{.5672} & .5761 & \textbf{.5077} \\
& 720 & .6760 & \underline{.6728} & \textbf{.6377} \\
\midrule

\multicolumn{2}{l}{Average}
& .3219 & .3114 & \textbf{.3034} \\

\bottomrule
\end{tabular}
\end{adjustbox}
\end{table}

\section{Few-shot forecasting results}
\label{app:fewshot_results}

We evaluate forecasting performance using limited labeled training data. Following the sampling protocol of TimeVLM~\cite{TimeVLM}, we construct few-shot training sets by randomly selecting 1\%, 5\%, and 10\% of the original training data. These experiments assess the pretrained representations when labeled data for downstream forecasting are scarce.

\subsection{Summary of few-shot results}
Figs.~\ref{fig:fewshot_summary_patchtst}, \ref{fig:fewshot_summary_itransformer}, and~\ref{fig:fewshot_summary_timevlm} summarize the relative MSE across the three few-shot ratios. The relative performance index is the mean MSE of OliO divided by the mean MSE of the corresponding method, each averaged over the same 24 dataset--horizon settings, with OliO fixed at 1.0. Values above 1.0 indicate lower MSE than OliO, while values below 1.0 indicate higher MSE.

The following tables report per-setting means and standard deviations. Ranks use the means rounded to four decimal places, and ties receive the last occupied rank (e.g., $2,2,3$ or $1,3,3$). Bold and underlined entries denote ranks 1 and 2, respectively; Top-1 and Top-2 counts tally these entries.

\begin{figure*}[h]
    \centering
    \includegraphics[width=1.0\linewidth]{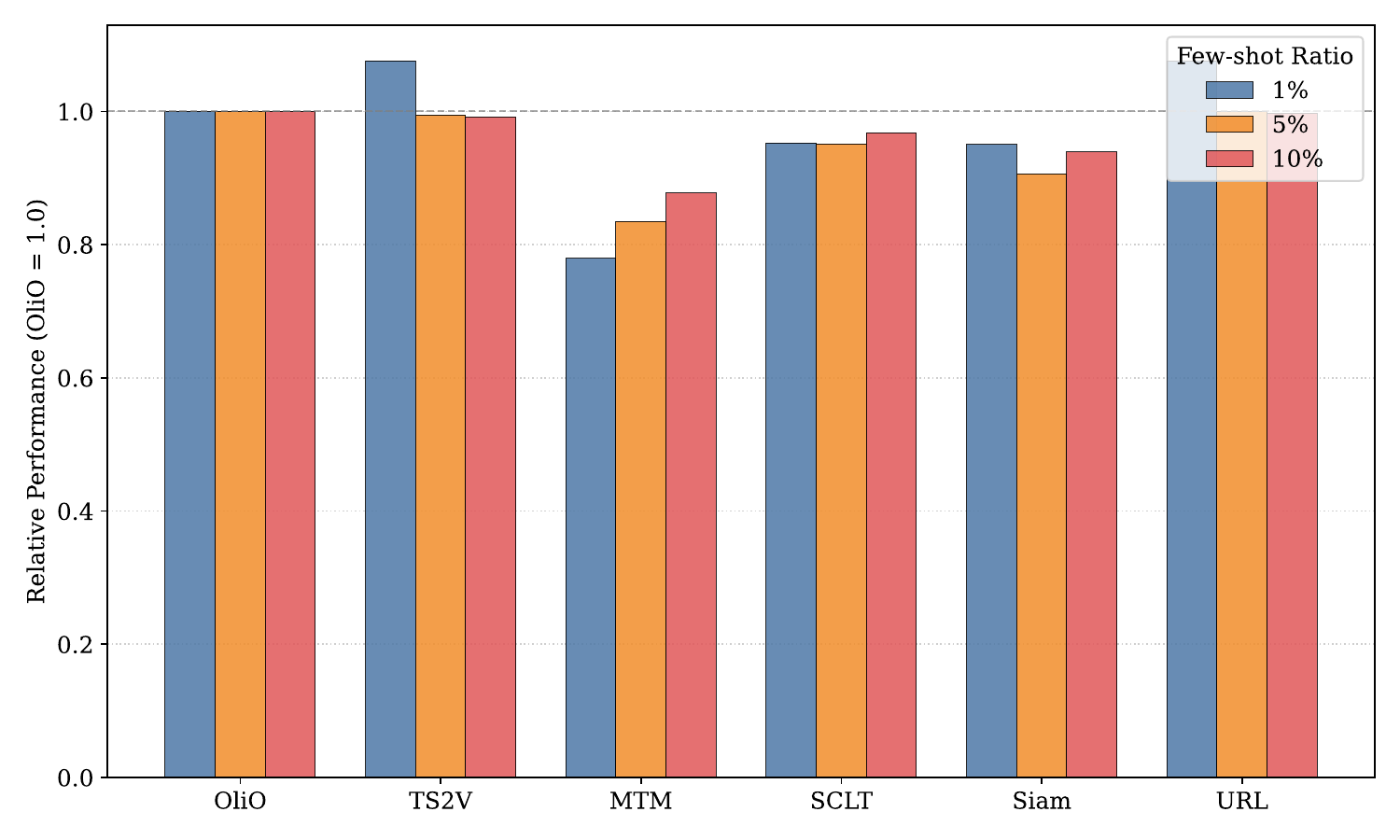}
    \caption{Summary of relative few-shot forecasting performance for \textbf{PatchTST} backbone across 1\%, 5\%, and 10\% settings. Performance is normalized relative to OliO ($=1.0$), where higher values indicate better forecasting quality (lower MSE).}
    \label{fig:fewshot_summary_patchtst}
\end{figure*}

\begin{figure*}[h]
    \centering
    \includegraphics[width=1.0\linewidth]{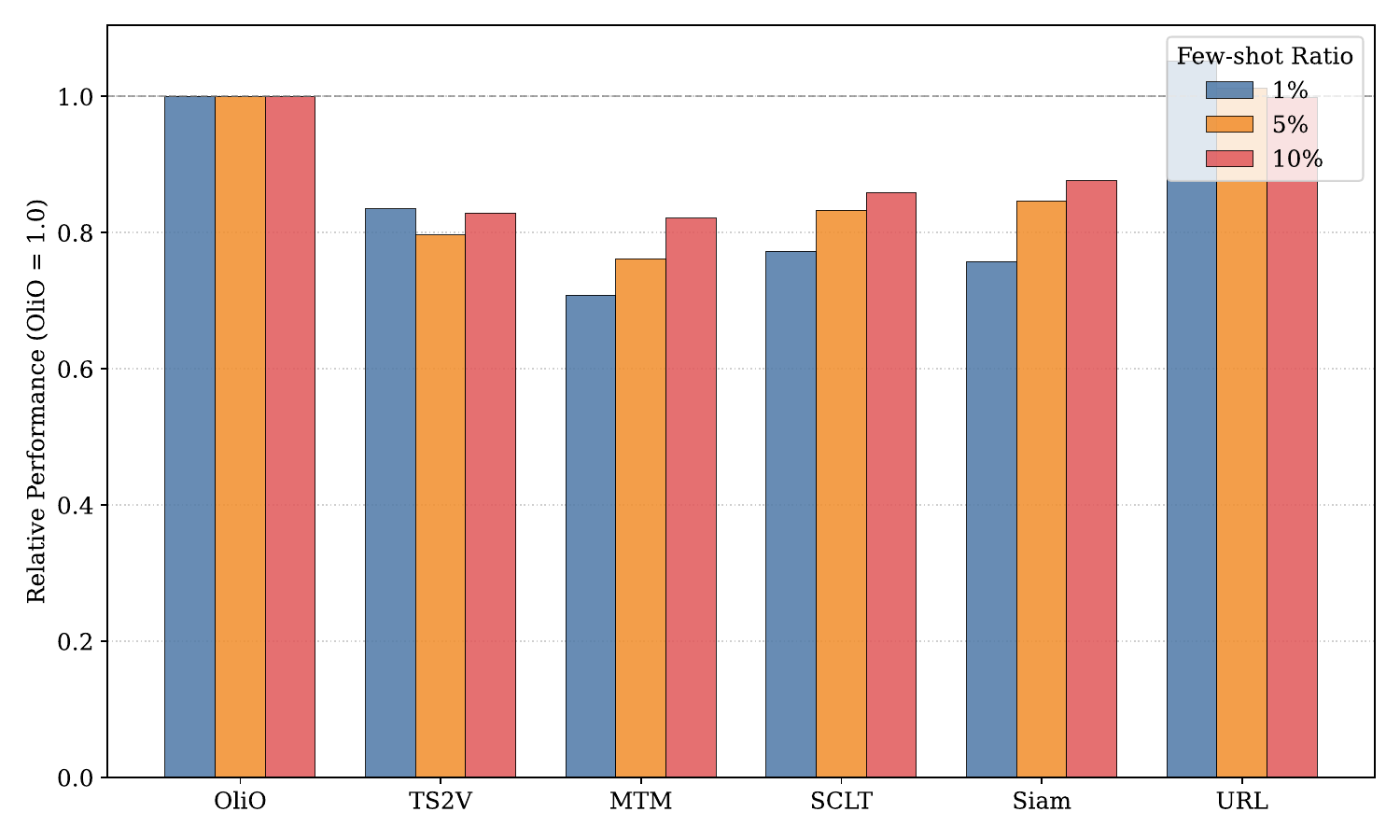}
    \caption{Summary of relative few-shot forecasting performance for \textbf{iTransformer} backbone across 1\%, 5\%, and 10\% settings. Performance is normalized relative to OliO ($=1.0$), where higher values indicate better forecasting quality (lower MSE).}
    \label{fig:fewshot_summary_itransformer}
\end{figure*}

\begin{figure*}[h]
    \centering
    \includegraphics[width=1.0\linewidth]{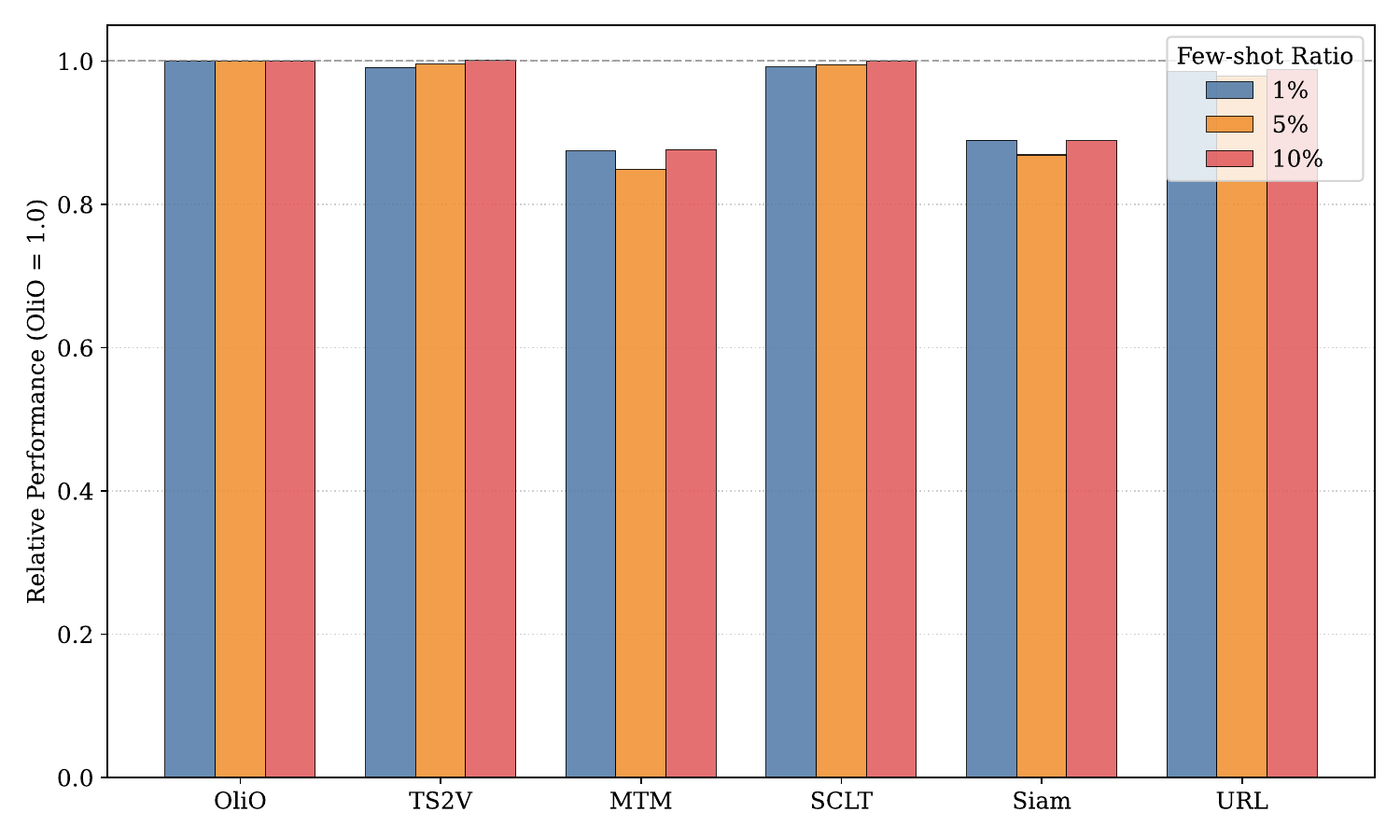}
    \caption{Summary of relative few-shot forecasting performance for \textbf{TimeVLM} backbone across 1\%, 5\%, and 10\% settings. Performance is normalized relative to OliO ($=1.0$), where higher values indicate better forecasting quality (lower MSE).}
    \label{fig:fewshot_summary_timevlm}
\end{figure*}

\clearpage
\subsection{PatchTST few-shot results}
\begin{table*}[!htbp]
\centering
\caption{1\% few-shot forecasting results for PatchTST~\cite{PatchTST} backbone (part 1). TS2Vec~\cite{TS2Vec} is denoted as TS2V, and SimMTM~\cite{SimMTM} as MTM. Best results are in \textbf{bold} and the second best are \underline{underlined}. (SCLT, Siam, and URL results are in Table~\ref{tab:fewshot_01_patchtst_p2}.)}
\label{tab:fewshot_01_patchtst_p1}
\begin{adjustbox}{width=1.0\linewidth}
\begin{tabular}{llcccccc}
\toprule
 \multicolumn{2}{c}{\multirow{2}{*}{Dataset}} & \multicolumn{2}{c}{\textbf{OliO (Ours)}} & \multicolumn{2}{c}{TS2V \cite{TS2Vec}} & \multicolumn{2}{c}{MTM \cite{SimMTM}}\\
\cmidrule(lr){3-4} \cmidrule(lr){5-6} \cmidrule(lr){7-8} & & MSE & MAE & MSE & MAE & MSE & MAE\\

\midrule
 \multirow{4}{*}{\rotatebox{90}{ETTh1}} & 96  & .6167$\pm$.0088 & .5307$\pm$.0033 & \textbf{.4697$\pm$.0042} & \textbf{.4594$\pm$.0022} & .7179$\pm$.0101 & .5701$\pm$.0036\\
& 192 & .6294$\pm$.0090 & .5401$\pm$.0042 & \textbf{.4979$\pm$.0046} & \textbf{.4768$\pm$.0023} & .7295$\pm$.0031 & .5781$\pm$.0013\\
& 336 & .6262$\pm$.0079 & .5453$\pm$.0038 & \textbf{.5049$\pm$.0029} & \textbf{.4829$\pm$.0014} & .7251$\pm$.0019 & .5833$\pm$.0004\\
& 720 & .6198$\pm$.0052 & .5631$\pm$.0017 & \textbf{.5134$\pm$.0072} & \textbf{.5058$\pm$.0035} & .7210$\pm$.0030 & .5999$\pm$.0016\\

\midrule
 \multirow{4}{*}{\rotatebox{90}{ETTh2}} & 96  & .3471$\pm$.0010 & .3923$\pm$.0026 & \textbf{.3290$\pm$.0037} & \textbf{.3788$\pm$.0005} & .3931$\pm$.0003 & .4246$\pm$.0005\\
& 192 & .3888$\pm$.0080 & .4205$\pm$.0038 & \textbf{.3790$\pm$.0092} & \textbf{.4135$\pm$.0054} & .4114$\pm$.0006 & .4371$\pm$.0004\\
\bottomrule

\end{tabular}
\end{adjustbox}
\end{table*}

\begin{table*}[!htbp]
\centering
\caption{1\% few-shot forecasting results for PatchTST~\cite{PatchTST} backbone (part 2). SoftCLT~\cite{SoftCLT} is denoted as SCLT, TimeSiam~\cite{TimeSiam} as Siam, and TimesURL~\cite{TimesURL} as URL. Best results are in \textbf{bold} and the second best are \underline{underlined}. (OliO, TS2V, and MTM results are in Table~\ref{tab:fewshot_01_patchtst_p1}.)}
\label{tab:fewshot_01_patchtst_p2}
\begin{adjustbox}{width=1.0\linewidth}
\begin{tabular}{llcccccc}
\toprule
 \multicolumn{2}{c}{\multirow{2}{*}{Dataset}} & \multicolumn{2}{c}{SCLT \cite{SoftCLT}} & \multicolumn{2}{c}{Siam \cite{TimeSiam}} & \multicolumn{2}{c}{URL \cite{TimesURL}}\\
\cmidrule(lr){3-4} \cmidrule(lr){5-6} \cmidrule(lr){7-8} & & MSE & MAE & MSE & MAE & MSE & MAE\\

\midrule
 \multirow{4}{*}{\rotatebox{90}{ETTh1}} & 96  & .6120$\pm$.0147 & .5269$\pm$.0061 & .7105$\pm$.0029 & .5661$\pm$.0015 & \underline{.4884$\pm$.0035} & \underline{.4722$\pm$.0021}\\
& 192 & .6315$\pm$.0087 & .5385$\pm$.0037 & .7178$\pm$.0004 & .5739$\pm$.0009 & \underline{.5192$\pm$.0034} & \underline{.4902$\pm$.0017}\\
& 336 & .6290$\pm$.0054 & .5434$\pm$.0029 & .7116$\pm$.0040 & .5776$\pm$.0019 & \underline{.5264$\pm$.0029} & \underline{.4962$\pm$.0013}\\
& 720 & .6291$\pm$.0074 & .5617$\pm$.0027 & .7095$\pm$.0018 & .5939$\pm$.0011 & \underline{.5328$\pm$.0088} & \underline{.5170$\pm$.0044}\\

\midrule
 \multirow{4}{*}{\rotatebox{90}{ETTh2}} & 96  & .3785$\pm$.0010 & .4157$\pm$.0008 & .3868$\pm$.0030 & .4207$\pm$.0015 & \underline{.3354$\pm$.0031} & \underline{.3838$\pm$.0013}\\
& 192 & .4061$\pm$.0042 & .4345$\pm$.0023 & .4121$\pm$.0019 & .4372$\pm$.0015 & \underline{.3883$\pm$.0105} & \underline{.4164$\pm$.0064}\\
\bottomrule

\end{tabular}
\end{adjustbox}
\end{table*}

\clearpage
\begin{table*}[!htbp]
\centering
\caption{5\% few-shot forecasting results for PatchTST~\cite{PatchTST} backbone (part 1). TS2Vec~\cite{TS2Vec} is denoted as TS2V, and SimMTM~\cite{SimMTM} as MTM. Best results are in \textbf{bold} and the second best are \underline{underlined}. (SCLT, Siam, and URL results are in Table~\ref{tab:fewshot_05_patchtst_p2}.)}
\label{tab:fewshot_05_patchtst_p1}
\begin{adjustbox}{width=1.0\linewidth}
\begin{tabular}{llcccccc}
\toprule
 \multicolumn{2}{c}{\multirow{2}{*}{Dataset}} & \multicolumn{2}{c}{\textbf{OliO (Ours)}} & \multicolumn{2}{c}{TS2V \cite{TS2Vec}} & \multicolumn{2}{c}{MTM \cite{SimMTM}}\\
\cmidrule(lr){3-4} \cmidrule(lr){5-6} \cmidrule(lr){7-8} & & MSE & MAE & MSE & MAE & MSE & MAE\\

\midrule
 \multirow{4}{*}{\rotatebox{90}{ETTh1}} & 96  & .4168$\pm$.0021 & .4311$\pm$.0009 & \textbf{.4128$\pm$.0017} & \textbf{.4258$\pm$.0009} & .6222$\pm$.0138 & .5332$\pm$.0055\\
& 192 & \underline{.4463$\pm$.0002} & .4482$\pm$.0001 & \textbf{.4446$\pm$.0021} & \textbf{.4456$\pm$.0016} & .6465$\pm$.0055 & .5461$\pm$.0024\\
& 336 & \underline{.4623$\pm$.0020} & \underline{.4593$\pm$.0003} & \textbf{.4595$\pm$.0028} & \textbf{.4565$\pm$.0010} & .6518$\pm$.0005 & .5535$\pm$.0006\\
& 720 & \underline{.4641$\pm$.0009} & .4793$\pm$.0002 & \textbf{.4620$\pm$.0011} & \textbf{.4775$\pm$.0005} & .6685$\pm$.0042 & .5770$\pm$.0015\\

\midrule
 \multirow{4}{*}{\rotatebox{90}{ETTh2}} & 96  & .3077$\pm$.0101 & .3636$\pm$.0061 & \textbf{.3011$\pm$.0054} & \textbf{.3580$\pm$.0037} & .3894$\pm$.0008 & .4228$\pm$.0009\\
& 192 & .3650$\pm$.0078 & .3996$\pm$.0047 & \underline{.3575$\pm$.0017} & \textbf{.3942$\pm$.0018} & .4156$\pm$.0010 & .4396$\pm$.0007\\
\bottomrule

\end{tabular}
\end{adjustbox}
\end{table*}

\begin{table*}[!htbp]
\centering
\caption{5\% few-shot forecasting results for PatchTST~\cite{PatchTST} backbone (part 2). SoftCLT~\cite{SoftCLT} is denoted as SCLT, TimeSiam~\cite{TimeSiam} as Siam, and TimesURL~\cite{TimesURL} as URL. Best results are in \textbf{bold} and the second best are \underline{underlined}. (OliO, TS2V, and MTM results are in Table~\ref{tab:fewshot_05_patchtst_p1}.)}
\label{tab:fewshot_05_patchtst_p2}
\begin{adjustbox}{width=1.0\linewidth}
\begin{tabular}{llcccccc}
\toprule
 \multicolumn{2}{c}{\multirow{2}{*}{Dataset}} & \multicolumn{2}{c}{SCLT \cite{SoftCLT}} & \multicolumn{2}{c}{Siam \cite{TimeSiam}} & \multicolumn{2}{c}{URL \cite{TimesURL}}\\
\cmidrule(lr){3-4} \cmidrule(lr){5-6} \cmidrule(lr){7-8} & & MSE & MAE & MSE & MAE & MSE & MAE\\

\midrule
 \multirow{4}{*}{\rotatebox{90}{ETTh1}} & 96  & .4564$\pm$.0002 & .4543$\pm$.0002 & .5686$\pm$.0011 & .5062$\pm$.0007 & \underline{.4138$\pm$.0033} & \underline{.4275$\pm$.0020}\\
& 192 & .4851$\pm$.0011 & .4704$\pm$.0007 & .5866$\pm$.0015 & .5181$\pm$.0006 & .4480$\pm$.0033 & \underline{.4475$\pm$.0027}\\
& 336 & .4991$\pm$.0037 & .4806$\pm$.0016 & .5877$\pm$.0024 & .5240$\pm$.0013 & .4656$\pm$.0021 & .4601$\pm$.0001\\
& 720 & .5038$\pm$.0033 & .5011$\pm$.0016 & .6054$\pm$.0005 & .5503$\pm$.0003 & .4653$\pm$.0026 & \underline{.4791$\pm$.0014}\\

\midrule
 \multirow{4}{*}{\rotatebox{90}{ETTh2}} & 96  & .3331$\pm$.0047 & .3822$\pm$.0021 & .3580$\pm$.0014 & .4015$\pm$.0017 & \underline{.3024$\pm$.0065} & \underline{.3614$\pm$.0043}\\
& 192 & .3844$\pm$.0040 & .4165$\pm$.0031 & .3958$\pm$.0061 & .4248$\pm$.0040 & \textbf{.3560$\pm$.0011} & \underline{.3957$\pm$.0011}\\
\bottomrule

\end{tabular}
\end{adjustbox}
\end{table*}

\clearpage
\begin{table*}[!htbp]
\centering
\caption{10\% few-shot forecasting results for PatchTST~\cite{PatchTST} backbone (part 1). TS2Vec~\cite{TS2Vec} is denoted as TS2V, and SimMTM~\cite{SimMTM} as MTM. Best results are in \textbf{bold} and the second best are \underline{underlined}. (SCLT, Siam, and URL results are in Table~\ref{tab:fewshot_10_patchtst_p2}.)}
\label{tab:fewshot_10_patchtst_p1}
\begin{adjustbox}{width=1.0\linewidth}
\begin{tabular}{llcccccc}
\toprule
 \multicolumn{2}{c}{\multirow{2}{*}{Dataset}} & \multicolumn{2}{c}{\textbf{OliO (Ours)}} & \multicolumn{2}{c}{TS2V \cite{TS2Vec}} & \multicolumn{2}{c}{MTM \cite{SimMTM}}\\
\cmidrule(lr){3-4} \cmidrule(lr){5-6} \cmidrule(lr){7-8} & & MSE & MAE & MSE & MAE & MSE & MAE\\

\midrule
 \multirow{4}{*}{\rotatebox{90}{ETTh1}} & 96  & \underline{.4063$\pm$.0053} & \underline{.4233$\pm$.0037} & .4087$\pm$.0105 & .4236$\pm$.0069 & .5325$\pm$.0088 & .4949$\pm$.0044\\
& 192 & .4328$\pm$.0037 & .4380$\pm$.0030 & \underline{.4320$\pm$.0072} & \underline{.4362$\pm$.0053} & .5579$\pm$.0048 & .5083$\pm$.0025\\
& 336 & \underline{.4484$\pm$.0035} & \textbf{.4484$\pm$.0019} & \textbf{.4477$\pm$.0029} & \underline{.4490$\pm$.0018} & .5623$\pm$.0036 & .5147$\pm$.0021\\
& 720 & \textbf{.4497$\pm$.0032} & \textbf{.4694$\pm$.0024} & \underline{.4505$\pm$.0046} & \underline{.4701$\pm$.0032} & .5729$\pm$.0024 & .5371$\pm$.0140\\

\midrule
 \multirow{4}{*}{\rotatebox{90}{ETTh2}} & 96  & .2856$\pm$.0031 & .3484$\pm$.0021 & \textbf{.2843$\pm$.0020} & \textbf{.3467$\pm$.0012} & .3643$\pm$.0018 & .4051$\pm$.0012\\
& 192 & .3473$\pm$.0020 & .3904$\pm$.0024 & \underline{.3456$\pm$.0017} & \textbf{.3857$\pm$.0021} & .4000$\pm$.0014 & .4282$\pm$.0009\\
\bottomrule

\end{tabular}
\end{adjustbox}
\end{table*}

\begin{table*}[!htbp]
\centering
\caption{10\% few-shot forecasting results for PatchTST~\cite{PatchTST} backbone (part 2). SoftCLT~\cite{SoftCLT} is denoted as SCLT, TimeSiam~\cite{TimeSiam} as Siam, and TimesURL~\cite{TimesURL} as URL. Best results are in \textbf{bold} and the second best are \underline{underlined}. (OliO, TS2V, and MTM results are in Table~\ref{tab:fewshot_10_patchtst_p1}.)}
\label{tab:fewshot_10_patchtst_p2}
\begin{adjustbox}{width=1.0\linewidth}
\begin{tabular}{llcccccc}
\toprule
 \multicolumn{2}{c}{\multirow{2}{*}{Dataset}} & \multicolumn{2}{c}{SCLT \cite{SoftCLT}} & \multicolumn{2}{c}{Siam \cite{TimeSiam}} & \multicolumn{2}{c}{URL \cite{TimesURL}}\\
\cmidrule(lr){3-4} \cmidrule(lr){5-6} \cmidrule(lr){7-8} & & MSE & MAE & MSE & MAE & MSE & MAE\\

\midrule
 \multirow{4}{*}{\rotatebox{90}{ETTh1}} & 96  & .4246$\pm$.0021 & .4348$\pm$.0010 & .4869$\pm$.0007 & .4693$\pm$.0004 & \textbf{.4039$\pm$.0158} & \textbf{.4202$\pm$.0109}\\
& 192 & .4513$\pm$.0009 & .4497$\pm$.0008 & .5070$\pm$.0009 & .4810$\pm$.0003 & \textbf{.4256$\pm$.0010} & \textbf{.4311$\pm$.0007}\\
& 336 & .4670$\pm$.0018 & .4612$\pm$.0006 & .5094$\pm$.0021 & .4859$\pm$.0007 & .4537$\pm$.0022 & .4535$\pm$.0021\\
& 720 & .4688$\pm$.0024 & .4819$\pm$.0014 & .5172$\pm$.0009 & .5084$\pm$.0009 & .4541$\pm$.0064 & .4717$\pm$.0040\\

\midrule
 \multirow{4}{*}{\rotatebox{90}{ETTh2}} & 96  & .3026$\pm$.0023 & .3607$\pm$.0140 & .3283$\pm$.0007 & .3794$\pm$.0007 & \underline{.2851$\pm$.0021} & \underline{.3476$\pm$.0013}\\
& 192 & .3610$\pm$.0015 & .3993$\pm$.0011 & .3828$\pm$.0038 & .4132$\pm$.0035 & \textbf{.3450$\pm$.0010} & \underline{.3867$\pm$.0017}\\
\bottomrule

\end{tabular}
\end{adjustbox}
\end{table*}

\clearpage

\subsection{iTransformer few-shot results}
\begin{table*}[!htbp]
\centering
\caption{1\% few-shot forecasting results for iTransformer~\cite{iTransformer} backbone (part 1). TS2Vec~\cite{TS2Vec} is denoted as TS2V, and SimMTM~\cite{SimMTM} as MTM. Best results are in \textbf{bold} and the second best are \underline{underlined}. (SCLT, Siam, and URL results are in Table~\ref{tab:fewshot_01_itransformer_p2}.)}
\label{tab:fewshot_01_itransformer_p1}
\begin{adjustbox}{width=1.0\linewidth}
\begin{tabular}{llcccccc}
\toprule
 \multicolumn{2}{c}{\multirow{2}{*}{Dataset}} & \multicolumn{2}{c}{\textbf{OliO (Ours)}} & \multicolumn{2}{c}{TS2V \cite{TS2Vec}} & \multicolumn{2}{c}{MTM \cite{SimMTM}}\\
\cmidrule(lr){3-4} \cmidrule(lr){5-6} \cmidrule(lr){7-8} & & MSE & MAE & MSE & MAE & MSE & MAE\\

\midrule
 \multirow{4}{*}{\rotatebox{90}{ETTh1}} & 96  & \underline{.6050$\pm$.0112} & \underline{.5178$\pm$.0026} & .6387$\pm$.0086 & .5302$\pm$.0014 & .8185$\pm$.0045 & .6034$\pm$.0006\\
& 192 & \underline{.6829$\pm$.0123} & \underline{.5564$\pm$.0047} & .7039$\pm$.0110 & .5629$\pm$.0039 & .8576$\pm$.0042 & .6233$\pm$.0019\\
& 336 & \underline{.7136$\pm$.0027} & \underline{.5753$\pm$.0002} & .7272$\pm$.0016 & .5793$\pm$.0001 & .8667$\pm$.0035 & .6338$\pm$.0020\\
& 720 & .7371$\pm$.0196 & .5995$\pm$.0079 & \underline{.7356$\pm$.0118} & \underline{.5978$\pm$.0042} & .8759$\pm$.0025 & .6503$\pm$.0013\\

\midrule
 \multirow{4}{*}{\rotatebox{90}{ETTh2}} & 96  & .3539$\pm$.0082 & \underline{.3878$\pm$.0052} & \underline{.3526$\pm$.0167} & .3898$\pm$.0098 & .3918$\pm$.0006 & .4114$\pm$.0004\\
& 192 & .4331$\pm$.0057 & .4322$\pm$.0039 & \underline{.4277$\pm$.0009} & \underline{.4311$\pm$.0014} & .4768$\pm$.0006 & .4575$\pm$.0008\\
\bottomrule

\end{tabular}
\end{adjustbox}
\end{table*}

\begin{table*}[!htbp]
\centering
\caption{1\% few-shot forecasting results for iTransformer~\cite{iTransformer} backbone (part 2). SoftCLT~\cite{SoftCLT} is denoted as SCLT, TimeSiam~\cite{TimeSiam} as Siam, and TimesURL~\cite{TimesURL} as URL. Best results are in \textbf{bold} and the second best are \underline{underlined}. (OliO, TS2V, and MTM results are in Table~\ref{tab:fewshot_01_itransformer_p1}.)}
\label{tab:fewshot_01_itransformer_p2}
\begin{adjustbox}{width=1.0\linewidth}
\begin{tabular}{llcccccc}
\toprule
 \multicolumn{2}{c}{\multirow{2}{*}{Dataset}} & \multicolumn{2}{c}{SCLT \cite{SoftCLT}} & \multicolumn{2}{c}{Siam \cite{TimeSiam}} & \multicolumn{2}{c}{URL \cite{TimesURL}}\\
\cmidrule(lr){3-4} \cmidrule(lr){5-6} \cmidrule(lr){7-8} & & MSE & MAE & MSE & MAE & MSE & MAE\\

\midrule
 \multirow{4}{*}{\rotatebox{90}{ETTh1}} & 96  & .7401$\pm$.0010 & .5718$\pm$.0022 & .7335$\pm$.0084 & .5706$\pm$.0038 & \textbf{.5008$\pm$.0050} & \textbf{.4761$\pm$.0014}\\
& 192 & .7895$\pm$.0059 & .5962$\pm$.0026 & .7744$\pm$.0064 & .5912$\pm$.0015 & \textbf{.5601$\pm$.0062} & \textbf{.5068$\pm$.0019}\\
& 336 & .8151$\pm$.0005 & .6137$\pm$.0010 & .7879$\pm$.0048 & .6045$\pm$.0024 & \textbf{.5975$\pm$.0056} & \textbf{.5248$\pm$.0032}\\
& 720 & .8253$\pm$.0026 & .6322$\pm$.0014 & .7917$\pm$.0044 & .6207$\pm$.0018 & \textbf{.6108$\pm$.0017} & \textbf{.5466$\pm$.0009}\\

\midrule
 \multirow{4}{*}{\rotatebox{90}{ETTh2}} & 96  & .3634$\pm$.0045 & .3956$\pm$.0028 & .3755$\pm$.0031 & .4015$\pm$.0027 & \textbf{.3377$\pm$.0065} & \textbf{.3763$\pm$.0051}\\
& 192 & .4483$\pm$.0025 & .4418$\pm$.0018 & .4578$\pm$.0052 & .4465$\pm$.0043 & \textbf{.4251$\pm$.0048} & \textbf{.4264$\pm$.0039}\\
\bottomrule

\end{tabular}
\end{adjustbox}
\end{table*}

\clearpage
\begin{table*}[!htbp]
\centering
\caption{5\% few-shot forecasting results for iTransformer~\cite{iTransformer} backbone (part 1). TS2Vec~\cite{TS2Vec} is denoted as TS2V, and SimMTM~\cite{SimMTM} as MTM. Best results are in \textbf{bold} and the second best are \underline{underlined}. (SCLT, Siam, and URL results are in Table~\ref{tab:fewshot_05_itransformer_p2}.)}
\label{tab:fewshot_05_itransformer_p1}
\begin{adjustbox}{width=1.0\linewidth}
\begin{tabular}{llcccccc}
\toprule
 \multicolumn{2}{c}{\multirow{2}{*}{Dataset}} & \multicolumn{2}{c}{\textbf{OliO (Ours)}} & \multicolumn{2}{c}{TS2V \cite{TS2Vec}} & \multicolumn{2}{c}{MTM \cite{SimMTM}}\\
\cmidrule(lr){3-4} \cmidrule(lr){5-6} \cmidrule(lr){7-8} & & MSE & MAE & MSE & MAE & MSE & MAE\\

\midrule
 \multirow{4}{*}{\rotatebox{90}{ETTh1}} & 96  & \underline{.4375$\pm$.0008} & \underline{.4371$\pm$.0008} & .4876$\pm$.0014 & .4671$\pm$.0008 & .6023$\pm$.0027 & .5206$\pm$.0016\\
& 192 & \underline{.5082$\pm$.0034} & \underline{.4778$\pm$.0022} & .5481$\pm$.0018 & .4991$\pm$.0003 & .6923$\pm$.0014 & .5611$\pm$.0006\\
& 336 & \underline{.5559$\pm$.0015} & \underline{.5014$\pm$.0007} & .5878$\pm$.0033 & .5192$\pm$.0024 & .7224$\pm$.0032 & .5796$\pm$.0007\\
& 720 & \underline{.5948$\pm$.0107} & \underline{.5394$\pm$.0063} & .6097$\pm$.0068 & .5469$\pm$.0038 & .7577$\pm$.0025 & .6082$\pm$.0010\\

\midrule
 \multirow{4}{*}{\rotatebox{90}{ETTh2}} & 96  & \underline{.3206$\pm$.0059} & \underline{.3649$\pm$.0054} & .3304$\pm$.0070 & .3758$\pm$.0058 & .3609$\pm$.0031 & .3917$\pm$.0020\\
& 192 & \underline{.4016$\pm$.0031} & \underline{.4120$\pm$.0013} & .4149$\pm$.0022 & .4232$\pm$.0022 & .4505$\pm$.0022 & .4423$\pm$.0015\\
\bottomrule

\end{tabular}
\end{adjustbox}
\end{table*}

\begin{table*}[!htbp]
\centering
\caption{5\% few-shot forecasting results for iTransformer~\cite{iTransformer} backbone (part 2). SoftCLT~\cite{SoftCLT} is denoted as SCLT, TimeSiam~\cite{TimeSiam} as Siam, and TimesURL~\cite{TimesURL} as URL. Best results are in \textbf{bold} and the second best are \underline{underlined}. (OliO, TS2V, and MTM results are in Table~\ref{tab:fewshot_05_itransformer_p1}.)}
\label{tab:fewshot_05_itransformer_p2}
\begin{adjustbox}{width=1.0\linewidth}
\begin{tabular}{llcccccc}
\toprule
 \multicolumn{2}{c}{\multirow{2}{*}{Dataset}} & \multicolumn{2}{c}{SCLT \cite{SoftCLT}} & \multicolumn{2}{c}{Siam \cite{TimeSiam}} & \multicolumn{2}{c}{URL \cite{TimesURL}}\\
\cmidrule(lr){3-4} \cmidrule(lr){5-6} \cmidrule(lr){7-8} & & MSE & MAE & MSE & MAE & MSE & MAE\\

\midrule
 \multirow{4}{*}{\rotatebox{90}{ETTh1}} & 96  & .5757$\pm$.0060 & .5036$\pm$.0047 & .5393$\pm$.0060 & .4894$\pm$.0033 & \textbf{.4161$\pm$.0034} & \textbf{.4238$\pm$.0024}\\
& 192 & .6608$\pm$.0097 & .5456$\pm$.0034 & .6235$\pm$.0048 & .5303$\pm$.0012 & \textbf{.4751$\pm$.0020} & \textbf{.4581$\pm$.0014}\\
& 336 & .6902$\pm$.0052 & .5686$\pm$.0037 & .6576$\pm$.0026 & .5515$\pm$.0018 & \textbf{.5182$\pm$.0027} & \textbf{.4803$\pm$.0016}\\
& 720 & .7183$\pm$.0058 & .5986$\pm$.0042 & .6784$\pm$.0039 & .5779$\pm$.0021 & \textbf{.5554$\pm$.0039} & \textbf{.5193$\pm$.0020}\\

\midrule
 \multirow{4}{*}{\rotatebox{90}{ETTh2}} & 96  & .3266$\pm$.0058 & .3706$\pm$.0048 & .3381$\pm$.0055 & .3773$\pm$.0047 & \textbf{.3134$\pm$.0051} & \textbf{.3588$\pm$.0039}\\
& 192 & .4094$\pm$.0023 & .4187$\pm$.0023 & .4261$\pm$.0035 & .4272$\pm$.0028 & \textbf{.3905$\pm$.0022} & \textbf{.4051$\pm$.0016}\\
\bottomrule

\end{tabular}
\end{adjustbox}
\end{table*}

\clearpage
\begin{table*}[!htbp]
\centering
\caption{10\% few-shot forecasting results for iTransformer~\cite{iTransformer} backbone (part 1). TS2Vec~\cite{TS2Vec} is denoted as TS2V, and SimMTM~\cite{SimMTM} as MTM. Best results are in \textbf{bold} and the second best are \underline{underlined}. (SCLT, Siam, and URL results are in Table~\ref{tab:fewshot_10_itransformer_p2}.)}
\label{tab:fewshot_10_itransformer_p1}
\begin{adjustbox}{width=1.0\linewidth}
\begin{tabular}{llcccccc}
\toprule
 \multicolumn{2}{c}{\multirow{2}{*}{Dataset}} & \multicolumn{2}{c}{\textbf{OliO (Ours)}} & \multicolumn{2}{c}{TS2V \cite{TS2Vec}} & \multicolumn{2}{c}{MTM \cite{SimMTM}}\\
\cmidrule(lr){3-4} \cmidrule(lr){5-6} \cmidrule(lr){7-8} & & MSE & MAE & MSE & MAE & MSE & MAE\\

\midrule
 \multirow{4}{*}{\rotatebox{90}{ETTh1}} & 96  & \underline{.4119$\pm$.0018} & \underline{.4220$\pm$.0008} & .4389$\pm$.0028 & .4404$\pm$.0020 & .4907$\pm$.0014 & .4694$\pm$.0009\\
& 192 & \underline{.4647$\pm$.0019} & \underline{.4517$\pm$.0012} & .5096$\pm$.0044 & .4807$\pm$.0018 & .5669$\pm$.0036 & .5082$\pm$.0007\\
& 336 & \underline{.5113$\pm$.0013} & \underline{.4757$\pm$.0006} & .5544$\pm$.0035 & .5028$\pm$.0023 & .6115$\pm$.0009 & .5313$\pm$.0009\\
& 720 & \underline{.5475$\pm$.0023} & \underline{.5129$\pm$.0014} & .5768$\pm$.0024 & .5305$\pm$.0019 & .6499$\pm$.0041 & .5660$\pm$.0024\\

\midrule
 \multirow{4}{*}{\rotatebox{90}{ETTh2}} & 96  & \underline{.3080$\pm$.0022} & \underline{.3557$\pm$.0013} & .3198$\pm$.0098 & .3698$\pm$.0072 & .3378$\pm$.0018 & .3763$\pm$.0011\\
& 192 & \textbf{.3892$\pm$.0019} & \underline{.4055$\pm$.0010} & .4102$\pm$.0023 & .4204$\pm$.0029 & .4324$\pm$.0032 & .4310$\pm$.0024\\
\bottomrule

\end{tabular}
\end{adjustbox}
\end{table*}

\begin{table*}[!htbp]
\centering
\caption{10\% few-shot forecasting results for iTransformer~\cite{iTransformer} backbone (part 2). SoftCLT~\cite{SoftCLT} is denoted as SCLT, TimeSiam~\cite{TimeSiam} as Siam, and TimesURL~\cite{TimesURL} as URL. Best results are in \textbf{bold} and the second best are \underline{underlined}. (OliO, TS2V, and MTM results are in Table~\ref{tab:fewshot_10_itransformer_p1}.)}
\label{tab:fewshot_10_itransformer_p2}
\begin{adjustbox}{width=1.0\linewidth}
\begin{tabular}{llcccccc}
\toprule
 \multicolumn{2}{c}{\multirow{2}{*}{Dataset}} & \multicolumn{2}{c}{SCLT \cite{SoftCLT}} & \multicolumn{2}{c}{Siam \cite{TimeSiam}} & \multicolumn{2}{c}{URL \cite{TimesURL}}\\
\cmidrule(lr){3-4} \cmidrule(lr){5-6} \cmidrule(lr){7-8} & & MSE & MAE & MSE & MAE & MSE & MAE\\

\midrule
 \multirow{4}{*}{\rotatebox{90}{ETTh1}} & 96  & .4820$\pm$.0011 & .4640$\pm$.0007 & .4607$\pm$.0022 & .4512$\pm$.0010 & \textbf{.4058$\pm$.0015} & \textbf{.4169$\pm$.0005}\\
& 192 & .5458$\pm$.0117 & .4960$\pm$.0046 & .5321$\pm$.0014 & .4893$\pm$.0007 & \textbf{.4555$\pm$.0022} & \textbf{.4457$\pm$.0013}\\
& 336 & .5867$\pm$.0091 & .5176$\pm$.0048 & .5808$\pm$.0011 & .5153$\pm$.0010 & \textbf{.4993$\pm$.0032} & \textbf{.4678$\pm$.0018}\\
& 720 & .6299$\pm$.0101 & .5571$\pm$.0051 & .6090$\pm$.0027 & .5472$\pm$.0017 & \textbf{.5287$\pm$.0024} & \textbf{.5033$\pm$.0018}\\

\midrule
 \multirow{4}{*}{\rotatebox{90}{ETTh2}} & 96  & .3092$\pm$.0018 & .3588$\pm$.0008 & .3186$\pm$.0015 & .3629$\pm$.0016 & \textbf{.3046$\pm$.0020} & \textbf{.3533$\pm$.0017}\\
& 192 & .3973$\pm$.0030 & .4111$\pm$.0034 & .4103$\pm$.0020 & .4171$\pm$.0025 & \underline{.3896$\pm$.0033} & \textbf{.4053$\pm$.0017}\\
\bottomrule

\end{tabular}
\end{adjustbox}
\end{table*}

\clearpage

\subsection{TimeVLM few-shot results}
\begin{table*}[!htbp]
\centering
\caption{1\% few-shot forecasting results for TimeVLM~\cite{TimeVLM} backbone (part 1). TS2Vec~\cite{TS2Vec} is denoted as TS2V, and SimMTM~\cite{SimMTM} as MTM. Best results are in \textbf{bold} and the second best are \underline{underlined}. (SCLT, Siam, and URL results are in Table~\ref{tab:fewshot_01_timevlm_p2}.)}
\label{tab:fewshot_01_timevlm_p1}
\begin{adjustbox}{width=1.0\linewidth}
\begin{tabular}{llcccccc}
\toprule
 \multicolumn{2}{c}{\multirow{2}{*}{Dataset}} & \multicolumn{2}{c}{\textbf{OliO (Ours)}} & \multicolumn{2}{c}{TS2V \cite{TS2Vec}} & \multicolumn{2}{c}{MTM \cite{SimMTM}}\\
\cmidrule(lr){3-4} \cmidrule(lr){5-6} \cmidrule(lr){7-8} & & MSE & MAE & MSE & MAE & MSE & MAE\\

\midrule
 \multirow{4}{*}{\rotatebox{90}{ETTh1}} & 96 & \textbf{.6168$\pm$.0133} & \textbf{.5361$\pm$.0072} & .6407$\pm$.0239 & \underline{.5452$\pm$.0107} & .6561$\pm$.0172 & .5599$\pm$.0088\\
& 192 & \textbf{.6362$\pm$.0132} & \textbf{.5501$\pm$.0072} & .6533$\pm$.0079 & .5542$\pm$.0059 & .6821$\pm$.0118 & .5702$\pm$.0080\\
& 336 & \textbf{.6101$\pm$.0128} & \textbf{.5442$\pm$.0098} & .6314$\pm$.0085 & .5613$\pm$.0051 & .6482$\pm$.0102 & .5744$\pm$.0065\\
& 720 & \textbf{.6476$\pm$.0060} & \textbf{.5796$\pm$.0026} & .6671$\pm$.0075 & \underline{.5886$\pm$.0039} & .7278$\pm$.0080 & .6241$\pm$.0051\\

\midrule
 \multirow{4}{*}{\rotatebox{90}{ETTh2}} & 96 & \textbf{.3746$\pm$.0068} & \textbf{.4174$\pm$.0049} & \underline{.3753$\pm$.0087} & \underline{.4182$\pm$.0059} & .4117$\pm$.0117 & .4416$\pm$.0074\\
& 192 & \underline{.3879$\pm$.0027} & \underline{.4260$\pm$.0019} & \textbf{.3876$\pm$.0029} & \textbf{.4256$\pm$.0021} & .4536$\pm$.0157 & .4720$\pm$.0115\\
\bottomrule

\end{tabular}
\end{adjustbox}
\end{table*}

\begin{table*}[!htbp]
\centering
\caption{1\% few-shot forecasting results for TimeVLM~\cite{TimeVLM} backbone (part 2). SoftCLT~\cite{SoftCLT} is denoted as SCLT, TimeSiam~\cite{TimeSiam} as Siam, and TimesURL~\cite{TimesURL} as URL. Best results are in \textbf{bold} and the second best are \underline{underlined}. (OliO, TS2V, and MTM results are in Table~\ref{tab:fewshot_01_timevlm_p1}.)}
\label{tab:fewshot_01_timevlm_p2}
\begin{adjustbox}{width=1.0\linewidth}
\begin{tabular}{llcccccc}
\toprule
 \multicolumn{2}{c}{\multirow{2}{*}{Dataset}} & \multicolumn{2}{c}{SCLT \cite{SoftCLT}} & \multicolumn{2}{c}{Siam \cite{TimeSiam}} & \multicolumn{2}{c}{URL \cite{TimesURL}}\\
\cmidrule(lr){3-4} \cmidrule(lr){5-6} \cmidrule(lr){7-8} & & MSE & MAE & MSE & MAE & MSE & MAE\\

\midrule
 \multirow{4}{*}{\rotatebox{90}{ETTh1}} & 96 & \underline{.6278$\pm$.0114} & .5524$\pm$.0097 & .6714$\pm$.0141 & .5924$\pm$.0072 & .6714$\pm$.0141 & .5924$\pm$.0072\\
& 192 & .6511$\pm$.0076 & .5540$\pm$.0037 & .6923$\pm$.0114 & .5858$\pm$.0061 & \underline{.6498$\pm$.0096} & \underline{.5533$\pm$.0055}\\
& 336 & \underline{.6199$\pm$.0068} & \underline{.5522$\pm$.0038} & .6703$\pm$.0102 & .5952$\pm$.0086 & .6270$\pm$.0085 & .5573$\pm$.0066\\
& 720 & \underline{.6584$\pm$.0047} & .5895$\pm$.0017 & .7321$\pm$.0080 & .6244$\pm$.0055 & .6633$\pm$.0093 & .5911$\pm$.0044\\

\midrule
 \multirow{4}{*}{\rotatebox{90}{ETTh2}} & 96 & .3776$\pm$.0058 & .4196$\pm$.0034 & .4201$\pm$.0069 & .4498$\pm$.0049 & .3773$\pm$.0054 & .4194$\pm$.0036\\
& 192 & .3889$\pm$.0037 & .4278$\pm$.0027 & .4498$\pm$.0086 & .4758$\pm$.0063 & .3880$\pm$.0031 & .4268$\pm$.0023\\
\bottomrule

\end{tabular}
\end{adjustbox}
\end{table*}

\clearpage
\begin{table*}[!htbp]
\centering
\caption{5\% few-shot forecasting results for TimeVLM~\cite{TimeVLM} backbone (part 1). TS2Vec~\cite{TS2Vec} is denoted as TS2V, and SimMTM~\cite{SimMTM} as MTM. Best results are in \textbf{bold} and the second best are \underline{underlined}. (SCLT, Siam, and URL results are in Table~\ref{tab:fewshot_05_timevlm_p2}.)}
\label{tab:fewshot_05_timevlm_p1}
\begin{adjustbox}{width=1.0\linewidth}
\begin{tabular}{llcccccc}
\toprule
 \multicolumn{2}{c}{\multirow{2}{*}{Dataset}} & \multicolumn{2}{c}{\textbf{OliO (Ours)}} & \multicolumn{2}{c}{TS2V \cite{TS2Vec}} & \multicolumn{2}{c}{MTM \cite{SimMTM}}\\
\cmidrule(lr){3-4} \cmidrule(lr){5-6} \cmidrule(lr){7-8} & & MSE & MAE & MSE & MAE & MSE & MAE\\

\midrule
 \multirow{4}{*}{\rotatebox{90}{ETTh1}} & 96  & \textbf{.4274$\pm$.0069} & \textbf{.4399$\pm$.0035} & .4343$\pm$.0076 & .4444$\pm$.0041 & .6452$\pm$.0063 & .5503$\pm$.0028\\
& 192 & \textbf{.4433$\pm$.0041} & \textbf{.4511$\pm$.0025} & .4539$\pm$.0040 & .4576$\pm$.0024 & .6613$\pm$.0057 & .5649$\pm$.0039\\
& 336 & \textbf{.4517$\pm$.0051} & \textbf{.4602$\pm$.0038} & .4611$\pm$.0049 & .4653$\pm$.0035 & .6546$\pm$.0068 & .5667$\pm$.0063\\
& 720 & \textbf{.4869$\pm$.0051} & \textbf{.4939$\pm$.0033} & \underline{.4976$\pm$.0049} & \underline{.4989$\pm$.0032} & .6860$\pm$.0022 & .6009$\pm$.0019\\

\midrule
 \multirow{4}{*}{\rotatebox{90}{ETTh2}} & 96  & \underline{.3161$\pm$.0019} & \textbf{.3685$\pm$.0005} & .3165$\pm$.0012 & .3690$\pm$.0002 & .3653$\pm$.0024 & .4136$\pm$.0023\\
& 192 & .3626$\pm$.0030 & .3998$\pm$.0005 & \textbf{.3623$\pm$.0030} & .3987$\pm$.0006 & .3829$\pm$.0014 & .4244$\pm$.0015\\
\bottomrule

\end{tabular}
\end{adjustbox}
\end{table*}

\begin{table*}[!htbp]
\centering
\caption{5\% few-shot forecasting results for TimeVLM~\cite{TimeVLM} backbone (part 2). SoftCLT~\cite{SoftCLT} is denoted as SCLT, TimeSiam~\cite{TimeSiam} as Siam, and TimesURL~\cite{TimesURL} as URL. Best results are in \textbf{bold} and the second best are \underline{underlined}. (OliO, TS2V, and MTM results are in Table~\ref{tab:fewshot_05_timevlm_p1}.)}
\label{tab:fewshot_05_timevlm_p2}
\begin{adjustbox}{width=1.0\linewidth}
\begin{tabular}{llcccccc}
\toprule
 \multicolumn{2}{c}{\multirow{2}{*}{Dataset}} & \multicolumn{2}{c}{SCLT \cite{SoftCLT}} & \multicolumn{2}{c}{Siam \cite{TimeSiam}} & \multicolumn{2}{c}{URL \cite{TimesURL}}\\
\cmidrule(lr){3-4} \cmidrule(lr){5-6} \cmidrule(lr){7-8} & & MSE & MAE & MSE & MAE & MSE & MAE\\

\midrule
 \multirow{4}{*}{\rotatebox{90}{ETTh1}} & 96  & .4357$\pm$.0069 & .4453$\pm$.0035 & .6357$\pm$.0014 & .5453$\pm$.0010 & \underline{.4319$\pm$.0055} & \underline{.4435$\pm$.0032}\\
& 192 & .4541$\pm$.0041 & .4576$\pm$.0025 & .6581$\pm$.0054 & .5618$\pm$.0035 & \underline{.4524$\pm$.0036} & \underline{.4566$\pm$.0022}\\
& 336 & .4614$\pm$.0051 & .4656$\pm$.0038 & .6566$\pm$.0019 & .5692$\pm$.0020 & \underline{.4604$\pm$.0042} & \underline{.4649$\pm$.0032}\\
& 720 & .4984$\pm$.0051 & .4993$\pm$.0033 & .6821$\pm$.0046 & .5996$\pm$.0026 & .4978$\pm$.0045 & .4990$\pm$.0030\\

\midrule
 \multirow{4}{*}{\rotatebox{90}{ETTh2}} & 96  & \textbf{.3160$\pm$.0019} & \underline{.3686$\pm$.0005} & .3623$\pm$.0021 & .4112$\pm$.0130 & .3165$\pm$.0018 & .3689$\pm$.0006\\
& 192 & .3626$\pm$.0030 & .3987$\pm$.0005 & .3845$\pm$.0009 & .4264$\pm$.0130 & \underline{.3625$\pm$.0027} & \textbf{.3986$\pm$.0004}\\
\bottomrule

\end{tabular}
\end{adjustbox}
\end{table*}

\clearpage
\begin{table*}[!htbp]
\centering
\caption{10\% few-shot forecasting results for TimeVLM~\cite{TimeVLM} backbone (part 1). TS2Vec~\cite{TS2Vec} is denoted as TS2V, and SimMTM~\cite{SimMTM} as MTM. Best results are in \textbf{bold} and the second best are \underline{underlined}. (SCLT, Siam, and URL results are in Table~\ref{tab:fewshot_10_timevlm_p2}.)}
\label{tab:fewshot_10_timevlm_p1}
\begin{adjustbox}{width=1.0\linewidth}
\begin{tabular}{llcccccc}
\toprule
 \multicolumn{2}{c}{\multirow{2}{*}{Dataset}} & \multicolumn{2}{c}{\textbf{OliO (Ours)}} & \multicolumn{2}{c}{TS2V \cite{TS2Vec}} & \multicolumn{2}{c}{MTM \cite{SimMTM}}\\
\cmidrule(lr){3-4} \cmidrule(lr){5-6} \cmidrule(lr){7-8} & & MSE & MAE & MSE & MAE & MSE & MAE\\

\midrule
 \multirow{4}{*}{\rotatebox{90}{ETTh1}} & 96 & .3912$\pm$.0023 & \textbf{.4169$\pm$.0014} & \textbf{.3901$\pm$.0031} & \underline{.4176$\pm$.0027} & .5677$\pm$.0081 & .5124$\pm$.0039\\
& 192 & .4160$\pm$.0023 & .4343$\pm$.0020 & .4157$\pm$.0023 & .4338$\pm$.0014 & .5819$\pm$.0053 & .5244$\pm$.0037\\
& 336 & .4296$\pm$.0029 & .4449$\pm$.0023 & .4289$\pm$.0026 & .4444$\pm$.0019 & .5699$\pm$.0042 & .5262$\pm$.0034\\
& 720 & .4795$\pm$.0051 & .4939$\pm$.0033 & .4802$\pm$.0091 & .4896$\pm$.0060 & .5939$\pm$.0027 & .5557$\pm$.0014\\

\midrule
 \multirow{4}{*}{\rotatebox{90}{ETTh2}} & 96 & \textbf{.2877$\pm$.0010} & \textbf{.3484$\pm$.0015} & .2880$\pm$.0002 & .3487$\pm$.0013 & .3507$\pm$.0014 & .4007$\pm$.0013\\
& 192 & \underline{.3454$\pm$.0008} & .3848$\pm$.0006 & .3455$\pm$.0010 & .3848$\pm$.0005 & .3770$\pm$.0004 & .4181$\pm$.0006\\
\bottomrule

\end{tabular}
\end{adjustbox}
\end{table*}

\begin{table*}[!htbp]
\centering
\caption{10\% few-shot forecasting results for TimeVLM~\cite{TimeVLM} backbone (part 2). SoftCLT~\cite{SoftCLT} is denoted as SCLT, TimeSiam~\cite{TimeSiam} as Siam, and TimesURL~\cite{TimesURL} as URL. Best results are in \textbf{bold} and the second best are \underline{underlined}. (OliO, TS2V, and MTM results are in Table~\ref{tab:fewshot_10_timevlm_p1}.)}
\label{tab:fewshot_10_timevlm_p2}
\begin{adjustbox}{width=1.0\linewidth}
\begin{tabular}{llcccccc}
\toprule
 \multicolumn{2}{c}{\multirow{2}{*}{Dataset}} & \multicolumn{2}{c}{SCLT \cite{SoftCLT}} & \multicolumn{2}{c}{Siam \cite{TimeSiam}} & \multicolumn{2}{c}{URL \cite{TimesURL}}\\
\cmidrule(lr){3-4} \cmidrule(lr){5-6} \cmidrule(lr){7-8} & & MSE & MAE & MSE & MAE & MSE & MAE\\

\midrule
 \multirow{4}{*}{\rotatebox{90}{ETTh1}} & 96 & .3910$\pm$.0028 & .4183$\pm$.0025 & .5611$\pm$.0033 & .5108$\pm$.0021 & \underline{.3902$\pm$.0028} & .4177$\pm$.0025\\
& 192 & \underline{.4153$\pm$.0024} & \underline{.4335$\pm$.0016} & .5787$\pm$.0075 & .5218$\pm$.0043 & \textbf{.4151$\pm$.0022} & \textbf{.4334$\pm$.0014}\\
& 336 & .4289$\pm$.0025 & \underline{.4443$\pm$.0018} & .5746$\pm$.0030 & .5296$\pm$.0022 & \textbf{.4286$\pm$.0025} & \textbf{.4441$\pm$.0018}\\
& 720 & \textbf{.4780$\pm$.0076} & \textbf{.4882$\pm$.0049} & .5856$\pm$.0057 & .5509$\pm$.0034 & \underline{.4787$\pm$.0090} & \underline{.4888$\pm$.0061}\\

\midrule
 \multirow{4}{*}{\rotatebox{90}{ETTh2}} & 96 & .2880$\pm$.0002 & .3487$\pm$.0012 & .3475$\pm$.0018 & .3986$\pm$.0014 & .2885$\pm$.0005 & .3491$\pm$.0015\\
& 192 & \underline{.3454$\pm$.0003} & .3848$\pm$.0004 & .3766$\pm$.0004 & .4175$\pm$.0007 & .3457$\pm$.0012 & .3849$\pm$.0005\\
\bottomrule

\end{tabular}
\end{adjustbox}
\end{table*}

\clearpage

\section{Detailed Full-shot Forecasting Results}
\label{app:detailed_results}

In this section, we provide the full-shot forecasting results for all backbone architectures, including both MSE and MAE metrics. Note that while the MSE results are consistent with the main text, the MAE results are provided for a comprehensive comparison across all benchmarks.

\subsection{Effect size and statistical comparisons}
\label{app:fullshot_statistics}
We compare OliO with TimesURL using unrounded three-seed means. Across 72 backbone--dataset--horizon configurations, OliO wins 58, with mean, median, and maximum relative MSE gains of 0.88\%, 0.55\%, and 5.75\%. Counts use unrounded values and can therefore differ from ranks inferred from displayed rounded means.

Our primary analysis averages the four horizons within each of the 18 dataset--backbone groups: OliO wins 17 groups, and the two-sided exact Wilcoxon signed-rank test gives $W=3$, $p=3.81\times10^{-5}$. Averaging further over backbones gives six dataset-level pairs, all favoring OliO, with $W=0$, $p=0.03125$. The supplementary 72-configuration analysis gives $W=347$, $p=5.87\times10^{-9}$; its horizon-level entries need not be independent, so the grouped comparisons receive greater weight.

For a descriptive comparison with seed variation, let $\Delta=\bar m_{\mathrm{URL}}-\bar m_{\mathrm{OliO}}$ and let $s_{\mathrm{URL}}$ and $s_{\mathrm{OliO}}$ be the corresponding three-seed sample standard deviations. In 38 settings, $\Delta>\max(s_{\mathrm{URL}},s_{\mathrm{OliO}})$; the symmetric criterion favors TimesURL in five settings. This compares effect magnitude with observed seed variation and is not a per-setting significance test. The tables below report the original per-setting means and sample standard deviations; the main-text summary averages these quantities separately and does not estimate uncertainty of the aggregated mean.

\subsection{PatchTST Results}
\begin{table*}[h!]
\centering
\caption{Full-shot forecasting results for PatchTST~\cite{PatchTST} backbone (part 1). TS2Vec~\cite{TS2Vec} is denoted as TS2V. Best results are in \textbf{bold} and the second best are \underline{underlined}. (MTM, SCLT, Siam, and URL results are in Table~\ref{tab:full_shot_patchtst_p2}.)}
\label{tab:full_shot_patchtst_p1}
\begin{adjustbox}{width=1.0\linewidth}
\begin{tabular}{llcccccc}
\toprule
\multicolumn{2}{c}{\multirow{2}{*}{Dataset}} & \multicolumn{2}{c}{Rand.} & \multicolumn{2}{c}{\textbf{OliO (Ours)}} & \multicolumn{2}{c}{TS2V \cite{TS2Vec}} \\
\cmidrule(lr){3-4} \cmidrule(lr){5-6} \cmidrule(lr){7-8}
& & MSE & MAE & MSE & MAE & MSE & MAE \\ \midrule
\multirow{4}{*}{\rotatebox{90}{ETTh1}} & 96  & .3735 $\pm$ .0047 & .3973 $\pm$ .0030 & .3716 $\pm$ .0003 & .3975 $\pm$ .0001 & \underline{.3708 $\pm$ .0006} & \textbf{.3953 $\pm$ .0003} \\
& 192 & .4094 $\pm$ .0041 & .4191 $\pm$ .0023 & .4082 $\pm$ .0009 & \textbf{.4175 $\pm$ .0008} & \underline{.4080 $\pm$ .0008} & .4190 $\pm$ .0006 \\
& 336 & .4343 $\pm$ .0030 & .4367 $\pm$ .0024 & \textbf{.4252 $\pm$ .0000} & \textbf{.4315 $\pm$ .0002} & \underline{.4311 $\pm$ .0026} & .4369 $\pm$ .0021 \\
& 720 & .4582 $\pm$ .0089 & .4711 $\pm$ .0047 & \textbf{.4356 $\pm$ .0049} & \textbf{.4596 $\pm$ .0036} & \underline{.4483 $\pm$ .0016} & \underline{.4691 $\pm$ .0015} \\ \midrule
\multirow{4}{*}{\rotatebox{90}{ETTh2}} & 96  & .2833 $\pm$ .0093 & .3418 $\pm$ .0030 & \textbf{.2769 $\pm$ .0014} & .3400 $\pm$ .0004 & .2778 $\pm$ .0017 & \textbf{.3383 $\pm$ .0011} \\
& 192 & .3638 $\pm$ .0071 & .3908 $\pm$ .0020 & \textbf{.3392 $\pm$ .0002} & .3822 $\pm$ .0003 & \underline{.3406 $\pm$ .0001} & .3817 $\pm$ .0005 \\
& 336 & .3918 $\pm$ .0014 & .4131 $\pm$ .0011 & \textbf{.3599 $\pm$ .0004} & .4020 $\pm$ .0004 & .3606 $\pm$ .0003 & .4017 $\pm$ .0001 \\
& 720 & .3986 $\pm$ .0024 & .4357 $\pm$ .0015 & .3934 $\pm$ .0009 & .4326 $\pm$ .0010 & \textbf{.3910 $\pm$ .0020} & \underline{.4295 $\pm$ .0012} \\ \midrule
\multirow{4}{*}{\rotatebox{90}{ETTm1}} & 96  & .2859 $\pm$ .0007 & .3387 $\pm$ .0006 & \textbf{.2830 $\pm$ .0015} & \textbf{.3368 $\pm$ .0010} & .2860 $\pm$ .0006 & .3384 $\pm$ .0004 \\
& 192 & .3266 $\pm$ .0005 & \underline{.3653 $\pm$ .0003} & \textbf{.3242 $\pm$ .0007} & .3660 $\pm$ .0011 & .3279 $\pm$ .0002 & .3680 $\pm$ .0003 \\
& 336 & .3642 $\pm$ .0028 & \underline{.3886 $\pm$ .0015} & \underline{.3605 $\pm$ .0009} & \textbf{.3877 $\pm$ .0005} & .3668 $\pm$ .0012 & .3904 $\pm$ .0014 \\
& 720 & .4140 $\pm$ .0015 & .4204 $\pm$ .0010 & \textbf{.4060 $\pm$ .0008} & \textbf{.4174 $\pm$ .0006} & .4162 $\pm$ .0009 & .4221 $\pm$ .0013 \\ \midrule
\multirow{4}{*}{\rotatebox{90}{ETTm2}} & 96  & .1660 $\pm$ .0006 & .2561 $\pm$ .0004 & \textbf{.1636 $\pm$ .0002} & \textbf{.2535 $\pm$ .0001} & .1669 $\pm$ .0001 & .2574 $\pm$ .0001 \\
& 192 & .2243 $\pm$ .0029 & .2951 $\pm$ .0008 & \textbf{.2189 $\pm$ .0003} & \textbf{.2920 $\pm$ .0002} & .2220 $\pm$ .0003 & .2961 $\pm$ .0004 \\
& 336 & .2770 $\pm$ .0013 & \underline{.3299 $\pm$ .0006} & \textbf{.2727 $\pm$ .0003} & \textbf{.3268 $\pm$ .0002} & .2760 $\pm$ .0003 & .3310 $\pm$ .0004 \\
& 720 & \underline{.3670 $\pm$ .0017} & \underline{.3836 $\pm$ .0006} & \textbf{.3644 $\pm$ .0008} & \textbf{.3820 $\pm$ .0005} & .3684 $\pm$ .0002 & .3869 $\pm$ .0002 \\ \midrule
\multirow{4}{*}{\rotatebox{90}{Weather}} & 96  & .1529 $\pm$ .0006 & .2020 $\pm$ .0005 & \textbf{.1507 $\pm$ .0003} & \textbf{.1997 $\pm$ .0011} & .1615 $\pm$ .0005 & .2111 $\pm$ .0002 \\
& 192 & .1959 $\pm$ .0007 & .2412 $\pm$ .0002 & \textbf{.1948 $\pm$ .0004} & \textbf{.2403 $\pm$ .0005} & .2036 $\pm$ .0003 & .2485 $\pm$ .0003 \\
& 336 & .2478 $\pm$ .0008 & .2821 $\pm$ .0009 & \textbf{.2453 $\pm$ .0003} & \textbf{.2785 $\pm$ .0007} & .2528 $\pm$ .0005 & .2859 $\pm$ .0003 \\
& 720 & .3221 $\pm$ .0011 & .3347 $\pm$ .0013 & \textbf{.3193 $\pm$ .0009} & \textbf{.3318 $\pm$ .0006} & .3242 $\pm$ .0004 & .3360 $\pm$ .0004 \\ \midrule
\multirow{4}{*}{\rotatebox{90}{ECL}}     & 96  & .1356 $\pm$ .0000 & .2302 $\pm$ .0000 & \underline{.1297 $\pm$ .0002} & \underline{.2228 $\pm$ .0003} & .1314 $\pm$ .0002 & .2249 $\pm$ .0002 \\
& 192 & .1515 $\pm$ .0001 & .2444 $\pm$ .0001 & \underline{.1476 $\pm$ .0001} & .2401 $\pm$ .0001 & .1481 $\pm$ .0004 & .2401 $\pm$ .0002 \\
& 336 & .1682 $\pm$ .0002 & .2614 $\pm$ .0001 & \textbf{.1633 $\pm$ .0001} & \underline{.2572 $\pm$ .0001} & .1653 $\pm$ .0004 & .2583 $\pm$ .0001 \\
& 720 & .2064 $\pm$ .0002 & .2937 $\pm$ .0001 & \textbf{.2001 $\pm$ .0005} & \textbf{.2896 $\pm$ .0005} & \underline{.2013 $\pm$ .0006} & \underline{.2899 $\pm$ .0001} \\ \midrule
\multicolumn{2}{l}{\textbf{Top1 count}} & 0 & 0 & \textbf{18} & \textbf{15} & 1 & 2 \\
\multicolumn{2}{l}{\underline{Top2 count}} & 1 & 4 & 3 & 2 & 6 & 3 \\ \bottomrule
\end{tabular}
\end{adjustbox}
\end{table*}

\begin{table*}[h!]
\centering
\caption{Full-shot forecasting results for PatchTST~\cite{PatchTST} backbone (part 2). SimMTM~\cite{SimMTM} is denoted as MTM, SoftCLT~\cite{SoftCLT} as SCLT, TimeSiam~\cite{TimeSiam} as Siam, and TimesURL~\cite{TimesURL} as URL. Best results are in \textbf{bold} and the second best are \underline{underlined}. (Rand., OliO, and TS2V results are in Table~\ref{tab:full_shot_patchtst_p1}.)}
\label{tab:full_shot_patchtst_p2}
\begin{adjustbox}{width=1.0\linewidth}
\setlength{\tabcolsep}{3pt}
\begin{tabular}{llcccccccc}
\toprule
\multicolumn{2}{c}{\multirow{2}{*}{Dataset}} & \multicolumn{2}{c}{MTM \cite{SimMTM}} & \multicolumn{2}{c}{SCLT \cite{SoftCLT}} & \multicolumn{2}{c}{Siam \cite{TimeSiam}} & \multicolumn{2}{c}{URL \cite{TimesURL}} \\
\cmidrule(lr){3-4} \cmidrule(lr){5-6} \cmidrule(lr){7-8} \cmidrule(lr){9-10}
& & MSE & MAE & MSE & MAE & MSE & MAE & MSE & MAE \\ \midrule
\multirow{4}{*}{\rotatebox{90}{ETTh1}} & 96  & .3985 $\pm$ .0014 & .4173 $\pm$ .0013 & .3754 $\pm$ .0007 & .3984 $\pm$ .0005 & .3871 $\pm$ .0004 & .4099 $\pm$ .0005 & \underline{.3708 $\pm$ .0007} & \underline{.3957 $\pm$ .0005} \\
& 192 & .4299 $\pm$ .0011 & .4355 $\pm$ .0011 & .4095 $\pm$ .0010 & \underline{.4179 $\pm$ .0008} & .4192 $\pm$ .0008 & .4286 $\pm$ .0007 & \textbf{.4062 $\pm$ .0006} & .4187 $\pm$ .0011 \\
& 336 & .4473 $\pm$ .0004 & .4480 $\pm$ .0005 & .4357 $\pm$ .0020 & .4374 $\pm$ .0019 & .4380 $\pm$ .0011 & .4426 $\pm$ .0008 & .4329 $\pm$ .0030 & \underline{.4352 $\pm$ .0029} \\
& 720 & .4560 $\pm$ .0021 & .4744 $\pm$ .016 & .4589 $\pm$ .0001 & .4783 $\pm$ .0008 & .4547 $\pm$ .0005 & .4719 $\pm$ .0003 & .4520 $\pm$ .0015 & .4707 $\pm$ .0013 \\ \midrule
\multirow{4}{*}{\rotatebox{90}{ETTh2}} & 96  & .2783 $\pm$ .0007 & .3403 $\pm$ .0005 & .2802 $\pm$ .0010 & .3434 $\pm$ .0006 & .2808 $\pm$ .0003 & .3410 $\pm$ .0005 & \underline{.2773 $\pm$ .0002} & \underline{.3394 $\pm$ .0004} \\
& 192 & .3431 $\pm$ .0005 & .3827 $\pm$ .0007 & .3441 $\pm$ .0009 & .3844 $\pm$ .0005 & .3440 $\pm$ .0005 & \textbf{.3794 $\pm$ .0002} & .3410 $\pm$ .0004 & \underline{.3812 $\pm$ .0004} \\
& 336 & .3642 $\pm$ .0011 & .4049 $\pm$ .0008 & .3643 $\pm$ .0033 & .4024 $\pm$ .0003 & .3676 $\pm$ .0010 & \textbf{.3989 $\pm$ .0004} & \underline{.3601 $\pm$ .0002} & \underline{.4007 $\pm$ .0005} \\
& 720 & .3949 $\pm$ .0004 & .4312 $\pm$ .0004 & .3960 $\pm$ .0031 & .4342 $\pm$ .0021 & \underline{.3925 $\pm$ .0006} & \textbf{.4268 $\pm$ .0004} & .3926 $\pm$ .0010 & .4298 $\pm$ .0010 \\ \midrule
\multirow{4}{*}{\rotatebox{90}{ETTm1}} & 96  & .2839 $\pm$ .0008 & .3412 $\pm$ .0004 & .2917 $\pm$ .0017 & .3463 $\pm$ .0015 & .2948 $\pm$ .0009 & .3450 $\pm$ .0002 & \underline{.2837 $\pm$ .0012} & \underline{.3381 $\pm$ .0005} \\
& 192 & .3255 $\pm$ .0019 & .3692 $\pm$ .0015 & .3331 $\pm$ .0023 & .3730 $\pm$ .0009 & .3363 $\pm$ .0007 & .3719 $\pm$ .0006 & \underline{.3252 $\pm$ .0003} & \textbf{.3650 $\pm$ .0003} \\
& 336 & \textbf{.3592 $\pm$ .0005} & .3895 $\pm$ .0008 & .3692 $\pm$ .0019 & .3947 $\pm$ .0010 & .3672 $\pm$ .0014 & .3927 $\pm$ .0006 & .3627 $\pm$ .0005 & .3890 $\pm$ .0005 \\
& 720 & \underline{.4120 $\pm$ .0008} & .4219 $\pm$ .0006 & .4154 $\pm$ .0005 & .4238 $\pm$ .0013 & .4165 $\pm$ .0020 & .4242 $\pm$ .0014 & .4126 $\pm$ .0002 & \underline{.4203 $\pm$ .0002} \\ \midrule
\multirow{4}{*}{\rotatebox{90}{ETTm2}} & 96  & .1725 $\pm$ .0006 & .2638 $\pm$ .0008 & .1694 $\pm$ .0012 & .2590 $\pm$ .0013 & .1692 $\pm$ .0007 & .2574 $\pm$ .0004 & \underline{.1658 $\pm$ .0003} & \underline{.2554 $\pm$ .0005} \\
& 192 & .2278 $\pm$ .0008 & .3036 $\pm$ .0004 & .2249 $\pm$ .0026 & .3009 $\pm$ .0031 & .2279 $\pm$ .0003 & .3005 $\pm$ .0008 & \underline{.2210 $\pm$ .0004} & \underline{.2947 $\pm$ .0005} \\
& 336 & .2835 $\pm$ .0022 & .3393 $\pm$ .0019 & .2803 $\pm$ .0019 & .3356 $\pm$ .0022 & .2802 $\pm$ .0006 & .3340 $\pm$ .0007 & \underline{.2752 $\pm$ .0004} & .3308 $\pm$ .0007 \\
& 720 & .3706 $\pm$ .0010 & .3902 $\pm$ .0010 & .3722 $\pm$ .0010 & .3913 $\pm$ .014 & .3672 $\pm$ .0005 & .3866 $\pm$ .0016 & .3701 $\pm$ .0012 & .3894 $\pm$ .0007 \\ \midrule
\multirow{4}{*}{\rotatebox{90}{Weather}} & 96  & .1646 $\pm$ .0011 & .2148 $\pm$ .0011 & .1540 $\pm$ .0001 & .2031 $\pm$ .0001 & \underline{.1511 $\pm$ .0007} & \underline{.2003 $\pm$ .0013} & .1528 $\pm$ .0009 & .2015 $\pm$ .0014 \\
& 192 & .2043 $\pm$ .0014 & .2489 $\pm$ .0010 & .1975 $\pm$ .0005 & .2434 $\pm$ .0009 & \underline{.1949 $\pm$ .0002} & \underline{.2406 $\pm$ .0002} & .1976 $\pm$ .0003 & .2428 $\pm$ .0000 \\
& 336 & .2549 $\pm$ .0003 & .2872 $\pm$ .0004 & \underline{.2466 $\pm$ .0003} & .2817 $\pm$ .0001 & .2467 $\pm$ .0004 & \underline{.2805 $\pm$ .0005} & .2467 $\pm$ .0001 & .2806 $\pm$ .0004 \\
& 720 & .3253 $\pm$ .0005 & .3368 $\pm$ .0004 & \underline{.3201 $\pm$ .0003} & .3339 $\pm$ .0005 & .3207 $\pm$ .0013 & \underline{.3326 $\pm$ .0011} & .3216 $\pm$ .0003 & .3334 $\pm$ .0003 \\ \midrule
\multirow{4}{*}{\rotatebox{90}{ECL}}     & 96  & .1387 $\pm$ .0000 & .2352 $\pm$ .0001 & .1325 $\pm$ .0003 & .2265 $\pm$ .0003 & .1300 $\pm$ .0001 & .2234 $\pm$ .0001 & \textbf{.1287 $\pm$ .0002} & \textbf{.2212 $\pm$ .0001} \\
& 192 & .1541 $\pm$ .0000 & .2486 $\pm$ .0001 & .1493 $\pm$ .0002 & .2419 $\pm$ .0003 & .1481 $\pm$ .0000 & \underline{.2399 $\pm$ .0000} & \textbf{.1471 $\pm$ .0001} & \textbf{.2394 $\pm$ .0001} \\
& 336 & .1707 $\pm$ .0001 & .2657 $\pm$ .0002 & .1656 $\pm$ .0002 & .2589 $\pm$ .0003 & .1672 $\pm$ .0003 & .2585 $\pm$ .0001 & \underline{.1638 $\pm$ .0007} & \underline{.2572 $\pm$ .0009} \\
& 720 & .2026 $\pm$ .0002 & .2914 $\pm$ .0002 & .2017 $\pm$ .0010 & .2907 $\pm$ .0007 & .2042 $\pm$ .0007 & .2924 $\pm$ .0007 & .2123 $\pm$ .0018 & .3010 $\pm$ .0019 \\ \midrule
\multicolumn{2}{l}{\textbf{Top1 count}} & 1 & 0 & 0 & 0 & 0 & 3 & 3 & 3 \\
\multicolumn{2}{l}{\underline{Top2 count}} & 1 & 0 & 2 & 1 & 3 & 5 & \underline{9} & \underline{10} \\ \bottomrule
\end{tabular}
\end{adjustbox}
\end{table*}

\clearpage

\subsection{iTransformer Results}
\begin{table*}[h!]
\centering
\caption{Full-shot forecasting results for iTransformer~\cite{iTransformer} backbone (part 1). TS2Vec~\cite{TS2Vec} is denoted as TS2V. Best results are in \textbf{bold} and the second best are \underline{underlined}. (MTM, SCLT, Siam, and URL results are in Table~\ref{tab:full_shot_itransformer_p2}.)}
\label{tab:full_shot_itransformer_p1}
\begin{adjustbox}{width=1.0\linewidth}
\begin{tabular}{llcccccc}
\toprule
\multicolumn{2}{c}{\multirow{2}{*}{Dataset}} & \multicolumn{2}{c}{Rand.} & \multicolumn{2}{c}{\textbf{OliO (Ours)}} & \multicolumn{2}{c}{TS2V \cite{TS2Vec}} \\
\cmidrule(lr){3-4} \cmidrule(lr){5-6} \cmidrule(lr){7-8}
& & MSE & MAE & MSE & MAE & MSE & MAE \\ \midrule
\multirow{4}{*}{\rotatebox{90}{ETTh1}}  & 96   & \textbf{.3903 $\pm$ .0007} & \textbf{.4055 $\pm$ .0002} & .3969 $\pm$ .0002 & .4101 $\pm$ .0003 & .4027 $\pm$ .0001 & .4155 $\pm$ .0001 \\
 & 192  & .4448 $\pm$ .0004 & .4372 $\pm$ .0003 & \underline{.4442 $\pm$ .0017} & \underline{.4368 $\pm$ .0006} & .4507 $\pm$ .0037 & .4445 $\pm$ .0021 \\
 & 336  & .4884 $\pm$ .0030 & .4612 $\pm$ .0022 & \underline{.4865 $\pm$ .0037} & \underline{.4607 $\pm$ .0027} & .4875 $\pm$ .0034 & .4638 $\pm$ .0020 \\
 & 720  & .5120 $\pm$ .0056 & .4966 $\pm$ .0038 & \underline{.4919 $\pm$ .0162} & \underline{.4814 $\pm$ .0064} & .4978 $\pm$ .0076 & .4848 $\pm$ .0049 \\ \midrule
\multirow{4}{*}{\rotatebox{90}{ETTh2}}  & 96   & \textbf{.2988 $\pm$ .0029} & \textbf{.3493 $\pm$ .0019} & \underline{.3002 $\pm$ .0021} & \underline{.3500 $\pm$ .0008} & .3079 $\pm$ .0044 & .3571 $\pm$ .0026 \\
 & 192  & \textbf{.3785 $\pm$ .0009} & \underline{.3979 $\pm$ .0010} & .3811 $\pm$ .0026 & .4000 $\pm$ .0009 & .3923 $\pm$ .0051 & .4084 $\pm$ .0033 \\
 & 336  & .4250 $\pm$ .0023 & \underline{.4327 $\pm$ .0007} & \underline{.4230 $\pm$ .0037} & .4336 $\pm$ .0020 & .4333 $\pm$ .0022 & .4423 $\pm$ .0006 \\
 & 720  & .4357 $\pm$ .0050 & .4500 $\pm$ .0028 & \textbf{.4271 $\pm$ .0029} & \underline{.4465 $\pm$ .0010} & .4373 $\pm$ .0018 & .4535 $\pm$ .0014 \\ \midrule
\multirow{4}{*}{\rotatebox{90}{ETTm1}}  & 96   & .3456 $\pm$ .0011 & .3782 $\pm$ .0007 & \textbf{.3291 $\pm$ .0008} & \textbf{.3655 $\pm$ .0008} & .3367 $\pm$ .0016 & .3697 $\pm$ .0010 \\
 & 192  & .3831 $\pm$ .0003 & .3946 $\pm$ .0005 & \textbf{.3711 $\pm$ .0015} & \textbf{.3876 $\pm$ .0011} & .3777 $\pm$ .0005 & .3937 $\pm$ .0005 \\
 & 336  & .4205 $\pm$ .0011 & .4183 $\pm$ .0007 & \textbf{.4064 $\pm$ .0023} & \textbf{.4109 $\pm$ .0009} & .4138 $\pm$ .0020 & .4159 $\pm$ .0007 \\
 & 720  & .4896 $\pm$ .0014 & .4573 $\pm$ .0008 & \textbf{.4730 $\pm$ .0027} & \textbf{.4492 $\pm$ .0013} & .4844 $\pm$ .0008 & .4538 $\pm$ .0004 \\ \midrule
\multirow{4}{*}{\rotatebox{90}{ETTm2}}  & 96   & .1857 $\pm$ .0001 & .2709 $\pm$ .0010 & \textbf{.1812 $\pm$ .0022} & \textbf{.2637 $\pm$ .0019} & .1901 $\pm$ .0002 & .2759 $\pm$ .0001 \\
 & 192  & .2528 $\pm$ .0006 & .3128 $\pm$ .0008 & \textbf{.2486 $\pm$ .0005} & \textbf{.3086 $\pm$ .0003} & .2548 $\pm$ .0004 & .3151 $\pm$ .0004 \\
 & 336  & .3146 $\pm$ .0022 & .3516 $\pm$ .0016 & \underline{.3105 $\pm$ .0016} & \textbf{.3486 $\pm$ .0005} & .3181 $\pm$ .0007 & .3555 $\pm$ .0003 \\
 & 720  & .4138 $\pm$ .0001 & .4074 $\pm$ .0008 & \textbf{.4088 $\pm$ .0033} & \textbf{.4044 $\pm$ .0021} & .4184 $\pm$ .0005 & .4105 $\pm$ .0005 \\ \midrule
\multirow{4}{*}{\rotatebox{90}{Weather}}  & 96   & \underline{.1776 $\pm$ .0005} & \underline{.2182 $\pm$ .0006} & \textbf{.1756 $\pm$ .0008} & \textbf{.2161 $\pm$ .0016} & .2139 $\pm$ .0011 & .2630 $\pm$ .0020 \\
 & 192  & .2263 $\pm$ .0010 & \underline{.2590 $\pm$ .0005} & \textbf{.2256 $\pm$ .0012} & \textbf{.2585 $\pm$ .0009} & .2591 $\pm$ .0018 & .2976 $\pm$ .2976 \\
 & 336  & .2829 $\pm$ .0013 & \textbf{.2999 $\pm$ .0008} & \underline{.2825 $\pm$ .0007} & \underline{.3002 $\pm$ .0006} & .3075 $\pm$ .0009 & .3270 $\pm$ .0024 \\
 & 720  & .3590 $\pm$ .0008 & .3505 $\pm$ .0001 & \underline{.3587 $\pm$ .0010} & \textbf{.3496 $\pm$ .0009} & .3774 $\pm$ .0007 & .3712 $\pm$ .0006 \\ \midrule
\multirow{4}{*}{\rotatebox{90}{ECL}}      & 96   & \underline{.1503 $\pm$ .0004} & \underline{.2419 $\pm$ .0004} & \textbf{.1497 $\pm$ .0003} & \textbf{.2416 $\pm$ .0004} & .2611 $\pm$ .0124 & .3462 $\pm$ .0086 \\
 & 192  & \textbf{.1649 $\pm$ .0003} & \textbf{.2560 $\pm$ .0003} & .1671 $\pm$ .0011 & .2582 $\pm$ .0012 & .2646 $\pm$ .0084 & .3503 $\pm$ .0056 \\
 & 336  & .1815 $\pm$ .0010 & .2739 $\pm$ .0007 & \underline{.1800 $\pm$ .0006} & \textbf{.2733 $\pm$ .0007} & .2880 $\pm$ .0024 & .3503 $\pm$ .0018 \\
 & 720  & .2131 $\pm$ .0002 & .3016 $\pm$ .0002 & \textbf{.2107 $\pm$ .0002} & \textbf{.3000 $\pm$ .0003} & .3216 $\pm$ .0110 & .3922 $\pm$ .0079 \\ \midrule
\multicolumn{2}{l}{\textbf{Top1 count}} & 4 & 4 & \textbf{12} & \textbf{14} & 0 & 0 \\
\multicolumn{2}{l}{\underline{Top2 count}} & 2 & 5 & \underline{9} & 6 & 0 & 0 \\ \bottomrule
\end{tabular}
\end{adjustbox}
\end{table*}

\begin{table*}[h!]
\centering
\caption{Full-shot forecasting results for iTransformer~\cite{iTransformer} backbone (part 2). SimMTM~\cite{SimMTM} is denoted as MTM, SoftCLT~\cite{SoftCLT} as SCLT, TimeSiam~\cite{TimeSiam} as Siam, and TimesURL~\cite{TimesURL} as URL. Best results are in \textbf{bold} and the second best are \underline{underlined}. (Rand., OliO, and TS2V results are in Table~\ref{tab:full_shot_itransformer_p1}.)}
\label{tab:full_shot_itransformer_p2}
\begin{adjustbox}{width=1.0\linewidth}
\setlength{\tabcolsep}{3pt}
\begin{tabular}{llcccccccc}
\toprule
\multicolumn{2}{c}{\multirow{2}{*}{Dataset}} & \multicolumn{2}{c}{MTM \cite{SimMTM}} & \multicolumn{2}{c}{SCLT \cite{SoftCLT}} & \multicolumn{2}{c}{Siam \cite{TimeSiam}} & \multicolumn{2}{c}{URL \cite{TimesURL}} \\
\cmidrule(lr){3-4} \cmidrule(lr){5-6} \cmidrule(lr){7-8} \cmidrule(lr){9-10}
& & MSE & MAE & MSE & MAE & MSE & MAE & MSE & MAE \\ \midrule
\multirow{4}{*}{\rotatebox{90}{ETTh1}}  & 96   & .3993 $\pm$ .0006 & .4118 $\pm$ .0004 & .4040 $\pm$ .0006 & .4152 $\pm$ .0002 & \underline{.3914 $\pm$ .0011} & \underline{.4057 $\pm$ .0009} & .3924 $\pm$ .0009 & .4076 $\pm$ .0006 \\
 & 192  & .4490 $\pm$ .0009 & .4417 $\pm$ .0014 & .4540 $\pm$ .0080 & .4444 $\pm$ .0058 & \textbf{.4408 $\pm$ .0013} & \textbf{.4342 $\pm$ .0005} & .4448 $\pm$ .0013 & .4388 $\pm$ .0008 \\
 & 336  & .4927 $\pm$ .0030 & .4634 $\pm$ .0026 & .4932 $\pm$ .0027 & .4661 $\pm$ .0004 & \textbf{.4824 $\pm$ .0028} & \textbf{.4559 $\pm$ .0021} & .4868 $\pm$ .0007 & .4612 $\pm$ .0004 \\
 & 720  & .5043 $\pm$ .0043 & .4846 $\pm$ .0019 & .5203 $\pm$ .0054 & .4970 $\pm$ .0028 & \textbf{.4812 $\pm$ .0064} & \textbf{.4738 $\pm$ .0030} & .5016 $\pm$ .0112 & .4873 $\pm$ .0051 \\ \midrule
\multirow{4}{*}{\rotatebox{90}{ETTh2}}  & 96   & .3181 $\pm$ .0021 & .3587 $\pm$ .0013 & .3016 $\pm$ .0044 & .3520 $\pm$ .0028 & .3050 $\pm$ .0017 & .3515 $\pm$ .0011 & .3037 $\pm$ .0002 & .3521 $\pm$ .0002 \\
 & 192  & .4040 $\pm$ .0047 & .4095 $\pm$ .0019 & .3871 $\pm$ .0043 & .4058 $\pm$ .0040 & \underline{.3794 $\pm$ .0016} & \textbf{.3971 $\pm$ .0007} & .3891 $\pm$ .0047 & .4057 $\pm$ .0037 \\
 & 336  & .4496 $\pm$ .0061 & .4424 $\pm$ .0022 & .4271 $\pm$ .0026 & .4333 $\pm$ .0023 & \textbf{.4180 $\pm$ .0024} & \textbf{.4296 $\pm$ .0013} & .4313 $\pm$ .0052 & .4339 $\pm$ .0020 \\
 & 720  & .4386 $\pm$ .0010 & .4479 $\pm$ .0004 & .4412 $\pm$ .0071 & .4541 $\pm$ .0045 & \underline{.4289 $\pm$ .0029} & \textbf{.4464 $\pm$ .0021} & .4398 $\pm$ .0034 & .4519 $\pm$ .0022 \\ \midrule
\multirow{4}{*}{\rotatebox{90}{ETTm1}}  & 96   & .3341 $\pm$ .0002 & .3732 $\pm$ .0002 & .3368 $\pm$ .0008 & .3712 $\pm$ .0001 & .3435 $\pm$ .0015 & .3764 $\pm$ .0007 & \underline{.3315 $\pm$ .0022} & \underline{.3660 $\pm$ .0013} \\
 & 192  & .3746 $\pm$ .0015 & .3945 $\pm$ .0010 & .3764 $\pm$ .0009 & .3919 $\pm$ .0006 & .3837 $\pm$ .0002 & .3958 $\pm$ .0003 & \underline{.3732 $\pm$ .0003} & \underline{.3901 $\pm$ .0005} \\
 & 336  & \underline{.4083 $\pm$ .0007} & .4157 $\pm$ .0003 & .4097 $\pm$ .0007 & \underline{.4140 $\pm$ .0001} & .4175 $\pm$ .0003 & .4173 $\pm$ .0001 & .4124 $\pm$ .0018 & .4146 $\pm$ .0011 \\
 & 720  & \underline{.4784 $\pm$ .0005} & .4534 $\pm$ .0002 & .4805 $\pm$ .0007 & \underline{.4521 $\pm$ .0002} & .4839 $\pm$ .0005 & .4543 $\pm$ .0004 & .4812 $\pm$ .0014 & .4527 $\pm$ .0006 \\ \midrule
\multirow{4}{*}{\rotatebox{90}{ETTm2}}  & 96   & .1836 $\pm$ .0003 & .2696 $\pm$ .0003 & .1875 $\pm$ .0017 & .2705 $\pm$ .0016 & \underline{.1826 $\pm$ .0005} & \underline{.2657 $\pm$ .0005} & .1838 $\pm$ .0004 & .2689 $\pm$ .0009 \\
 & 192  & \underline{.2488 $\pm$ .0008} & \underline{.3103 $\pm$ .0006} & .2539 $\pm$ .0021 & .3125 $\pm$ .0010 & .2536 $\pm$ .0043 & .3113 $\pm$ .0031 & .2517 $\pm$ .0011 & .3123 $\pm$ .014 \\
 & 336  & \textbf{.3095 $\pm$ .0010} & \underline{.3498 $\pm$ .0005} & .3173 $\pm$ .0008 & .3516 $\pm$ .0006 & .3203 $\pm$ .0022 & .3530 $\pm$ .0012 & .3134 $\pm$ .0011 & .3501 $\pm$ .0010 \\
 & 720  & .4131 $\pm$ .0032 & .4091 $\pm$ .0018 & .4151 $\pm$ .0009 & .4078 $\pm$ .0007 & .4191 $\pm$ .0018 & .4100 $\pm$ .0143 & \underline{.4129 $\pm$ .0020} & \underline{.4062 $\pm$ .0012} \\ \midrule
\multirow{4}{*}{\rotatebox{90}{Weather}}  & 96   & .2041 $\pm$ .0017 & .2478 $\pm$ .0007 & .2156 $\pm$ .0056 & .2648 $\pm$ .0050 & .1778 $\pm$ .0008 & .2190 $\pm$ .0007 & .1834 $\pm$ .0003 & .2236 $\pm$ .0004 \\
 & 192  & .2499 $\pm$ .0006 & .2821 $\pm$ .0003 & .2570 $\pm$ .0016 & .2913 $\pm$ .0027 & \underline{.2262 $\pm$ .0012} & .2604 $\pm$ .0008 & .2314 $\pm$ .0004 & .2629 $\pm$ .0005 \\
 & 336  & .3008 $\pm$ .0006 & .3163 $\pm$ .0005 & .3085 $\pm$ .0030 & .3255 $\pm$ .0032 & \textbf{.2819 $\pm$ .0010} & .3014 $\pm$ .0008 & .2865 $\pm$ .0003 & .3023 $\pm$ .0002 \\
 & 720  & .3706 $\pm$ .0009 & .3609 $\pm$ .0008 & .3798 $\pm$ .0007 & .3706 $\pm$ .0005 & \textbf{.3574 $\pm$ .0010} & \underline{.3502 $\pm$ .0004} & .3623 $\pm$ .0005 & .3518 $\pm$ .0006 \\ \midrule
\multirow{4}{*}{\rotatebox{90}{ECL}}      & 96   & .3656 $\pm$ .0010 & .4260 $\pm$ .0007 & .1646 $\pm$ .0004 & .2567 $\pm$ .0004 & .2849 $\pm$ .0011 & .3656 $\pm$ .0009 & .1508 $\pm$ .0003 & .2435 $\pm$ .0004 \\
 & 192  & .3655 $\pm$ .0021 & .4279 $\pm$ .013 & .1789 $\pm$ .0005 & .2695 $\pm$ .0006 & .2908 $\pm$ .0023 & .3718 $\pm$ .0111 & \underline{.1659 $\pm$ .0004} & \underline{.2572 $\pm$ .0004} \\
 & 336  & .3800 $\pm$ .0025 & .4388 $\pm$ .0020 & .1982 $\pm$ .0007 & .2894 $\pm$ .0007 & .3050 $\pm$ .0008 & .3836 $\pm$ .0004 & \textbf{.1797 $\pm$ .0001} & .2739 $\pm$ .0001 \\
 & 720  & .4121 $\pm$ .0006 & .4584 $\pm$ .0005 & .2394 $\pm$ .0005 & .3244 $\pm$ .0003 & .3138 $\pm$ .0007 & .3855 $\pm$ .0006 & \underline{.2114 $\pm$ .0002} & \underline{.3015 $\pm$ .0001} \\ \midrule
\multicolumn{2}{l}{\textbf{Top1 count}} & 1 & 0 & 0 & 0 & 6 & 6 & 1 & 0 \\
\multicolumn{2}{l}{\underline{Top2 count}} & 3 & 2 & 0 & 2 & 5 & 3 & 5 & 5 \\ \bottomrule
\end{tabular}
\end{adjustbox}
\end{table*}

\clearpage

\subsection{TimeVLM Results}
\begin{table*}[h!]
\centering
\caption{Full-shot forecasting results for TimeVLM~\cite{TimeVLM} backbone (part 1). TS2Vec~\cite{TS2Vec} is denoted as TS2V. Best results are in \textbf{bold} and the second best are \underline{underlined}. (MTM, SCLT, Siam, and URL results are in Table~\ref{tab:full_shot_timevlm_p2}.)}
\label{tab:full_shot_timevlm_p1}
\begin{adjustbox}{width=1.0\linewidth}
\begin{tabular}{llcccccc}
\toprule
\multicolumn{2}{c}{\multirow{2}{*}{Dataset}} & \multicolumn{2}{c}{Rand.} & \multicolumn{2}{c}{\textbf{OliO (Ours)}} & \multicolumn{2}{c}{TS2V \cite{TS2Vec}} \\
\cmidrule(lr){3-4} \cmidrule(lr){5-6} \cmidrule(lr){7-8}
& & MSE & MAE & MSE & MAE & MSE & MAE \\ \midrule
\multirow{4}{*}{\rotatebox{90}{ETTh1}} & 96  & .3704 $\pm$ .0033 & .4028 $\pm$ .0036 & .3669 $\pm$ .0024 & \textbf{.3979 $\pm$ .0000} & .3666 $\pm$ .0012 & \underline{.3983 $\pm$ .0014} \\
& 192 & .4074 $\pm$ .0006 & .4301 $\pm$ .0015 & .4029 $\pm$ .0046 & .4251 $\pm$ .0000 & .4040 $\pm$ .0023 & .4278 $\pm$ .0008 \\
& 336 & .4659 $\pm$ .0034 & .4661 $\pm$ .0032 & \underline{.4232 $\pm$ .0003} & \textbf{.4399 $\pm$ .0000} & .4287 $\pm$ .0038 & .4442 $\pm$ .0035 \\
& 720 & .4981 $\pm$ .0066 & .5011 $\pm$ .0037 & .4587 $\pm$ .0069 & \underline{.4752 $\pm$ .0000} & .4634 $\pm$ .0095 & .4824 $\pm$ .0063 \\ \midrule
\multirow{4}{*}{\rotatebox{90}{ETTh2}} & 96  & .2818 $\pm$ .0004 & .3456 $\pm$ .0009 & \underline{.2755 $\pm$ .0000} & .3407 $\pm$ .0000 & .2766 $\pm$ .0004 & .3407 $\pm$ .0009 \\
& 192 & .3497 $\pm$ .0010 & .3828 $\pm$ .0027 & \textbf{.3359 $\pm$ .0002} & \textbf{.3811 $\pm$ .0000} & .3369 $\pm$ .0013 & .3814 $\pm$ .0009 \\
& 336 & .3767 $\pm$ .0079 & .4093 $\pm$ .0036 & \textbf{.3613 $\pm$ .0005} & \textbf{.4040 $\pm$ .0000} & \underline{.3621 $\pm$ .0016} & \underline{.4051 $\pm$ .0010} \\
& 720 & .4117 $\pm$ .0036 & .4463 $\pm$ .0047 & \underline{.4009 $\pm$ .0042} & \underline{.4407 $\pm$ .0000} & .4069 $\pm$ .0048 & .4437 $\pm$ .0028 \\ \midrule
\multirow{4}{*}{\rotatebox{90}{ETTm1}} & 96  & .2908 $\pm$ .0009 & .3483 $\pm$ .0004 & \underline{.2895 $\pm$ .0009} & \textbf{.3468 $\pm$ .0000} & .2896 $\pm$ .0003 & .3477 $\pm$ .0006 \\
& 192 & .3354 $\pm$ .0015 & .3759 $\pm$ .0004 & \textbf{.3271 $\pm$ .0011} & \textbf{.3713 $\pm$ .0000} & .3310 $\pm$ .0022 & .3751 $\pm$ .0011 \\
& 336 & .3676 $\pm$ .0081 & .3966 $\pm$ .0027 & \textbf{.3593 $\pm$ .0008} & .3932 $\pm$ .0000 & .3617 $\pm$ .0027 & .3938 $\pm$ .0016 \\
& 720 & .4209 $\pm$ .0012 & .4297 $\pm$ .0022 & \underline{.4096 $\pm$ .0029} & .4256 $\pm$ .0000 & .4102 $\pm$ .0030 & .4256 $\pm$ .0021 \\ \midrule
\multirow{4}{*}{\rotatebox{90}{ETTm2}} & 96  & .1690 $\pm$ .0006 & .2603 $\pm$ .0005 & .1674 $\pm$ .0004 & .2584 $\pm$ .0000 & \underline{.1673 $\pm$ .0011} & \underline{.2569 $\pm$ .0011} \\
& 192 & .2251 $\pm$ .0020 & .2992 $\pm$ .0028 & \underline{.2216 $\pm$ .0023} & .2964 $\pm$ .0000 & .2233 $\pm$ .0025 & .2963 $\pm$ .0029 \\
& 336 & .2745 $\pm$ .0027 & .3313 $\pm$ .0009 & .2733 $\pm$ .0014 & .3293 $\pm$ .0000 & \underline{.2702 $\pm$ .0017} & .3292 $\pm$ .0013 \\
& 720 & .3676 $\pm$ .0018 & .3887 $\pm$ .0017 & \textbf{.3521 $\pm$ .0014} & \underline{.3824 $\pm$ .0000} & .3540 $\pm$ .0025 & .3829 $\pm$ .0004 \\ \midrule
\multirow{4}{*}{\rotatebox{90}{Weather}} & 96  & .1495 $\pm$ .0006 & .2004 $\pm$ .0004 & .1490 $\pm$ .0010 & \textbf{.1988 $\pm$ .0000} & .1490 $\pm$ .0006 & .1991 $\pm$ .0016 \\
& 192 & \underline{.1937 $\pm$ .0005} & .2428 $\pm$ .0008 & \underline{.1937 $\pm$ .0003} & .2431 $\pm$ .0000 & .1938 $\pm$ .0100 & \textbf{.2409 $\pm$ .0005} \\
& 336 & .2445 $\pm$ .0005 & .2828 $\pm$ .0002 & \textbf{.2433 $\pm$ .0002} & \textbf{.2793 $\pm$ .0000} & \underline{.2444 $\pm$ .0008} & \underline{.2804 $\pm$ .0003} \\
& 720 & .3180 $\pm$ .0005 & .3342 $\pm$ .0004 & \underline{.3148 $\pm$ .0003} & .3313 $\pm$ .0000 & \textbf{.3147 $\pm$ .0003} & \textbf{.3310 $\pm$ .0007} \\ \midrule
\multirow{4}{*}{\rotatebox{90}{ECL}}     & 96  & .1353 $\pm$ .0002 & .2346 $\pm$ .0004 & .1319 $\pm$ .0002 & \underline{.2264 $\pm$ .0000} & \textbf{.1315 $\pm$ .0000} & .2266 $\pm$ .0001 \\
& 192 & .1513 $\pm$ .0001 & .2487 $\pm$ .0007 & .1483 $\pm$ .0001 & \textbf{.2412 $\pm$ .0000} & \textbf{.1478 $\pm$ .0001} & \underline{.2418 $\pm$ .0001} \\
& 336 & .1673 $\pm$ .0007 & .2648 $\pm$ .0008 & \textbf{.1643 $\pm$ .0003} & \textbf{.2578 $\pm$ .0000} & .1670 $\pm$ .0005 & .2629 $\pm$ .0006 \\
& 720 & .2080 $\pm$ .0013 & .3000 $\pm$ .0026 & \textbf{.2044 $\pm$ .0005} & \textbf{.2924 $\pm$ .0000} & .2091 $\pm$ .0023 & .3005 $\pm$ .0270 \\ \midrule
\multicolumn{2}{l}{\textbf{Top1 count}} & 0 & 0 & \textbf{8} & \textbf{11} & 3 & 2 \\
\multicolumn{2}{l}{\underline{Top2 count}} & 1 & 0 & \underline{8} & 4 & 4 & 5 \\ \bottomrule
\end{tabular}
\end{adjustbox}
\end{table*}

\begin{table*}[h!]
\centering
\caption{Full-shot forecasting results for TimeVLM~\cite{TimeVLM} backbone (part 2). SimMTM~\cite{SimMTM} is denoted as MTM, SoftCLT~\cite{SoftCLT} as SCLT, TimeSiam~\cite{TimeSiam} as Siam, and TimesURL~\cite{TimesURL} as URL. Best results are in \textbf{bold} and the second best are \underline{underlined}. (Rand., OliO, and TS2V results are in Table~\ref{tab:full_shot_timevlm_p1}.)}
\label{tab:full_shot_timevlm_p2}
\begin{adjustbox}{width=1.0\linewidth}
\setlength{\tabcolsep}{3pt}
\begin{tabular}{llcccccccc}
\toprule
\multicolumn{2}{c}{\multirow{2}{*}{Dataset}} & \multicolumn{2}{c}{MTM \cite{SimMTM}} & \multicolumn{2}{c}{SCLT \cite{SoftCLT}} & \multicolumn{2}{c}{Siam \cite{TimeSiam}} & \multicolumn{2}{c}{URL \cite{TimesURL}} \\
\cmidrule(lr){3-4} \cmidrule(lr){5-6} \cmidrule(lr){7-8} \cmidrule(lr){9-10}
& & MSE & MAE & MSE & MAE & MSE & MAE & MSE & MAE \\ \midrule
\multirow{4}{*}{\rotatebox{90}{ETTh1}} & 96  & .3666 $\pm$ .0016 & .3984 $\pm$ .0017 & \textbf{.3660 $\pm$ .0005} & .3988 $\pm$ .0004 & .3779 $\pm$ .0001 & .4085 $\pm$ .0000 & .3672 $\pm$ .0008 & .4000 $\pm$ .0008 \\
& 192 & .4038 $\pm$ .0017 & .4278 $\pm$ .0011 & \textbf{.4017 $\pm$ .0015} & \textbf{.4223 $\pm$ .0013} & .4074 $\pm$ .0007 & .4263 $\pm$ .0010 & \underline{.4027 $\pm$ .0028} & \underline{.4234 $\pm$ .0026} \\
& 336 & .4341 $\pm$ .0073 & .4533 $\pm$ .0062 & \underline{.4232 $\pm$ .0042} & \underline{.4401 $\pm$ .0042} & .4285 $\pm$ .0011 & .4428 $\pm$ .0011 & .4250 $\pm$ .0067 & .4414 $\pm$ .0064 \\
& 720 & .4930 $\pm$ .0009 & .4964 $\pm$ .0006 & \textbf{.4556 $\pm$ .0049} & \textbf{.4736 $\pm$ .0041} & \underline{.4577 $\pm$ .0114} & .4769 $\pm$ .0070 & .4596 $\pm$ .0119 & .4794 $\pm$ .0092 \\ \midrule
\multirow{4}{*}{\rotatebox{90}{ETTh2}} & 96  & .2862 $\pm$ .0061 & .3470 $\pm$ .0036 & .2762 $\pm$ .0010 & .3428 $\pm$ .0016 & .2777 $\pm$ .0003 & .3448 $\pm$ .0004 & \textbf{.2754 $\pm$ .0006} & \textbf{.3401 $\pm$ .0007} \\
& 192 & .3409 $\pm$ .0018 & .3868 $\pm$ .0014 & .3377 $\pm$ .0004 & .3835 $\pm$ .0021 & .3409 $\pm$ .0018 & .3865 $\pm$ .0017 & \underline{.3368 $\pm$ .0011} & \underline{.3822 $\pm$ .0021} \\
& 336 & .3664 $\pm$ .0004 & .4076 $\pm$ .0009 & .3624 $\pm$ .0007 & .4059 $\pm$ .0003 & .3713 $\pm$ .0142 & .4093 $\pm$ .0082 & .3625 $\pm$ .0019 & .4059 $\pm$ .0028 \\
& 720 & .4073 $\pm$ .0038 & .4438 $\pm$ .0022 & .4119 $\pm$ .0032 & .4493 $\pm$ .0032 & .4218 $\pm$ .0024 & .4599 $\pm$ .0023 & \textbf{.3977 $\pm$ .0006} & \textbf{.4372 $\pm$ .0007} \\ \midrule
\multirow{4}{*}{\rotatebox{90}{ETTm1}} & 96  & .2899 $\pm$ .0005 & .3478 $\pm$ .0003 & .2903 $\pm$ .0008 & \underline{.3472 $\pm$ .0007} & .2940 $\pm$ .0016 & .3515 $\pm$ .0008 & \underline{.2895 $\pm$ .0004} & .3480 $\pm$ .0005 \\
& 192 & .3348 $\pm$ .0036 & .3770 $\pm$ .0022 & \underline{.3308 $\pm$ .0009} & \underline{.3735 $\pm$ .0003} & .3331 $\pm$ .0006 & .3742 $\pm$ .0004 & .3309 $\pm$ .0020 & .3736 $\pm$ .0010 \\
& 336 & .3663 $\pm$ .0027 & .3967 $\pm$ .0012 & .3615 $\pm$ .0018 & .3930 $\pm$ .0004 & .3619 $\pm$ .0004 & \underline{.3919 $\pm$ .0010} & \underline{.3605 $\pm$ .0003} & \textbf{.3908 $\pm$ .0004} \\
& 720 & .4104 $\pm$ .0017 & .4262 $\pm$ .0015 & .4102 $\pm$ .0025 & \underline{.4244 $\pm$ .0008} & .4165 $\pm$ .0012 & \textbf{.4236 $\pm$ .0010} & \textbf{.4081 $\pm$ .0012} & \underline{.4244 $\pm$ .0008} \\ \midrule
\multirow{4}{*}{\rotatebox{90}{ETTm2}} & 96  & .1691 $\pm$ .0004 & .2603 $\pm$ .0007 & .1688 $\pm$ .0018 & .2590 $\pm$ .0011 & .1678 $\pm$ .0018 & .2586 $\pm$ .0018 & \underline{.1673 $\pm$ .0011} & \underline{.2569 $\pm$ .0011} \\
& 192 & .2238 $\pm$ .0008 & .2978 $\pm$ .0011 & \textbf{.2211 $\pm$ .0004} & \underline{.2951 $\pm$ .0007} & .2227 $\pm$ .012 & \textbf{.2950 $\pm$ .0014} & .2229 $\pm$ .0025 & .2963 $\pm$ .0029 \\
& 336 & .2728 $\pm$ .0040 & .3311 $\pm$ .0031 & \textbf{.2699 $\pm$ .0014} & \underline{.3279 $\pm$ .0015} & .2716 $\pm$ .0015 & \textbf{.3269 $\pm$ .0016} & .2712 $\pm$ .0035 & .3293 $\pm$ .0028 \\
& 720 & .3581 $\pm$ .0082 & .3859 $\pm$ .0032 & .3544 $\pm$ .0011 & .3825 $\pm$ .0008 & .3590 $\pm$ .0034 & \textbf{.3823 $\pm$ .0025} & .3540 $\pm$ .0025 & .3829 $\pm$ .0004 \\ \midrule
\multirow{4}{*}{\rotatebox{90}{Weather}} & 96  & .1492 $\pm$ .0006 & \underline{.1990 $\pm$ .0011} & \textbf{.1488 $\pm$ .0005} & .1991 $\pm$ .0006 & .1520 $\pm$ .0002 & .2021 $\pm$ .0002 & .1491 $\pm$ .0009 & .1998 $\pm$ .0015 \\
& 192 & .1946 $\pm$ .0003 & .2431 $\pm$ .0006 & .1942 $\pm$ .0009 & \underline{.2411 $\pm$ .0009} & .1956 $\pm$ .0001 & .2422 $\pm$ .0002 & .1938 $\pm$ .0003 & .2412 $\pm$ .0009 \\
& 336 & .2451 $\pm$ .0003 & .2826 $\pm$ .0005 & .2453 $\pm$ .010 & .2808 $\pm$ .0007 & .2453 $\pm$ .0004 & .2809 $\pm$ .0003 & .2445 $\pm$ .0003 & .2822 $\pm$ .0007 \\
& 720 & .3163 $\pm$ .011 & .3335 $\pm$ .0005 & .3170 $\pm$ .0008 & .3334 $\pm$ .0008 & .3151 $\pm$ .0008 & .3321 $\pm$ .0005 & .3149 $\pm$ .0001 & \underline{.3312 $\pm$ .0004} \\ \midrule
\multirow{4}{*}{\rotatebox{90}{ECL}}     & 96  & \underline{.1318 $\pm$ .0003} & \underline{.2264 $\pm$ .0003} & .1325 $\pm$ .0001 & .2276 $\pm$ .0001 & .1346 $\pm$ .0002 & .2321 $\pm$ .0002 & .1363 $\pm$ .0004 & .2330 $\pm$ .0007 \\
& 192 & .1487 $\pm$ .0008 & .2426 $\pm$ .0012 & \underline{.1481 $\pm$ .0001} & .2424 $\pm$ .0002 & .1533 $\pm$ .0002 & .2498 $\pm$ .0004 & .1534 $\pm$ .0005 & .2493 $\pm$ .0120 \\
& 336 & .1665 $\pm$ .0010 & .2612 $\pm$ .0130 & \underline{.1651 $\pm$ .0001} & \underline{.2605 $\pm$ .0002} & .1692 $\pm$ .0110 & .2647 $\pm$ .0011 & .1727 $\pm$ .0003 & .2720 $\pm$ .0004 \\
& 720 & .2080 $\pm$ .0026 & .2989 $\pm$ .0029 & .2077 $\pm$ .0001 & .2971 $\pm$ .0001 & .2080 $\pm$ .0006 & .2978 $\pm$ .0007 & \underline{.2061 $\pm$ .0001} & \underline{.2941 $\pm$ .0001} \\ \midrule
\multicolumn{2}{l}{\textbf{Top1 count}} & 0 & 0 & 6 & 2 & 0 & 4 & 3 & 3 \\
\multicolumn{2}{l}{\underline{Top2 count}} & 1 & 2 & 4 & \underline{8} & 1 & 1 & 6 & 3 \\ \bottomrule
\end{tabular}
\end{adjustbox}
\end{table*}

\clearpage

\section{Detailed comparative analysis of performance improvements}
\label{app:comparative_analysis}

Figs.~\ref{fig:heatmap_patchtst}, \ref{fig:heatmap_itransformer}, and~\ref{fig:heatmap_timevlm} show backbone-specific forecasting comparisons. Each cell averages the performance index $\mathrm{MSE}_{\mathrm{Rand}}/\mathrm{MSE}_{\mathrm{Method}}$ across the six benchmark datasets at a given prediction horizon. Values above 1.0 correspond to lower MSE than random initialization.

Figs.~\ref{fig:heatmap_etth1}, \ref{fig:heatmap_etth2}, \ref{fig:heatmap_ettm1}, \ref{fig:heatmap_ettm2}, \ref{fig:heatmap_weather}, and~\ref{fig:heatmap_ecl} instead average the index across the three backbones for each dataset. Together, the two views show how the relative benefit of pre-training varies across architectures, datasets, and horizons. App.~\ref{app:detailed_results} provides per-setting MSE and sample standard deviations.

\begin{figure*}[h]
    \centering
    \begin{minipage}[t]{0.48\textwidth}
        \centering
        \includegraphics[width=\linewidth]{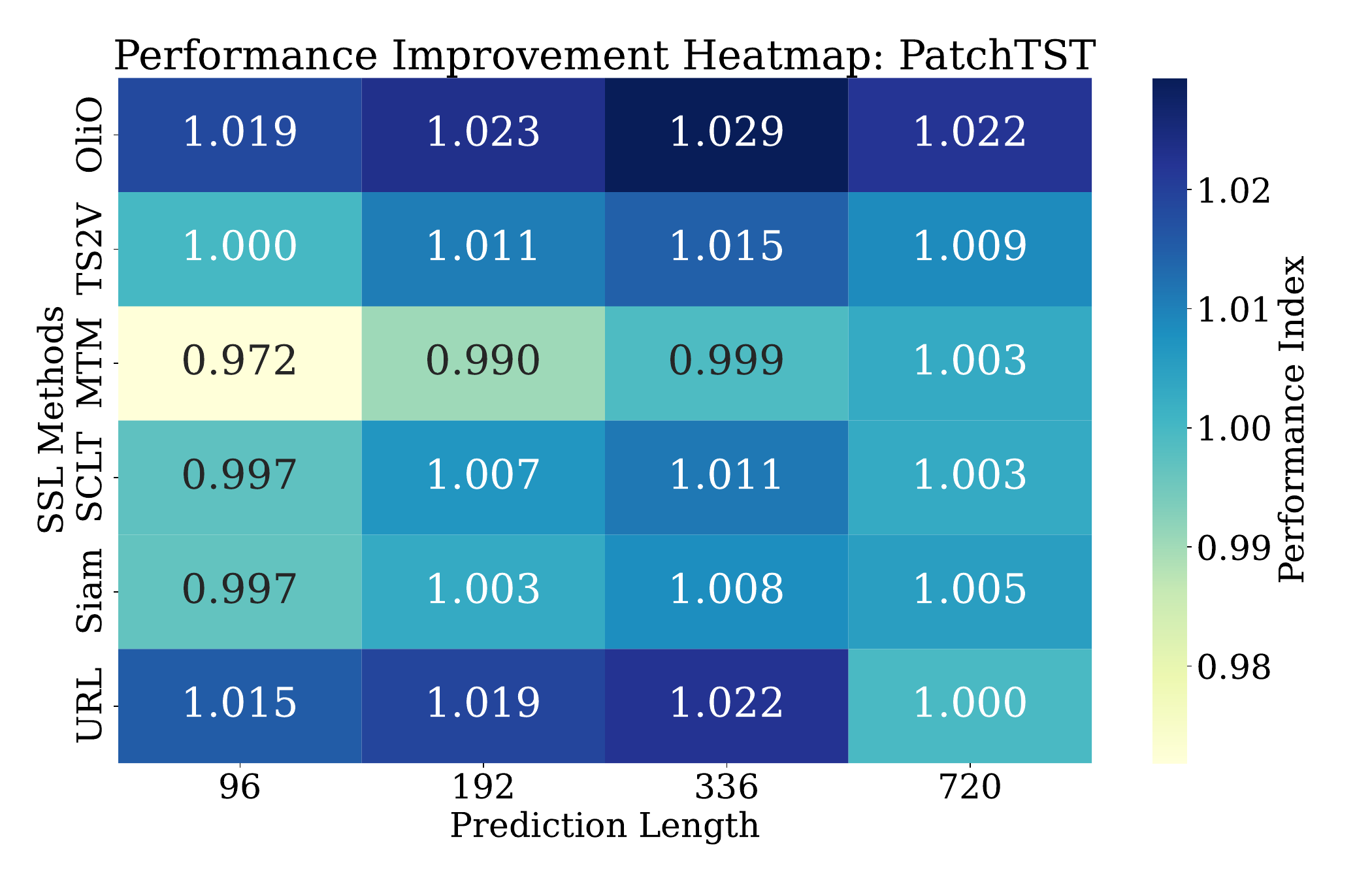}
        \caption{PatchTST backbone.}
        \label{fig:heatmap_patchtst}
    \end{minipage}
    \hfill
    \begin{minipage}[t]{0.48\textwidth}
        \centering
        \includegraphics[width=\linewidth]{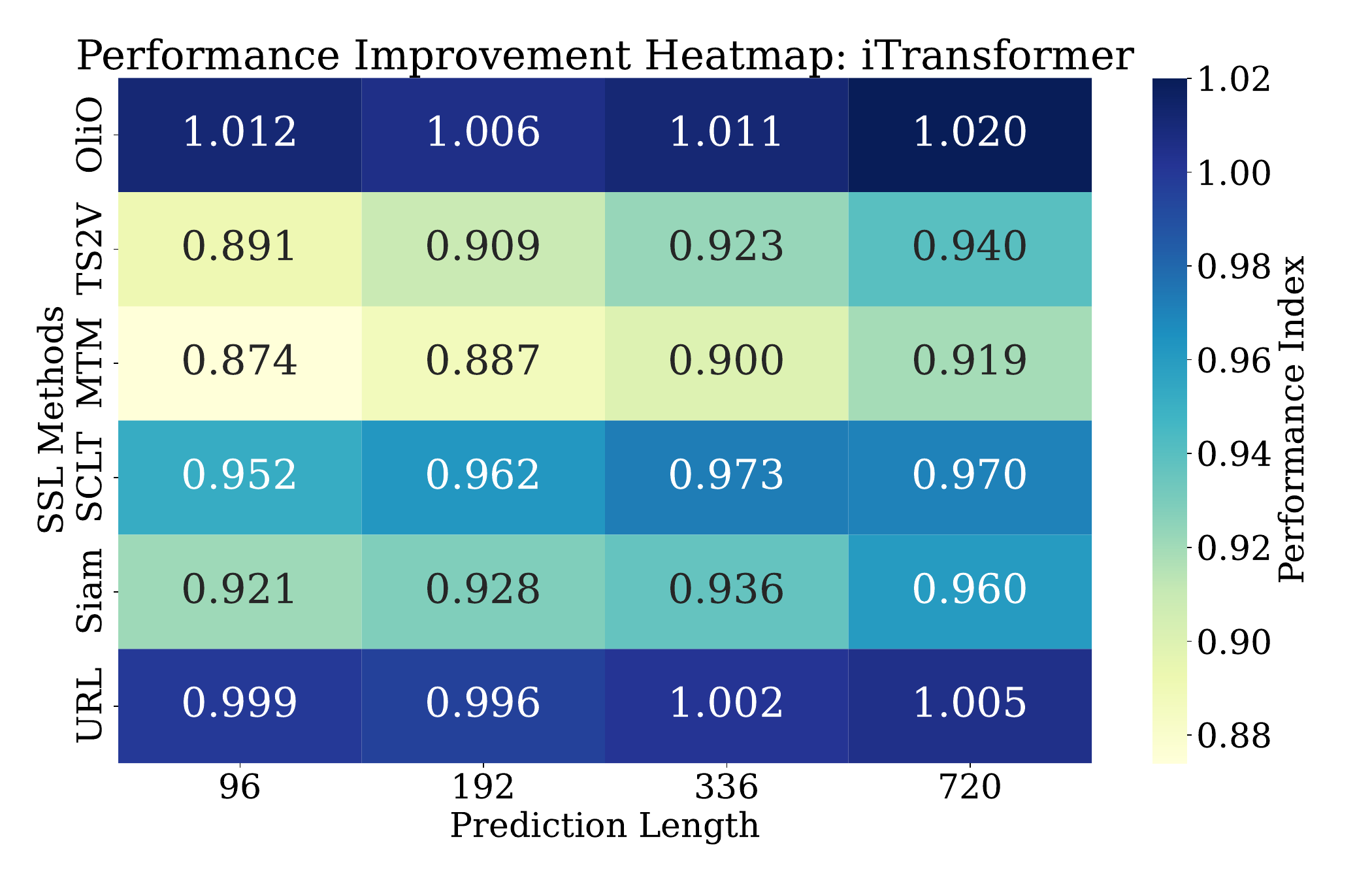}
        \caption{iTransformer backbone.}
        \label{fig:heatmap_itransformer}
    \end{minipage}

    \vspace{2em}

    \begin{minipage}[t]{0.48\textwidth}
        \centering
        \includegraphics[width=\linewidth]{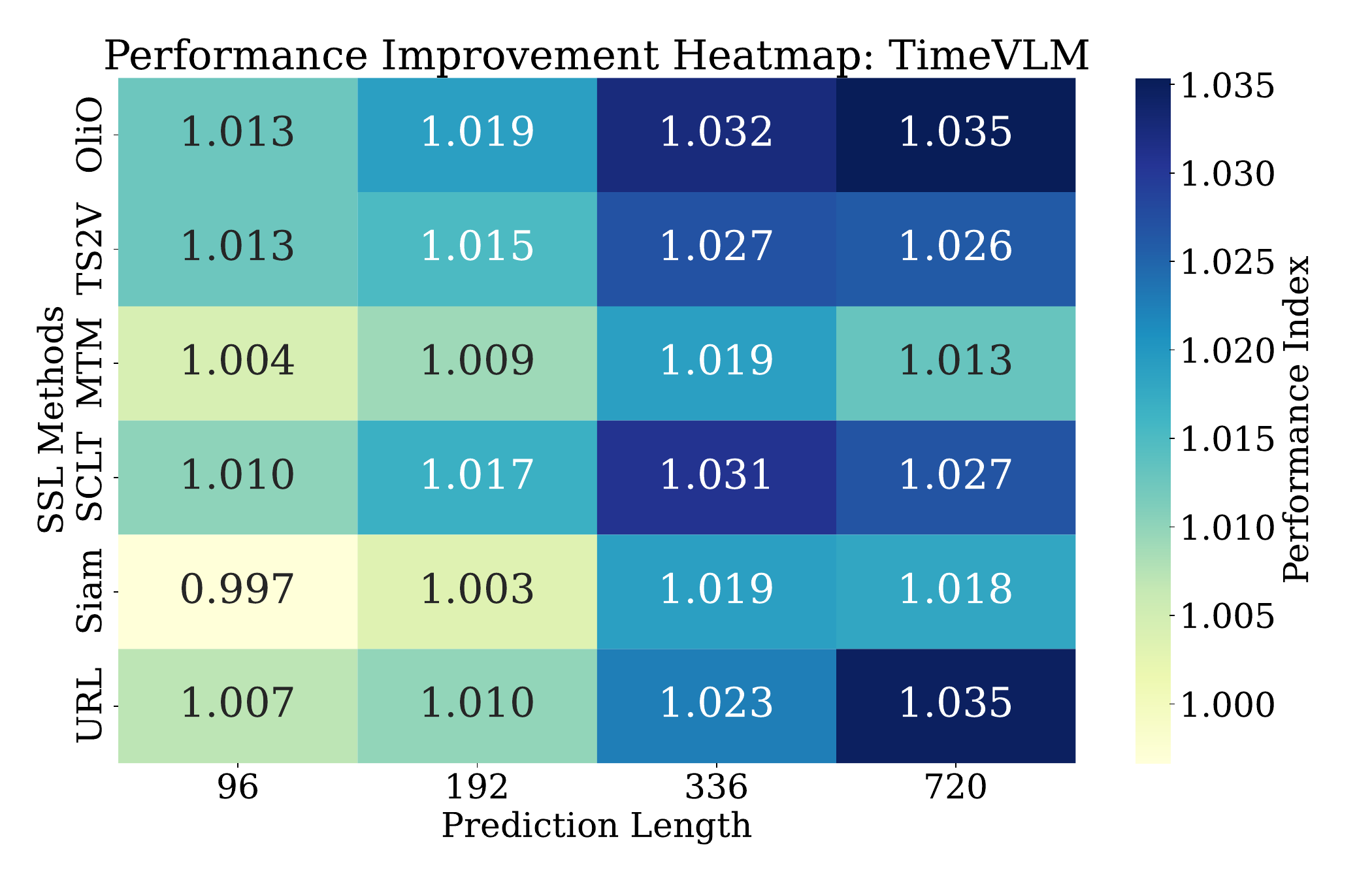}
        \caption{TimeVLM backbone.}
        \label{fig:heatmap_timevlm}
    \end{minipage}
    \caption*{Appendix figure: detailed performance improvement heatmaps for each backbone architecture. The values represent the mean performance index (Random MSE / Method MSE) aggregated across six benchmark datasets. Dark blue cells indicate higher performance indices.}
\end{figure*}

\begin{figure*}[h]
    \centering
    \begin{minipage}[t]{0.48\textwidth}
        \centering
        \includegraphics[width=\linewidth]{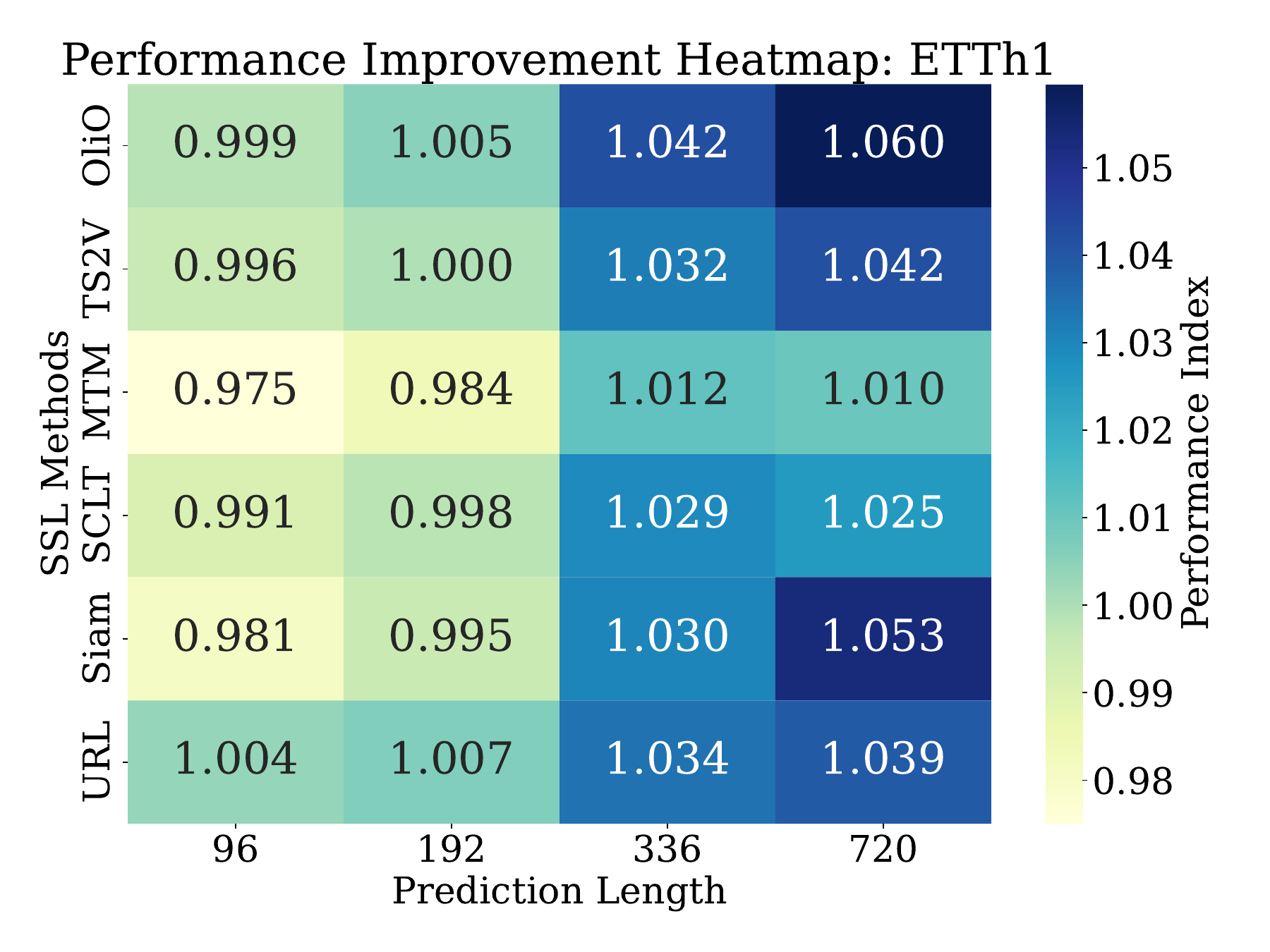}
        \caption{ETTh1 dataset.}
        \label{fig:heatmap_etth1}
    \end{minipage}
    \hfill
    \begin{minipage}[t]{0.48\textwidth}
        \centering
        \includegraphics[width=\linewidth]{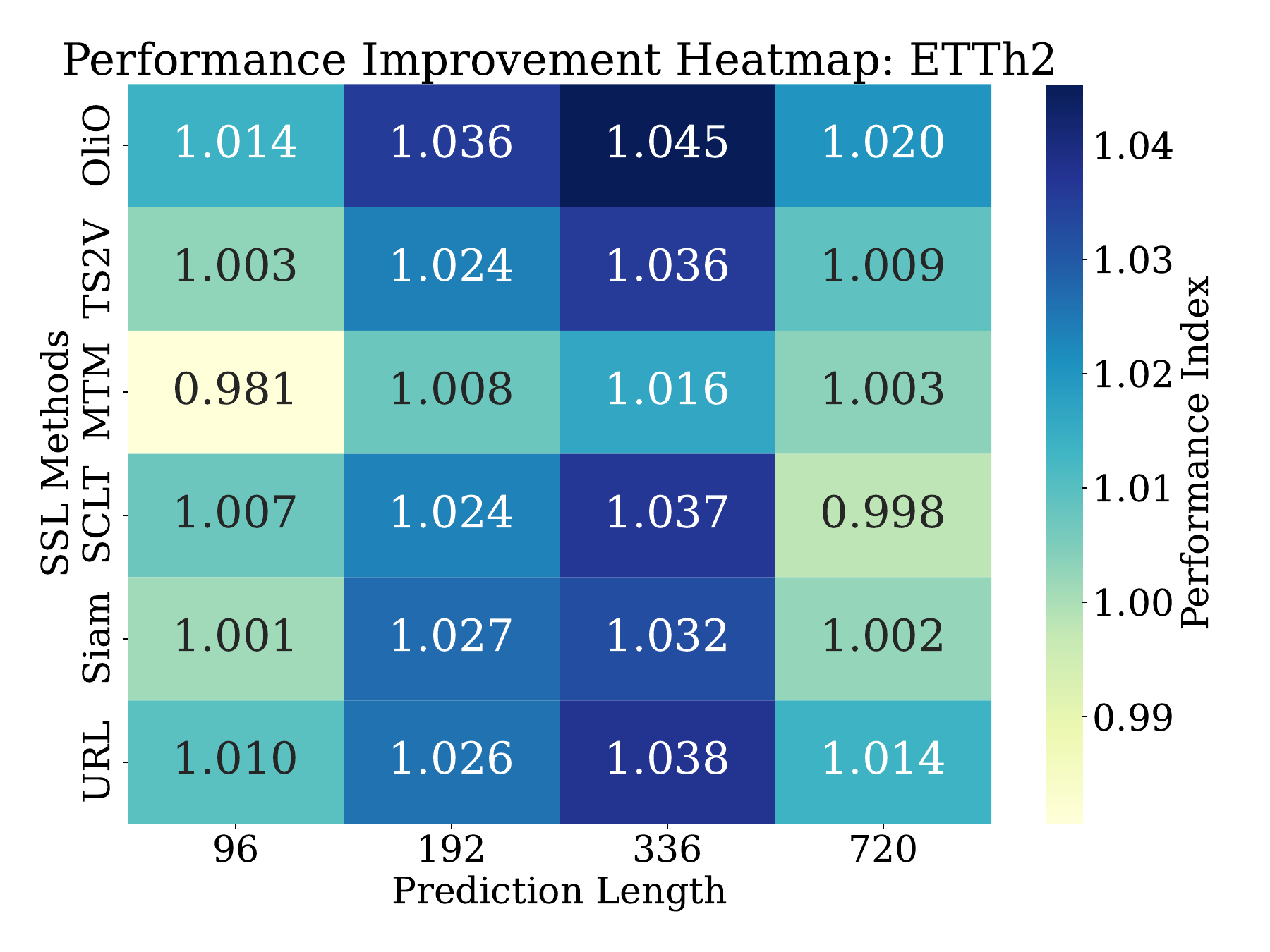}
        \caption{ETTh2 dataset.}
        \label{fig:heatmap_etth2}
    \end{minipage}

    \vspace{1.5em}

    \begin{minipage}[t]{0.48\textwidth}
        \centering
        \includegraphics[width=\linewidth]{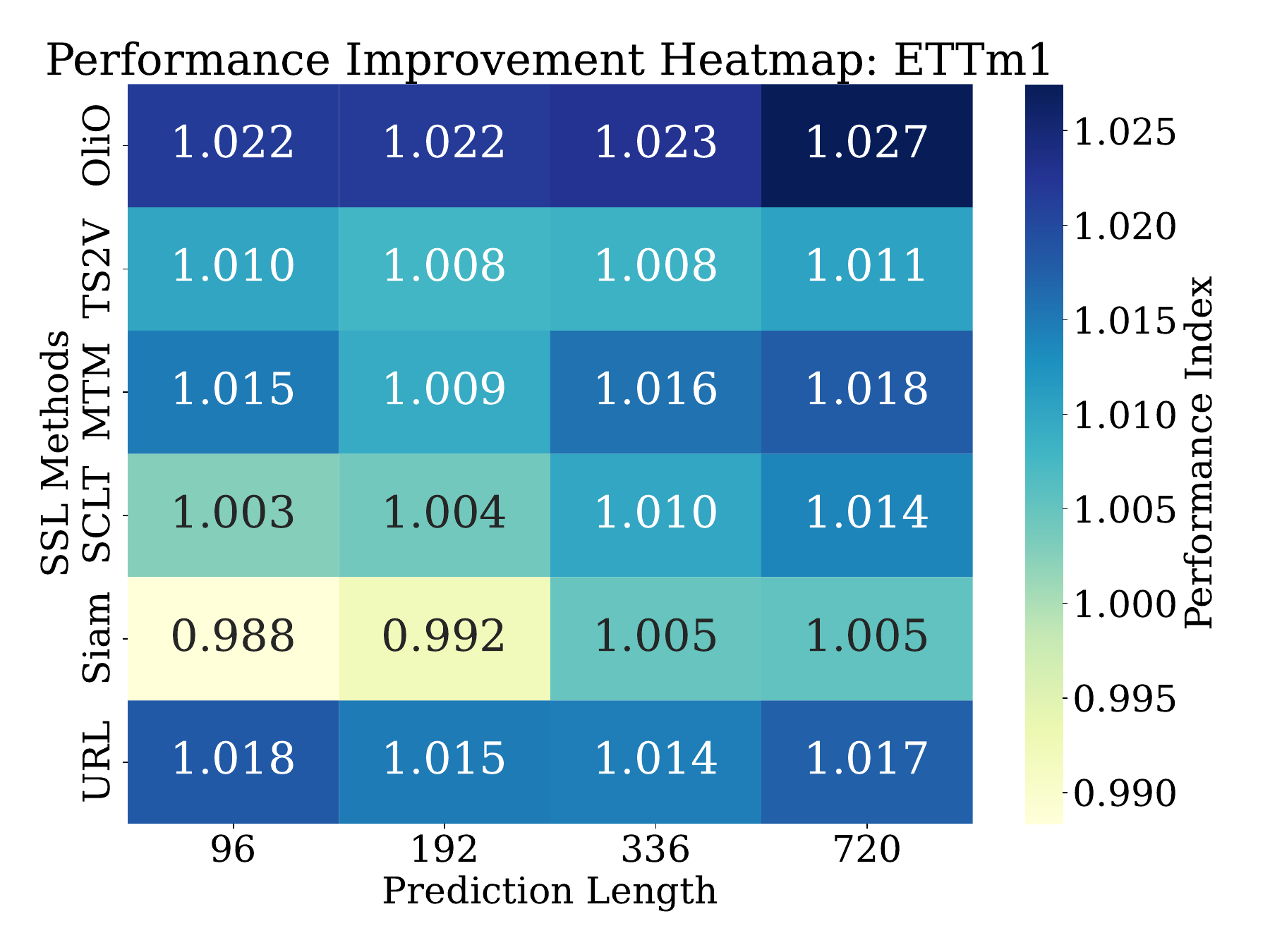}
        \caption{ETTm1 dataset.}
        \label{fig:heatmap_ettm1}
    \end{minipage}
    \hfill
    \begin{minipage}[t]{0.48\textwidth}
        \centering
        \includegraphics[width=\linewidth]{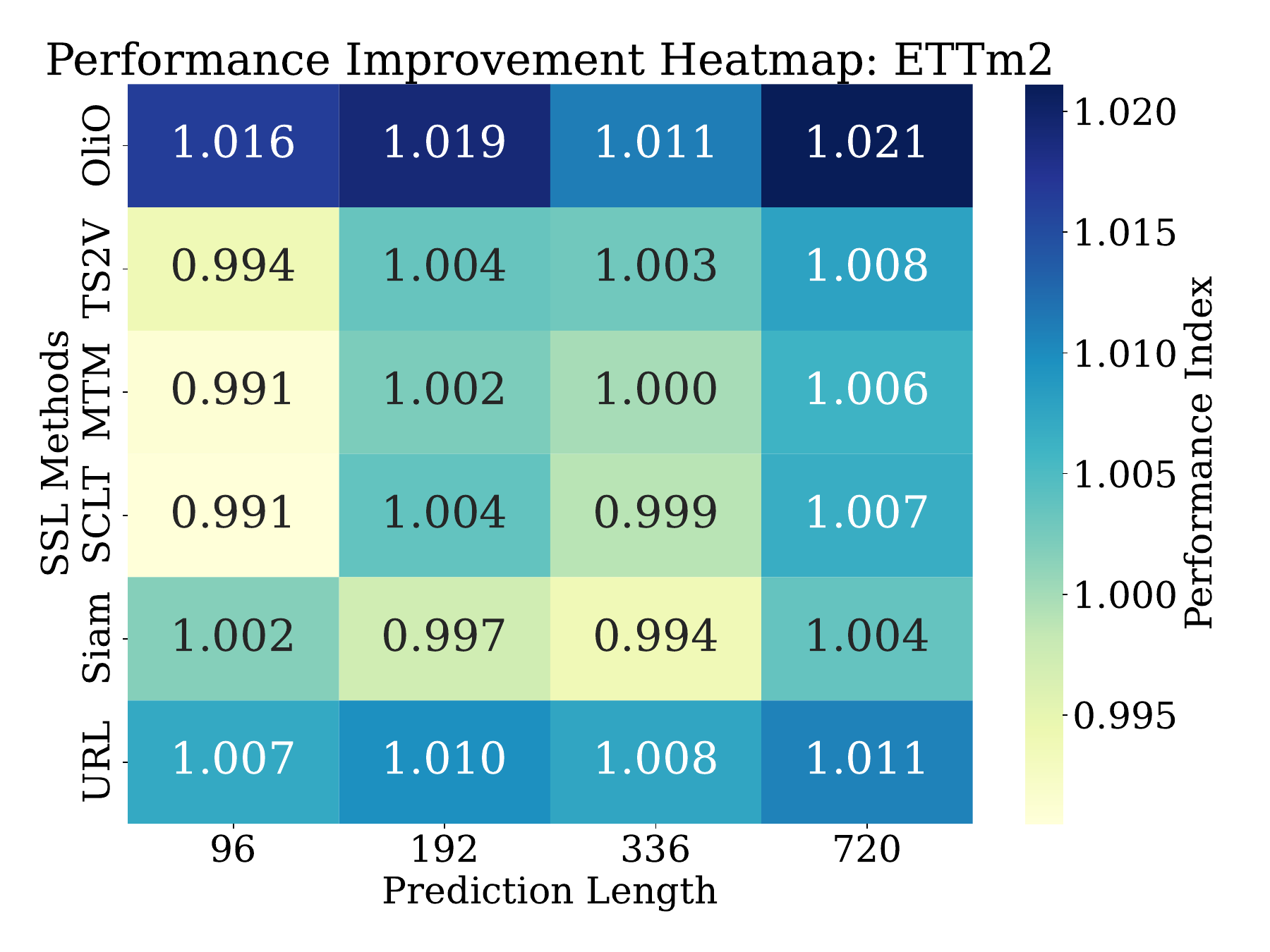}
        \caption{ETTm2 dataset.}
        \label{fig:heatmap_ettm2}
    \end{minipage}

    \vspace{1.5em}

    \begin{minipage}[t]{0.48\textwidth}
        \centering
        \includegraphics[width=\linewidth]{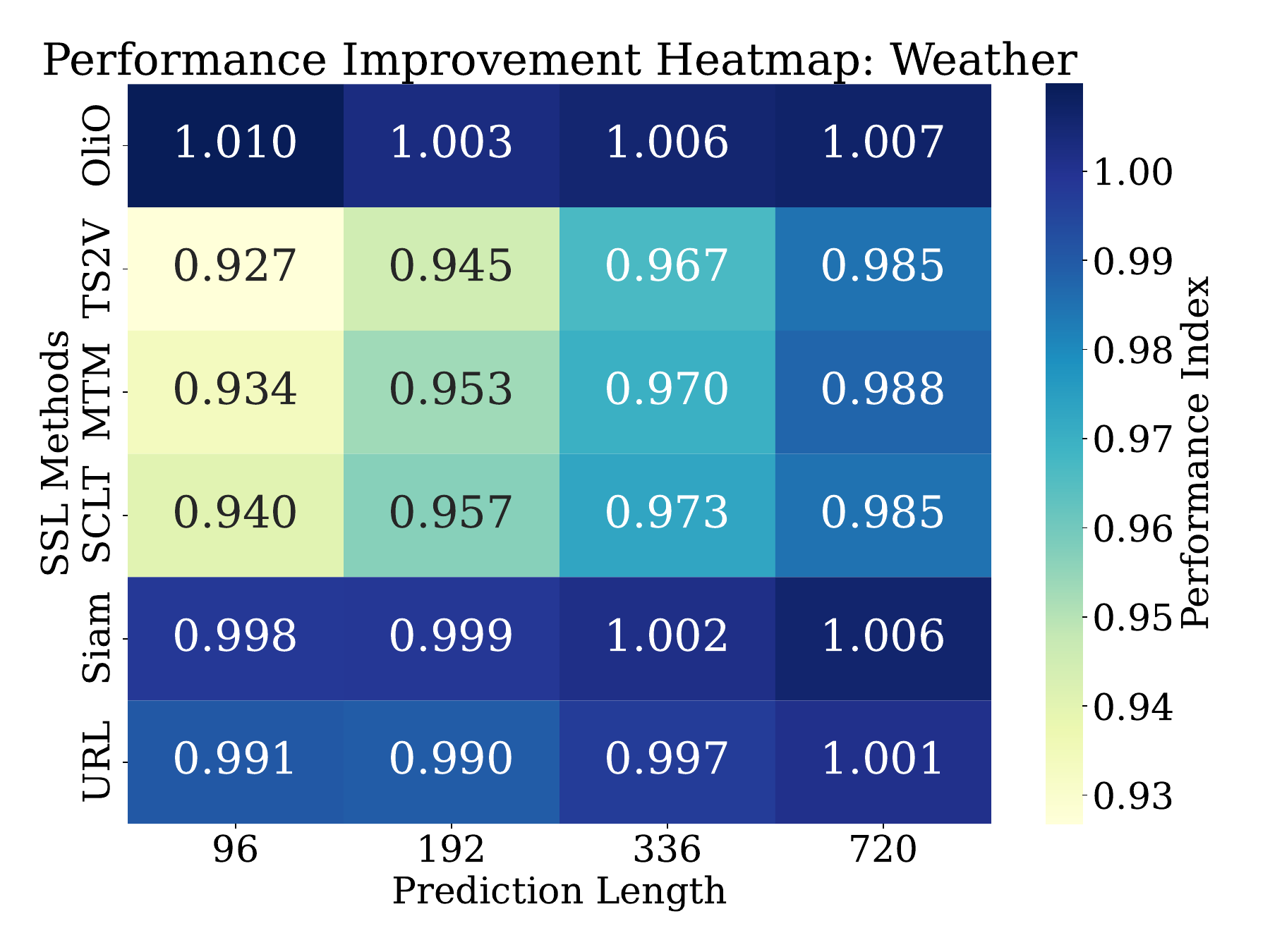}
        \caption{Weather dataset.}
        \label{fig:heatmap_weather}
    \end{minipage}
    \hfill
    \begin{minipage}[t]{0.48\textwidth}
        \centering
        \includegraphics[width=\linewidth]{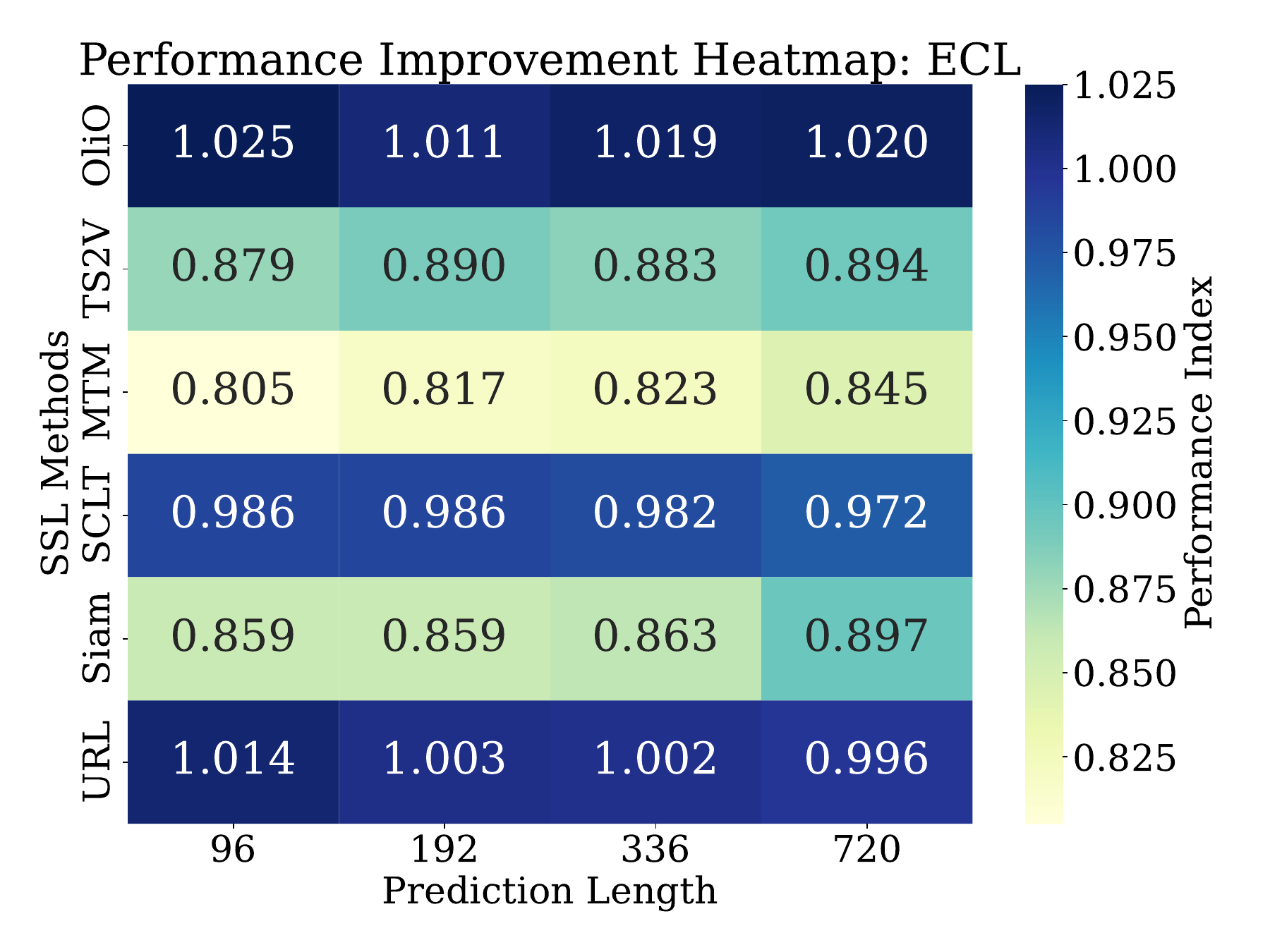}
        \caption{ECL dataset.}
        \label{fig:heatmap_ecl}
    \end{minipage}
    \caption*{Appendix figure: performance improvement heatmaps aggregated over all backbone architectures for each benchmark dataset. The values represent the mean performance index (Random MSE / Method MSE). These visualizations summarize how relative forecasting performance varies across datasets and prediction horizons.}
\end{figure*}

\clearpage

\section{Detailed ablation study results}
\label{app:ablation_study}

In this section, we provide detailed ablation results complementing the analyses in Sec.~\ref{sec:ablation_study}. We first report backbone-specific results comparing the bidirectional formulation of OliO with its unidirectional variants in Table~\ref{tab:ablation_detailed}. We then provide complete results for the transition-matrix constraint and approximation order $p$ in Table~\ref{tab:ablation_order_full}.

We then report overlap-ratio sweeps, endpoint comparisons, and fixed-pair masking measurements, followed by alternative transition parameterizations. Higher-order transition measurements and finite-range distance ordering complete the analysis.

\begin{table}[h]
\centering
\small
\caption{Detailed ablation study results for each backbone architecture. Values represent the average MSE across all prediction horizons $T \in \{96, 192, 336, 720\}$.}
\label{tab:ablation_detailed}
\begin{adjustbox}{width=\textwidth}
\begin{tabular}{llcccccc}
\toprule
Backbone & Method & ETTh1 & ETTh2 & ETTm1 & ETTm2 & Weather & ECL \\
\midrule
\multirow{3}{*}{PatchTST}
& \textbf{OliO (Bi)} & \textbf{.4103} & \textbf{.3424} & .3434          & \textbf{.2549} & \textbf{.2275} & \textbf{.1602} \\
& Forward Only        & .4133          & .3431          & .3474          & .2601          & .2137          & .2063          \\
& Backward Only       & .4126          & .3428          & \textbf{.3432} & .2601          & .2303          & .1621          \\
\midrule
\multirow{3}{*}{iTransformer}
& \textbf{OliO (Bi)} & \textbf{.4549} & \textbf{.3829} & \textbf{.3949} & \textbf{.2873} & \textbf{.2606} & \textbf{.1769} \\
& Forward Only        & .5181          & .4015          & .4877          & .3027          & .2867          & .3313          \\
& Backward Only       & .5117          & .3996          & .4895          & .3047          & .2798          & .3608          \\
\midrule
\multirow{3}{*}{TimeVLM}
& \textbf{OliO (Bi)} & \textbf{.4129} & .3434          & .3464          & .2536          & \textbf{.2252} & \textbf{.1622} \\
& Forward Only        & .4142          & \textbf{.3430} & .3464          & .2515          & .2256          & .1625          \\
& Backward Only       & .4151          & \textbf{.3430} & \textbf{.3463} & \textbf{.2503} & .2256          & .1625          \\
\bottomrule
\end{tabular}
\end{adjustbox}
\end{table}

\begin{table}[h]
\centering
\small
\caption{Full computation analysis: evaluation of transition matrix constraints and approximation order $p=1 \dots 5$. Values represent the average MSE across all backbones.}
\label{tab:ablation_order_full}
\setlength{\tabcolsep}{4pt}
\begin{tabular}{lccccccc}
\toprule
Matrix Type & Order ($p$) & ETTh1 & ETTh2 & ETTm1 & ETTm2 & Weather & ECL \\
\midrule
Unconstrained (Full) & Exact & NaN & NaN & NaN & NaN & NaN & NaN \\
\textbf{Nilpotent (OliO)} & \textbf{Exact} & \textbf{.4260} & \textbf{.3562} & \textbf{.3616} & \textbf{.2653} & \textbf{.2378} & \textbf{.1664} \\
\midrule
Nilpotent (OliO) & $p=1$ & .4434 & .3661 & .3752 & .2674 & .2470 & .2137 \\
Nilpotent (OliO) & $p=2$ & .4315 & .3624 & .3695 & .2661 & .2480 & .2063 \\
Nilpotent (OliO) & $p=3$ & .4315 & .3606 & .3686 & .2678 & .2382 & .2029 \\
Nilpotent (OliO) & $p=4$ & .4325 & .3629 & .3720 & .2675 & .2457 & .2051 \\
Nilpotent (OliO) & $p=5$ & .4299 & .3646 & .3678 & .2669 & .2464 & .1957 \\
\bottomrule
\end{tabular}
\end{table}

\subsection{Effect of window overlap}
\label{app:overlap_analysis}

For windows of length $l$ with start separation $0\le\tau\le l$, the overlap ratio is $1-\tau/l$. We vary $\tau$ while keeping the window length, backbone, and remaining training configuration fixed. Table~\ref{tab:overlap_ratio} reports forecasting MSE across four overlap ratios.

\begin{table}[!htbp]
\centering
\caption{Forecasting MSE under different overlap ratios. Lower is better.}
\label{tab:overlap_ratio}
\begin{tabular}{lcccc}
\toprule
Setting & 90\% & 80\% & 60\% & 20\% \\
\midrule
PatchTST--ETTh1      & 0.421 & 0.418 & \textbf{0.417} & \textbf{0.417} \\
PatchTST--ETTh2      & \textbf{0.343} & \textbf{0.343} & \textbf{0.343} & 0.344 \\
iTransformer--ETTm1 & 0.404 & 0.406 & 0.403 & \textbf{0.400} \\
iTransformer--ETTm2 & \textbf{0.299} & 0.305 & 0.306 & 0.306 \\
\bottomrule
\end{tabular}
\end{table}

The lowest reported MSE occurs at different overlap ratios across the four settings. Increasing overlap produces no monotonic improvement in forecasting MSE. Both lag and shared content vary in this sweep; the fixed-pair masking analysis in App.~\ref{app:matched_masking} holds the pair and lag constant.

For Weather, we additionally report an extended sweep spanning both overlap endpoints in Table~\ref{tab:weather_overlap}. The preferred ratio depends on the backbone, and increasing overlap does not monotonically improve forecasting MSE.

\begin{table}[!htbp]
\centering
\caption{Extended overlap-ratio sweep on Weather. Lower MSE is better.}
\label{tab:weather_overlap}
\begin{adjustbox}{max width=\linewidth}
\begin{tabular}{lccccccc}
\toprule
Backbone & 100\% & 90\% & 80\% & 60\% & 40\% & 20\% & 0\% \\
\midrule
PatchTST & .2299 & .2298 & .2288 & .2288 & \textbf{.2279} & .2288 & .2287 \\
iTransformer & .2871 & .2681 & .2725 & .2818 & .2752 & \textbf{.2668} & \textbf{.2668} \\
\bottomrule
\end{tabular}
\end{adjustbox}
\end{table}

\subsection{Zero-overlap analysis}
\label{app:zero_overlap}

Table~\ref{tab:zero_overlap} compares the backbone without OliO with two overlap endpoints. At 100\% overlap, $\tau=0$ and $\mathcal T_0(\bm h)=\bm h$, giving a zero-lag control. At 0\% overlap, the windows share no observations. These endpoint conditions supplement the standard pre-training range $0<\tau<l$.

\begin{table}[!htbp]
\centering
\caption{Comparison of models without OliO, OliO with full-overlap pairs, and OliO with zero-overlap pairs. Results are averaged over three seeds. Lower MSE is better.}
\label{tab:zero_overlap}
\begin{tabular}{lccc}
\toprule
Setting & w/o OliO & 100\% Overlap & 0\% Overlap \\
\midrule
PatchTST--ETTh1      & 0.4189 & \textbf{0.4178} & 0.4179 \\
PatchTST--ETTh2      & 0.3594 & \textbf{0.3416} & 0.3437 \\
PatchTST--Weather    & 0.2297 & 0.2299 & \textbf{0.2287} \\
iTransformer--ETTm1 & 0.4097 & 0.4068 & \textbf{0.4060} \\
iTransformer--ETTm2 & 0.2917 & \textbf{0.2889} & 0.2908 \\
iTransformer--Weather & \textbf{0.2615} & 0.2871 & 0.2668 \\
\bottomrule
\end{tabular}
\end{table}

In five of the six settings, the MSE difference between the zero-overlap and full-overlap conditions is within 0.7\%. The zero-overlap condition also has lower MSE than the corresponding backbone without OliO in five settings. iTransformer--Weather favors the backbone without OliO, illustrating variation across the evaluated configurations.

\subsection{Matched masking control}
\label{app:matched_masking}

We further conduct a controlled masking experiment at 50\% overlap. The underlying window pair, signed temporal lag, mask ratio, backbone, training schedule, and hyperparameters are held identical. The shared half of the second window is set to zero in the overlap-mask condition; an equal-sized unique half is set to zero in the unique-mask condition. Masked observations retain their original positions, and the input length is unchanged.

\begin{table}[!htbp]
\centering
\caption{Matched masking control at 50\% overlap. MSE is averaged over ETTh1, ETTh2, ETTm1, ETTm2, Weather, and ECL and horizons $\{96,192,336,720\}$ (24 settings per backbone; 48 overall). ``Wins'' reports Overlap-Mask / Unique-Mask wins. Lower MSE is better.}
\label{tab:matched_masking}
\begin{tabular}{lcccc}
\toprule
Backbone & Overlap Mask & Unique Mask & Gap & Wins \\
\midrule
PatchTST
& 0.29818 & \textbf{0.29767} & +0.17\% & 9 / 15 \\
iTransformer
& 0.35022 & \textbf{0.34932} & +0.26\% & 14 / 10 \\
\midrule
Overall
& 0.32420 & \textbf{0.32349} & +0.22\% & 23 / 25 \\
\bottomrule
\end{tabular}
\end{table}

The aggregate MSE difference is 0.22\%. Overlap masking yields lower MSE in 23 settings, and unique masking in 25 settings. The equal-budget conditions therefore produce closely matched aggregate errors and balanced setting-level wins. This diagnostic measures sensitivity to mask location for fixed window pairs and signed lags.

\subsection{Alternative transition parameterizations}
\label{app:transition_parameterization}

The parameterization comparison in Sec.~\ref{sec:ablation_study} uses PatchTST and iTransformer across six datasets and four horizons, giving 48 settings. For signed lag $s$, Bias-only uses $\bm h+s\bm v$, and K-only uses $e^{sK}\bm h$. MLP replaces the complete transition with a generic nonlinear mapping. Spec-K uses a full generator regularized by spectral normalization in place of the strictly upper-triangular generator.

Across the 48 backbone--dataset--horizon settings, OliO achieves an average MSE of 0.3080, compared with 0.3116 for Spec-K, 0.3123 for K-only, 0.3176 for MLP, and 0.3198 for Bias-only. OliO achieves lower MSE than these alternatives in 39/48, 42/48, 45/48, and 47/48 settings, respectively.

Table~\ref{tab:transition_parameterization_detail} gives the dataset--backbone averages. Bias-only has the largest aggregate MSE increase among these alternatives. The dataset-level results also show variation: K-only has the lowest MSE on PatchTST--ETTm1, and Spec-K on iTransformer--Weather. Spec-K trains stably in these settings, providing a full-generator comparison alongside the unconstrained generator in Table~\ref{tab:ablation_order_full}.

\begin{table}[!htbp]
\centering
\caption{
Dataset-averaged MSE of alternative transition parameterizations.
Each value is averaged over four forecasting horizons.
Lower MSE is better.
}
\label{tab:transition_parameterization_detail}
\begin{adjustbox}{max width=\linewidth}
\begin{tabular}{llccccc}
\toprule
Backbone & Dataset & OliO & Bias-only & K-only & MLP & Spec-K \\
\midrule

\multirow{6}{*}{iTransformer}
& ECL     & \textbf{.1769} & .1904 & .1786 & .1954 & .1785 \\
& ETTh1   & \textbf{.4549} & .4581 & .4592 & .4600 & .4596 \\
& ETTh2   & \textbf{.3829} & .3865 & .3879 & .3888 & .3873 \\
& ETTm1   & \textbf{.3949} & .4731 & .4005 & .4308 & .3986 \\
& ETTm2   & \textbf{.2873} & .2901 & .2883 & .3035 & .2894 \\
& Weather & .2606 & .2789 & .2840 & .2805 & \textbf{.2600} \\
\midrule

\multirow{6}{*}{PatchTST}
& ECL     & \textbf{.1602} & .1634 & .1605 & .1620 & .1605 \\
& ETTh1   & \textbf{.4102} & .4200 & .4132 & .4144 & .4158 \\
& ETTh2   & \textbf{.3424} & .3433 & .3437 & .3433 & .3440 \\
& ETTm1   & .3434 & .3484 & \textbf{.3431} & .3440 & .3458 \\
& ETTm2   & \textbf{.2549} & .2573 & .2577 & .2594 & .2722 \\
& Weather & \textbf{.2275} & .2282 & .2309 & .2293 & .2278 \\

\bottomrule
\end{tabular}
\end{adjustbox}
\end{table}

\subsection{Higher-order transition analysis}
\label{app:higher_order_transition}

We evaluate the matrix component of the transition at all integer lags from 1 to 64. Table~\ref{tab:higher_order_transition_full} reports eight representative lags, extending the summary in Table~\ref{tab:higher_order_transition_main}. The diagnostic compares the exact matrix exponential with its first-order approximation; the additive offset is identical in both affine maps and is excluded from the relative matrix error.

We measure the relative approximation error between the ODE-derived matrix
transport and its first-order approximation as
\begin{equation}
r(\tau)
=
\frac{
\left\|e^{\tau K}-(I+\tau K)\right\|_F
}{
\left\|e^{\tau K}\right\|_F
}.
\end{equation}

\begin{table}[!htbp]
\centering
\caption{Relative approximation error between the complete transition and its first-order approximation across temporal lags.}
\label{tab:higher_order_transition_full}
\begin{tabular}{ccc}
\toprule
Lag $\tau$ & ETTh1 & ETTh2 \\
\midrule
1  & 0.018 & 0.000 \\
10 & 0.617 & 0.034 \\
20 & 0.857 & 0.104 \\
30 & 0.931 & 0.182 \\
40 & 0.965 & 0.258 \\
50 & 0.986 & 0.331 \\
60 & 1.001 & 0.402 \\
64 & 1.004 & 0.430 \\
\bottomrule
\end{tabular}
\end{table}

For the two evaluated operators, relative error increases monotonically over lags 1--64. The discrepancy is smaller at short lags and grows as higher-order terms contribute more to the matrix transition. The order ablation in Table~\ref{tab:ablation_order_full} reports the corresponding forecasting results for separately trained truncation orders.

\subsection{Distance ordering over the evaluated lag range}
\label{app:distance_ordering}
We evaluate the learned checkpoints using the complete affine transition, including its shared lag-dependent offset. For each encoded state $\bm h$, we compute the adjacent-lag increment
\begin{equation}
\Delta_h(k)=\|\mathcal T_{+(k+1)}(\bm h)-\bm h\|_2^2-\|\mathcal T_{+k}(\bm h)-\bm h\|_2^2,
\qquad k=0,\ldots,7.
\end{equation}
Complete pointwise distance ordering holds for 88\%--100\% of evaluated representations across checkpoints. At every adjacent lag in every evaluated checkpoint, both the mean and median of $\Delta_h(k)$ are positive. Every evaluated representation also has positive initial transition-velocity norm $\|K\bm h+\bm v\|_2$. These measurements cover integer lags 0 through 8. App.~\ref{app:displacement_growth} gives the analytical local and asymptotic properties of fixed-source displacement.

\clearpage
\newpage
\section*{NeurIPS Paper Checklist}

\begin{enumerate}

\item {\bf Claims}
    \item[] Question: Do the main claims made in the abstract and introduction accurately reflect the paper's contributions and scope?
    \item[] Answer: \answerYes{} 
    \item[] Justification: The abstract and introduction explicitly detail the main contribution.

    \item {\bf Limitations}
    \item[] Question: Does the paper discuss the limitations of the work performed by the authors?
    \item[] Answer: \answerYes{} 
    \item[] Justification: The limitations regarding the expressivity of low-order polynomial expansions for highly complex datasets are discussed in \cref{sec:experiment}. The impact of static text prompts in specific backbone models (e.g., TimeVLM) on latent visualization is addressed in \cref{appendix:latent_timevlm}. Computational efficiency and scaling analysis are provided in \cref{sec:complexity_analysis}.

\item {\bf Theory assumptions and proofs}
    \item[] Question: For each theoretical result, does the paper provide the full set of assumptions and a complete (and correct) proof?
    \item[] Answer: \answerYes{} 
    \item[] Justification: The paper provides formal statements and assumptions for all theoretical results, including the polynomial distance relationship. Detailed proofs for these results, along with auxiliary Corollary are provided in \cref{sec:proofs}. 

    \item {\bf Experimental result reproducibility}
    \item[] Question: Does the paper fully disclose all the information needed to reproduce the main experimental results of the paper to the extent that it affects the main claims and/or conclusions of the paper (regardless of whether the code and data are provided or not)?
    \item[] Answer: \answerYes{} 
    \item[] Justification: Detailed setup, including hardware, software, and hyperparameter search space, is provided in \cref{sec:implementation_details}.

\item {\bf Open access to data and code}
    \item[] Question: Does the paper provide open access to the data and code, with sufficient instructions to faithfully reproduce the main experimental results, as described in supplemental material?
    \item[] Answer: \answerYes{} 
    \item[] Justification: An anonymous GitHub repository containing the implementation code and instructions is provided at \url{https://anonymous.4open.science/r/OliO-46EF}.

\item {\bf Experimental setting/details}
    \item[] Question: Does the paper specify all the training and test details (e.g., data splits, hyperparameters, how they were chosen, type of optimizer) necessary to understand the results?
    \item[] Answer: \answerYes{} 
    \item[] Justification: Detailed pre-training and fine-tuning configurations, including the hyperparameter search space for each backbone, are provided in \cref{sec:implementation_details,sec:datasets}.

\item {\bf Experiment statistical significance}
    \item[] Question: Does the paper report error bars suitably and correctly defined or other appropriate information about the statistical significance of the experiments?
    \item[] Answer: \answerYes{} 
    \item[] Justification: Experimental results are reported with standard deviations calculated over three independent runs with different random seeds, as shown in the result tables in \cref{app:detailed_results}.

\item {\bf Experiments compute resources}
    \item[] Question: For each experiment, does the paper provide sufficient information on the computer resources (type of compute workers, memory, time of execution) needed to reproduce the experiments?
    \item[] Answer: \answerYes{} 
    \item[] Justification: Hardware specifications (single NVIDIA RTX 3090 GPU), compute environment, average execution times, and total compute estimates are detailed in \cref{sec:implementation_details}.
    
\item {\bf Code of ethics}
    \item[] Question: Does the research conducted in the paper conform, in every respect, with the NeurIPS Code of Ethics \url{https://neurips.cc/public/EthicsGuidelines}?
    \item[] Answer: \answerYes{} 
    \item[] Justification: The research conducted in this paper conforms to the NeurIPS Code of Ethics.

\item {\bf Broader impacts}
    \item[] Question: Does the paper discuss both potential positive societal impacts and negative societal impacts of the work performed?
    \item[] Answer: \answerNA{} 
    \item[] Justification: This work focuses on foundational representation learning for time series. As such, it does not directly address specific societal applications or impacts.
    \item[] Guidelines:
    \begin{itemize}
        \item The answer \answerNA{} means that there is no societal impact of the work performed.
        \item If the authors answer \answerNA{} or \answerNo, they should explain why their work has no societal impact or why the paper does not address societal impact.
        \item Examples of negative societal impacts include potential malicious or unintended uses (e.g., disinformation, generating fake profiles, surveillance), fairness considerations (e.g., deployment of technologies that could make decisions that unfairly impact specific groups), privacy considerations, and security considerations.
        \item The conference expects that many papers will be foundational research and not tied to particular applications, let alone deployments. However, if there is a direct path to any negative applications, the authors should point it out. For example, it is legitimate to point out that an improvement in the quality of generative models could be used to generate Deepfakes for disinformation.
        \item The authors should consider possible harms that could arise when the technology is being used as intended and functioning correctly, harms that could arise when the technology is being used as intended but gives incorrect results, and harms following from (intentional or unintentional) misuse of the technology.
        \item If there are negative societal impacts, the authors could also discuss possible mitigation strategies (e.g., gated release of models, providing defenses in addition to attacks, mechanisms for monitoring misuse, mechanisms to monitor how a system learns from feedback over time, improving the efficiency and accessibility of ML).
    \end{itemize}
    
\item {\bf Safeguards}
    \item[] Question: Does the paper describe safeguards that have been put in place for responsible release of data or models that have a high risk for misuse (e.g., pre-trained language models, image generators, or scraped datasets)?
    \item[] Answer: \answerNA{} 
    \item[] Justification: The proposed framework is a general time series forecasting tool and does not pose significant risks for misuse compared to large-scale generative models.

\item {\bf Licenses for existing assets}
    \item[] Question: Are the creators or original owners of assets (e.g., code, data, models), used in the paper, properly credited and are the license and terms of use explicitly mentioned and properly respected?
    \item[] Answer: \answerYes{} 
    \item[] Justification: All benchmark datasets (ETT, Weather, Electricity) are properly credited and cited, and were used in accordance with their research-oriented terms of use.

\item {\bf New assets}
    \item[] Question: Are new assets introduced in the paper well documented and is the documentation provided alongside the assets?
    \item[] Answer: \answerNA{} 
    \item[] Justification: No new assets (like new datasets or physical artifacts) are introduced in this paper.

\item {\bf Crowdsourcing and research with human subjects}
    \item[] Question: For crowdsourcing experiments and research with human subjects, does the paper include the full text of instructions given to participants and screenshots, if applicable, as well as details about compensation (if any)? 
    \item[] Answer: \answerNA{} 
    \item[] Justification: This research does not involve crowdsourcing or human subjects.

\item {\bf Institutional review board (IRB) approvals or equivalent for research with human subjects}
    \item[] Question: Does the paper describe potential risks incurred by study participants, whether such risks were disclosed to the subjects, and whether Institutional Review Board (IRB) approvals (or an equivalent approval/review based on the requirements of your country or institution) were obtained?
    \item[] Answer: \answerNA{} 
    \item[] Justification: This research does not involve human subjects.

\item {\bf Declaration of LLM usage}
    \item[] Question: Does the paper describe the usage of LLMs if it is an important, original, or non-standard component of the core methods in this research? Note that if the LLM is used only for writing, editing, or formatting purposes and does \emph{not} impact the core methodology, scientific rigor, or originality of the research, declaration is not required.
    \item[] Answer: \answerNA{} 
    \item[] Justification: LLMs were not used as an important, original, or non-standard component of the core methods in this research.

\end{enumerate}

\end{document}